\documentclass[nohyperref]{article}

\usepackage{mymath,enumitem}

\usepackage{microtype}
\usepackage{graphicx}
\usepackage{subfigure}
\usepackage{booktabs} 

\usepackage{hyperref}

\usepackage[accepted]{icml2023}

\usepackage{amsmath}
\usepackage{amssymb}
\usepackage{mathtools}
\usepackage{amsthm}

\usepackage[capitalize,noabbrev]{cleveref}

\theoremstyle{plain}
\newtheorem{theorem}{Theorem}[section]

\newtheorem{lemma}[theorem]{Lemma}
\newtheorem{corollary}[theorem]{Corollary}
\newtheorem*{corollary*}{Corollary}
\theoremstyle{definition}
\newtheorem{definition}[theorem]{Definition}
\newtheorem{assumption}[theorem]{Assumption}
\theoremstyle{remark}

\icmltitlerunning{Finite-Sample Analysis of Learning High-Dimensional Single ReLU Neuron}

\begin{document}

\twocolumn[
\icmltitle{Finite-Sample Analysis of Learning High-Dimensional Single ReLU Neuron}



\icmlsetsymbol{equal}{*}

\begin{icmlauthorlist}
\icmlauthor{Jingfeng Wu}{equal,jhu}
\icmlauthor{Difan Zou}{equal,hku}
\icmlauthor{Zixiang Chen}{equal,ucla}
\icmlauthor{Vladimir Braverman}{rice}
\icmlauthor{Quanquan Gu}{ucla}
\icmlauthor{Sham M.~Kakade}{harvard}
\end{icmlauthorlist}

\icmlaffiliation{jhu}{Department of Computer Science, Johns Hopkins University}
\icmlaffiliation{hku}{Department of Computer Science, The University of Hong Kong}
\icmlaffiliation{ucla}{Department of Computer Science, University of California, Los Angeles}
\icmlaffiliation{rice}{Department of Computer Science, Rice University}
\icmlaffiliation{harvard}{Department of Computer Science and Department of Statistics, Harvard University}


\icmlcorrespondingauthor{Vladimir Braverman}{vb21@rice.edu}
\icmlcorrespondingauthor{Quanquan Gu}{qgu@cs.ucla.edu}
\icmlcorrespondingauthor{Sham M.~Kakade}{sham@seas.harvard.edu}

\icmlkeywords{SGD, ReLU, High-dimension, Risk bound}

\vskip 0.3in
]



\printAffiliationsAndNotice{\icmlEqualContribution} 

\begin{abstract}
This paper considers the problem of learning a single ReLU neuron with squared loss (a.k.a., ReLU regression) in the overparameterized regime, where the input dimension can exceed the number of samples.
We analyze a Perceptron-type algorithm called GLM-tron \citep{kakade2011efficient} and provide its dimension-free risk upper bounds for high-dimensional ReLU regression in both well-specified and misspecified settings. Our risk bounds recover several existing results as special cases. 
Moreover, in the well-specified setting, we provide an instance-wise matching risk lower bound for GLM-tron.
Our upper and lower risk bounds provide a sharp characterization of the high-dimensional ReLU regression problems that can be learned via GLM-tron.
On the other hand, we provide some negative results for stochastic gradient descent (SGD) for ReLU regression with symmetric Bernoulli data: if the model is well-specified, the excess risk of SGD is provably no better than that of GLM-tron ignoring constant factors, for each problem instance; and in the noiseless case, GLM-tron can achieve a small risk while SGD unavoidably suffers from a constant risk in expectation.
These results together suggest that GLM-tron might be preferable to SGD for high-dimensional ReLU regression.

\end{abstract}

\allowdisplaybreaks

\section{Introduction}
In modern machine learning such as deep learning, the number of model parameters often exceeds the amount of training data, which is often referred to as overparameterization.
Yet, overparameterized models (when properly optimized) can still achieve strong generalization performance in practice. 
Understanding the statistical learning mechanism in the overparameterized regime has drawn great attention in the learning theory community.

Recently, overparameterized linear regression problems have been extensively investigated. 
Dimensional-free, finite-sample, and instance-wise excess risk bounds have been established for various algorithms, including the minimal $\ell_2$-norm interpolator \citep{bartlett2020benign}, ridge regression \citep{tsigler2020benign,cheng2022dimension}, low-norm interpolator \citep{zhou2020uniform,zhou2021optimistic,koehler2021uniform} and the online stochastic gradient descent (SGD) methods \citep{zou2021benign,wu2022iterate}.
These results together deliver a relatively comprehensive picture of when and how high-dimensional linear regression problems can be learned with finite samples.

However, when the model is not linear, the overparameterized regime is much less well understood, even for the arguably simplest \emph{ReLU regression} problems (see \eqref{eq:risk}). 
This work aims to fill this gap by providing sharp risk bounds for learning high-dimensional ReLU regression problems with finite samples. 

\paragraph{High-Dimensional ReLU Regression.}
The problem of ReLU Regression aims to minimize the following risk:
\begin{equation}\label{eq:risk}
    \risk(\wB):=  \Ebb \big(\relu(\xB^\top\wB) - y\big)^2,\quad \wB \in \Hbb,
\end{equation}
where $\Hbb$ is a Hilbert space that can be either $d$-dimensional for a finite $d$ or countably infinite dimensional;
\(\relu(\cdot) := \max\{\cdot, 0 \}\) is the \emph{Rectified Linear Unit} (ReLU);
$(\xB, y) \in \Hbb \otimes \Rbb$ denotes a pair of an input feature vector and the corresponding scalar response; the expectation is taken over some unknown distribution of $(\xB, y)$; and $\wB \in \Hbb$ denotes the model parameter. 
It is worth noting that in general $\risk(\cdot)$ is non-convex due to
the non-linearity of $\relu$.
Therefore, ReLU regression is significantly harder than linear regression. 

Given $N$ i.i.d.\ samples, $(\xB_t, y_t)_{t=1}^N$, two iterative algorithms will be considered for optimizing \eqref{eq:risk}. 
The first algorithm is \emph{stochastic gradient descent} (SGD), which is initialized from $\wB_0$ and then makes the following update:
for $t=1,\dots,N$,
\begin{align}
    \wB_{t}
    &= \wB_{t-1} - \gamma_t \cdot \gB_t, \ \text{and}\ \label{eq:sgd}\tag{\text{SGD}}\\
    \gB_t &:= \big( \relu(\xB_t^\top \wB_{t-1}) - y_{t} \big) \xB_t\cdot \onebb[\xB_t^\top\wB_{t-1} > 0] \notag
\end{align}
where $(\gamma_t)_{t=0}^{N}$ refers to a stepsize scheduler, e.g., a geometrically decaying stepsize scheduler \citep{ge2019step,wu2022iterate},
\begin{equation}\label{eq:geometry-tail-decay-lr}
\text{for}\ t\ge 1, \ 
    \gamma_{t} = 
    \begin{cases}
    \gamma_{t-1} / 2, & t\ \% \ \big( N / \log(N) \big) = 0; \\
    \gamma_{t-1}, & \text{otherwise};
    \end{cases}
\end{equation}
and the output is the last iterate, i.e., $\wB_N$.
The second algorithm is known as \emph{Generalized Linear Model Perceptron} (GLM-tron) \citep{kalai2009isotron,kakade2011efficient}, which is also initialized from $\wB_0$ and makes the following update:
for $t=1,\dots,N$,
\begin{equation}\label{eq:tron}\tag{\text{GLM-tron}}
    \wB_{t}
    = \wB_{t-1} - \gamma_t \cdot \big( \relu(\xB_t^\top \wB_{t-1}) - y_{t} \big) \xB_t,
\end{equation}
where $(\gamma_t)_{t=0}^N$ is a stepsize scheduler, e.g., \eqref{eq:geometry-tail-decay-lr};
and the output is the last iterate, i.e., $\wB_N$. 
Comparing these two algorithms, the only difference is that \eqref{eq:tron} ignores the derivative of $\relu(\cdot)$ in its updates.

\paragraph{Contribution 1 (Well-Specified Setting).}
We first consider the well-specified setting (also known as the ``noisy teacher'' setting \citep{frei2020agnostic}), where the expectation of the label conditioned on the input is a linear function followed by ReLU. 
In this setting, we provide a risk upper bound, $\min\risk(\cdot) + \Ocal(\Deff / N)$, for \eqref{eq:tron}, where $\Deff$ is an \emph{effective} dimension jointly determined by the sample size, stepsize, and the data covariance matrix, and is independent of the ambient dimension.
In particular, $\Deff$ is small when the spectrum of the data covariance matrix decays fast. 
Moreover, we provide an instance-wise nearly-matching risk lower bound, demonstrating the tightness of our analysis. 
These bounds are in a similar flavor as the benign-overfitting-type bounds established for high-dimensional linear models (see, e.g.,  \citet{bartlett2020benign,tsigler2020benign,zou2021benign}), but are the first of their kind for high-dimensional non-linear models.

\paragraph{Contribution 2 (Misspecified Setting).}
We then consider the misspecified setting (also known as the agnostic setting, see, e.g., \citet{diakonikolas2020approximation}), where no distributional assumption is made on the label generation. 
In this case, we provide an $\Ocal(\min\risk(\cdot) + \Deff/N )$ risk upper bound for \eqref{eq:tron}, where the $\Deff$ is the same effective dimension defined in the well-specified setting.
Therefore, we can characterize when \eqref{eq:tron} achieves a constant-factor approximation for misspecified ReLU regression in the overparameterized regime. 
In particular, when specialized to the finite-dimensional case, our upper bound improves an existing analysis for GLM-tron by \citet{diakonikolas2020approximation}.

\paragraph{Contribution 3 (Comparison with SGD).}
We also show some negative results on \eqref{eq:sgd} for ReLU regression with symmetric Bernoulli data:
in the well-specified case, we show that the excess risk achieved by \eqref{eq:sgd} is always no better than that achieved by \eqref{eq:tron} ignoring constant factors, for every problem instance;
in the noiseless case, 
\eqref{eq:sgd} unavoidably suffers from a constant risk (in expectation) while \eqref{eq:tron} is able to attain an arbitrarily small risk. 
These together suggest a potentially more preferable algorithmic bias of \eqref{eq:tron} (compared with \eqref{eq:sgd})
in ReLU regression.

\paragraph{Contribution 4 (Techniques).}
From a technical perspective, we introduce new analysis techniques which extend the operator method initially developed for linear models (see, e.g., \citet{jain2017parallelizing,zou2021benign,wu2022iterate} and references therein) to handle the non-linearity of ReLU. 
The key idea is, instead of controlling the entire covariance matrix of the iterates as in the linear case, one should work with  the \emph{diagonal} matrix  to better deal with the non-linearity of ReLU. 
Our novel development of the operator method can be of independent  interest.

\paragraph{Paper Organization.}
The remaining paper is organized as follows.
We first review related literature in Section \ref{sec:related}.
Then we set up the preliminaries in Section \ref{sec:setup}.
We present our main results for well-specified, misspecified ReLU regression, and the comparison between GLM-tron and SGD in Sections \ref{sec:well-specified}, \ref{sec:misspecified} and \ref{sec:sgd}, respectively.
We sketch our proof techniques in Section \ref{sec:proof}. 
Finally, the paper is concluded in Section \ref{sec:conclusion}.
All proofs are deferred to the appendix.

\section{Related Work}\label{sec:related}

\paragraph{ReLU Regression.}
We first review a set of literature on the hardness  results and achievable bounds for ReLU regression. 
On the negative side, \citet{Goel2020TightHR} showed that learning ReLU regression is NP-hard without distributional assumption.
Moreover,
\citet{goel2019time} showed that even for Gaussian features, learning ReLU regression with small \emph{excess} risk is as hard as the learning sparse parities with noise problem, which is believed to be computationally intractable.
On the positive side, \citet{frei2020agnostic} showed that under certain conditions (e.g., bounded and well-spread features),
GD or SGD can learn ReLU regression problems with $\min\risk(\cdot) +o(1)$ risk in the well-specified cases and $\Ocal( \min\risk(\cdot) +o(1) )$ risk in the misspecified cases.
Compared to \citet{frei2020agnostic}, our risk bounds for \eqref{eq:tron} are more general in both settings and can recover their bounds. 
For finite-dimensional misspecified ReLU regression, \citet{diakonikolas2022learning} showed that a constant-factor approximation is possible with only poly-logarithmic samples.
However, their result becomes vacuous in the overparameterized regime.
Finally, in a significantly easier, noiseless setting where $y = \relu(\wB_*^\top \xB)$ for some $\wB_*\in\Hbb$,
there are far more results (see, e.g., \citet{soltanolkotabi2017learning,du2017convolutional,yehudai2020learning,frei2020agnostic} and the references therein).
Although our results can be directly applied, the noiseless setting is not the main focus of our paper.

Tangibly related to ReLU regression, the problem of learning leaky ReLU regression has been studied by \citet{mei2018landscape,foster2018uniform,frei2020agnostic,yehudai2020learning}. 
Since ReLU is not a strictly increasing function (unlikely leaky ReLU), these results for leaky ReLU regression cannot be applied to ReLU regression.

Recent work by \citet{zhou2022non} provided dimension-free bounds on the generalization gap between the \emph{Moreau envelope} of the empirical and population loss for general GLMs including ReLU regression. 
But their analysis is limited to Gaussian data while our analysis imposes much fewer constraints on the data distribution.

\paragraph{GLM-Tron.}
The GLM-tron algorithm dates back to at least \citet{kalai2009isotron,kakade2011efficient} for learning the well-specified generalized linear model (GLM), where the expectation of the label conditioning on the feature is generated through a GLM.
As a special case, their results apply to well-specified ReLU regression as well.
However, our results are significantly different from theirs. 
First of all, in the well-specified regime, we show nearly matching upper and lower excess risk bounds for \eqref{eq:tron}, which can recover the excess risk upper bounds from \citet{kalai2009isotron,kakade2011efficient}.
Moreover, 
from a technical standpoint,
their analysis is motivated by the classical analysis for the perceptron algorithm (see, e.g., Section 4.1.7 in \citet{bishop2006pattern}),
while we take a completely different approach by analyzing \eqref{eq:tron} in ReLU regression with the operator methods developed for analyzing SGD in linear regression (see, e.g., \citet{zou2021benign,wu2022iterate}).
We refer the reader to Section \ref{sec:proof} for a detailed overview of our techniques.
On the other hand, we remark that our analysis is specialized to ReLU regression and may not directly apply to general GLMs covered by \citet{kalai2009isotron,kakade2011efficient}.

More recently, \citet{diakonikolas2020approximation} revisited GLM-tron for learning misspecified ReLU regression and showed a risk upper bound of 
\( \Ocal( \min \risk(\cdot) + \sqrt{d/N} ) \)
, where $d$ is the ambient dimension and $N$ is the sample size. 
Their bound becomes vacuous in the overparameterized regime. 
In comparison, our bound in the misspecified setting can be applied in the overparameterized setting. 
Moreover, when specialized to the finite-dimensional cases, our bound improves the bound in \citet{diakonikolas2020approximation}.

\section{Preliminaries}\label{sec:setup}
In this part, we set up some additional preliminaries before presenting our results. 
The following defines the data covariance matrix.
\begin{definition}[Data covariance matrix]\label{assump:data:covariance}
Assume that each entry and the trace of the $\Ebb[\xB \xB^\top]$ are finite.
Define 
\(
\HB := \Ebb [\xB \xB^\top].
\)
Denote the eigenvalues of $\HB$ by $(\lambda_i)_{i\ge 1}$, sorted in non-increasing order. 
\end{definition}

In what follows, we will make the following assumption about the symmetricity of the feature vector. 
\begin{assumption}[Symmetricity conditions]\label{assump:symmetric}
Assume that for every $\uB \in \Hbb$ and $\vB \in \Hbb$, it holds that 
\begin{align*}
     &\Ebb \big[\xB \xB^\top \cdot \ind{\xB^\top \uB > 0, \xB^\top \vB>0} \big] \\
     & \qquad = \Ebb \big[\xB \xB^\top \cdot \ind{\xB^\top \uB < 0, \xB^\top \vB<0} \big]; \\
      &\Ebb \big[ (\xB^\top \vB)^2 \xB \xB^\top \cdot \ind{\xB^\top \uB > 0, \xB^\top \vB>0} \big] 
      \\
      &\qquad = \Ebb \big[ (\xB^\top \vB)^2 \xB \xB^\top \cdot \ind{\xB^\top \uB < 0, \xB^\top \vB<0} \big].
\end{align*}
\end{assumption}
Assumption~\ref{assump:symmetric} requires that both the second and fourth moments of $\xB$, when projected into a sector, are invariant under sign flipping. 
Clearly, Assumption~\ref{assump:symmetric} holds when $\xB$ follows a symmetric distribution, i.e., $\xB$ and $-\xB$ satisfy the same distribution, which covers Gaussian or symmetric Bernoulli distributions.
We also remark that Assumption~\ref{assump:symmetric} can be slightly relaxed; see more discussions in Appendix \ref{append:sec:symmetricity}. 

Most existing results for ReLU regression impose some distributional conditions on the feature vectors. For example, \citet{frei2020agnostic,yehudai2020learning} assumed that the p.d.f.\ of $\xB$ is ``well-spreaded'' along every two-dimensional projection.
\citet{diakonikolas2020approximation,diakonikolas2022learning} assumed concentration and anti-concentration (and anti-anti-concentration) conditions on $\xB$. 
Our Assumption~\ref{assump:symmetric} only involves up to the fourth moments of $\xB$ and is not directly comparable to theirs that involve the entire p.d.f.\ of $\xB$.

\paragraph{Notation.}
We reserve upper-case calligraphic letters for linear operators on symmetric matrices. 
For two positive-value functions $f(x)$ and $g(x)$ we write  $f(x)\lesssim g(x)$ or $f(x)\gtrsim g(x)$
if $f(x) \le cg(x)$ or $f(x) \ge cg(x)$ for some absolute constant 
$c>0$ respectively;
we write $f(x) \eqsim g(x)$ if $f(x) \lesssim g(x) \lesssim f(x)$.
For two vectors $\uB$ and $\vB$ in a Hilbert space, their inner product is denoted by $\abracket{\uB, \vB}$ or equivalently, $\uB^\top \vB$.
For a matrix $\AB$, its spectral norm is denoted by $\norm{\AB}_2$.
For two matrices $\AB$ and $\BB$ of appropriate dimension, their inner product is defined as $\langle \AB, \BB \rangle := \tr(\AB^\top \BB)$.
For a positive semi-definite (PSD) matrix $\AB$ and a vector $\vB$ of appropriate dimension, we write $\norm{\vB}_{\AB}^2 := \vB^\top \AB \vB$.
The Kronecker/tensor product is denoted by $\otimes$.
Moreover, $\log (\cdot )$ refers to logarithm base $2$.

Denote the eigen decomposition of the data covariance by
\(\HB = \sum_{i} \lambda_i \vB_i \vB_i^\top\), where $(\lambda_i)_{i\ge 1}$ are eigenvalues in a non-increasing order and $(\vB_i)_{i\ge 1}$ are the corresponding eigenvectors.
We denote $\HB_{k^*:k^\dagger} := \sum_{k^* < i \le k^\dagger} \lambda_i \vB_i \vB_i^\top$, where $0\le k^* \le k^\dagger$ are two integers, and we allow $k^\dagger = \infty$. For example,
\[
\HB_{0:k} = \sum_{1\le i \le k} \lambda_i \vB_i \vB_i^\top,\quad 
\HB_{k:\infty} = \sum_{i > k} \lambda_i \vB_i \vB_i^\top.
\]
Similarly, we denote $\IB_{k^*:k^\dagger} := \sum_{k^* < i \le k^\dagger} \vB_i \vB_i^\top$.

\section{Well-Specified ReLU Regression}\label{sec:well-specified}
In this part, we present our results for well-specified ReLU regression. 
In the literature, the well-specified setting is also extensively referred to as the ``noisy teacher'' setting \citep{frei2020agnostic}.
We formally define a well-specified noise as follows.
\begin{assumption}[Well-specified noise]\label{assump:noise:well-specified}
Assume that there exists a parameter $\wB_*\in \Hbb$ such that 
\[
\Ebb[ y  | \xB] = \relu(\xB^\top \wB_*).
\]
Moreover, denote the variance of the additive noise by
\[
\sigma^2:= \risk(\wB_*) = \Ebb[ (y - \relu(\xB^\top \wB_*))^2].
\]
\end{assumption}
Clearly, in the well-specified case, 
we have 
\begin{align*}
    \risk(\wB) = \risk(\wB_*) + \Ebb[ (\relu(\xB^\top\wB) - \relu(\xB^\top \wB_*))^2],
\end{align*}
which implies that $\wB^* \in \arg\min \risk(\cdot)$.
In this case, we will work with the \emph{excess risk}, defined by
\begin{equation}\label{eq:excess-risk}
\excessrisk(\wB) := \risk(\wB) - \risk(\wB_*). 
\end{equation}

\paragraph{Excess Risk Landscape.}
Our first observation is that the landscape of the excess risk \eqref{eq:excess-risk} in ReLU regression is closely related to that in linear regression, i.e., a quadratic landscape. 
The following lemma rigorously characterizes this connection.
\begin{lemma}[Excess risk landscape]\label{thm:loss-landscape}
Under Assumptions \ref{assump:symmetric} and \ref{assump:noise:well-specified}, the following holds for \eqref{eq:excess-risk}:
\[
 0.25 \cdot \|\wB - \wB_*\|_{\HB}^2  \le \excessrisk(\wB) \le \| \wB - \wB_*\|_{\HB}^2.
\]
\end{lemma}
Even though the excess risk \eqref{eq:excess-risk} could be non-convex locally, Lemma~\ref{thm:loss-landscape} suggests that the landscape of the excess risk in a large scale is ``approximately'' quadratic in the sense of ignoring some multiplicative factors. 
This landscape enables us to build sharp upper and lower bounds on the excess risk by bounding a simpler quadratic, $\|\wB - \wB_*\|^2_{\HB}$.

\paragraph{Operators.}
We follow \citet{zou2021benign,wu2022iterate} and introduce some matrix operators for applying the operator methods  to analyze \eqref{eq:tron}.
Firstly, we denote the covariance of the~\eqref{eq:tron} iterates by
\begin{equation}\label{eq:tron:cov}
        \AB_t := \Ebb ( \wB_{t} - \wB_*  )( \wB_{t} - \wB_*  )^{\top},\quad t\ge 0.
\end{equation}
We next define a set of linear operators on the matrix space:
\begin{equation}\label{eq:operators}
    \begin{gathered}
    \Ical := \IB \otimes \IB,\quad 
    \Mcal := \Ebb [\xB^{\otimes 4} ],\quad 
    \widetilde{\Mcal}:= \HB\otimes \HB,\\
    \Tcal(\gamma) :=  \IB\otimes \HB + \HB \otimes \IB - \gamma \cdot \Mcal,  \\
    \widetilde{\Tcal}(\gamma) := \IB\otimes \HB + \HB \otimes \IB - \gamma \cdot \widetilde{\Mcal}.
\end{gathered}
\end{equation}

\paragraph{A Key Lemma.}
The next lemma is the key to our analysis, which relates the covariance of a sequence of \eqref{eq:tron} iterates for a ReLU regression problem
with the covariance of a sequence of ``imaginary'' SGD iterates for an ``imaginary'' linear regression problem. 
\begin{lemma}[Generic bounds on the GLM-tron iterates]\label{thm:tron:covariance}
    Under Assumptions \ref{assump:symmetric} and \ref{assump:noise:well-specified}, the following holds for \eqref{eq:tron:cov}:
    \begin{enumerate}[label=(\Alph*),nosep]
        \item \(
        \AB_{t+1} \preceq \big(\Ical - \frac{\gamma_t}{2}\cdot \Tcal(2\gamma_t)\big) \circ  \AB_t  +  \gamma_t^2 \sigma^2 \cdot \HB;
  \)
  \item \(\AB_{t+1} \succeq \Big(\Ical - \frac{\gamma_t}{2}\cdot \Tcal\big(\frac{\gamma_t}{2}\big)\Big) \circ  \AB_t  +  \frac{\gamma_t^2 \sigma^2}{4} \cdot \HB.\)
    \end{enumerate}
\end{lemma}
In the remaining part of this section, we will derive sharp risk bounds for \eqref{eq:tron} in high-dimensional ReLU regression based on Lemma~\ref{thm:tron:covariance} and the results for SGD in high-dimensional linear regression developed by \citet{zou2021benign,zou2021benefits,wu2022iterate,wu2022power}.

\subsection{Symmetric Bernoulli Distributions}

In order to gain intuitions on the behaviors of \eqref{eq:tron}, we start with a simple, symmetric Bernoulli  data model defined as follows. 
Note that this is just the symmetrization of the one-hot data model considered in \citet{zou2021benefits}.
\begin{assumption}
[Symmetric Bernoulli distribution]\label{assump:bernoulli}
Let $(\eB_i)_{i\ge 1}$ be a set of orthogonal basis for $\Hbb$.
    Assume that
    \(\Pbb\{\xB = \eB_i\} = \Pbb \{\xB = -\eB_i\} = \lambda_i/2\) for $i\ge1$, 
    where $\lambda_i \ge 0$ and $\sum_i \lambda_i = 1$.
\end{assumption}
Clearly Assumption~\ref{assump:bernoulli} implies Assumption~\ref{assump:symmetric}.
We now present our instance-wise sharp excess risk bounds for \eqref{eq:tron} under Assumption~\ref{assump:bernoulli}.

\begin{theorem}[Risk Bounds for GLM-tron]\label{thm:tron:bernoulli}
Suppose that Assumptions \ref{assump:noise:well-specified} and \ref{assump:bernoulli} hold.
Let $\wB_{N}$ be the output of \eqref{eq:tron} with stepsize scheduler \eqref{eq:geometry-tail-decay-lr}.
Assume that $N > 100$. 
Let $\Neff := N / \log(N)$.
Suppose that $ \gamma_0 < 1/2$. 
\begin{enumerate}[label=(\Alph*),nosep,leftmargin=*]
    \item
For every $k^* \ge 0$ it holds that
\begin{align*}
    \Ebb \excessrisk (\wB_{N})
    &\lesssim 
    \big\| \wB_0 - \wB_* \big\|^2_{\prod_{t=1}^{N}(\IB-\frac{\gamma_t}{2}\HB)\HB} + \sigma^2 \cdot \frac{\Deff}{\Neff},
\end{align*}
where $\Deff$ is defined by
\begin{equation}\label{eq:effective-dim}
    \Deff := k^* + \Neff^2 \gamma^2_0 \cdot \sum_{i>k^*} \lambda_i^2.
\end{equation}
\item
For $\Deff$ defined by \eqref{eq:effective-dim} with
\begin{equation}\label{eq:opt-index-sets}
    k^* := \max\{k: \lambda_k \ge 1/(\gamma_0 \Neff)\},
\end{equation}
it holds that
\begin{align*}
   \Ebb \excessrisk (\wB_{N})
    &\gtrsim 
    \big\| \wB_0 - \wB_* \big\|^2_{\prod_{t=1}^{N}(\IB-\gamma_t\HB)\HB} + 
     \sigma^2  \cdot  \frac{\Deff}{\Neff}.
\end{align*}
\end{enumerate}
\end{theorem}
\begin{proof}[Proof Sketch]
    We first use Lemmas \ref{thm:loss-landscape} and \ref{thm:tron:covariance} to relate \eqref{eq:tron} for ReLU regression problems to SGD for linear regression problems. Then we invoke the one-hot analysis in \citet{zou2021benefits} to get the results.
\end{proof}

\subsection{Hypercontractive Distributions}
We are ready to present our results for the more interesting distributions  that satisfy the \emph{hypercontractivity} conditions. 
\begin{assumption}[Hypercontractivity conditions]\label{assump:gaussian}
Assume that the fourth moment of $\xB$ is finite and:
\begin{enumerate}[label=(\Alph*),nosep]
    \item \label{item:gaussian:upper} There is a constant $\alpha > 0$, such that for every PSD matrix $\AB$, we have
\begin{equation*}
\Ebb [\xB \xB^\top \AB \xB \xB^\top] \preceq \alpha \cdot \tr (\HB \AB) \cdot \HB.
\end{equation*}
Clearly, it must hold that $\alpha \ge 1$. 
\item \label{item:gaussian:lower} There is a constant $\beta > 0$, such that for every PSD matrix $\AB$, we have
\[
\Ebb [\xB \xB^\top \AB \xB \xB^\top] - \HB \AB \HB   \succeq  \beta \cdot \tr (\HB \AB) \cdot \HB.
\]
\end{enumerate}
\end{assumption}
One can verify that Assumption~\ref{assump:gaussian} holds with $\alpha=3$ and $\beta=1$ when $\xB \sim \Ncal(0, \HB)$.
Moreover, Assumption~\ref{assump:gaussian}\ref{item:gaussian:upper} holds when $\HB^{-1/2}\xB$ is sub-Gaussian or sub-Exponential
and Assumption~\ref{assump:gaussian}\ref{item:gaussian:lower} holds when $\HB^{-1/2}\xB$ follows a multi-dimensional spherically symmetric distribution \citep{zou2021benign,wu2022iterate}.
For more examples of Assumption~\ref{assump:gaussian} we refer the readers to \citet{zou2021benign,wu2022iterate}.

\begin{theorem}[Risk Bounds for GLM-tron]\label{thm:tron:gaussian}
Suppose that Assumptions \ref{assump:symmetric} and \ref{assump:noise:well-specified} hold.
Let $\wB_{N}$ be the output of \eqref{eq:tron} with stepsize scheduler \eqref{eq:geometry-tail-decay-lr}.
Assume that $N > 100$. 
Let $\Neff := N / \log(N)$.
\begin{enumerate}[label=(\Alph*),nosep]
    \item
If in addition Assumption~\ref{assump:gaussian}\ref{item:gaussian:upper} holds,
then for $\gamma_0 < 1/(4\alpha(\tr(\HB)) )$ it holds that
\begin{align*}
   & \Ebb \excessrisk (\wB_{N})
    \lesssim 
    \bigg\| \prod_{t=1}^{N}\Big(\IB-\frac{\gamma_t}{2}\HB\Big)(\wB_0 - \wB_*) \bigg\|^2_\HB \\
    &\quad + \Big( \alpha \big\| \wB_0 - \wB_* \big\|^2_{\frac{\IB_{0:k^*} }{\Neff \gamma_0} + \HB_{k^*:\infty}} + \sigma^2 \Big) \cdot \frac{\Deff}{\Neff},
\end{align*}
where $\Deff$ is defined by \eqref{eq:effective-dim} and $k^*\ge 0$ is arbitrary.
\item
If in addition Assumption~\ref{assump:gaussian}\ref{item:gaussian:lower} holds, then
for $\gamma_0 < 1/\lambda_1$,
it holds that
\begin{align*}
    \Ebb \excessrisk (\wB_{N})
    &\gtrsim 
    \bigg\| \prod_{t=1}^{N}\Big(\IB-\frac{\gamma_t}{2}\HB\Big)(\wB_0 - \wB_*) \bigg\|^2_\HB  \\
    &\quad + 
     \big( \beta \|\wB_0 - \wB_* \|^2_{\HB_{k^*:\infty}} + \sigma^2 \big) \cdot  \frac{\Deff}{\Neff},
\end{align*}
where $\Deff$ is defined by \eqref{eq:effective-dim} and $k^*$ is defined by \eqref{eq:opt-index-sets}.
\end{enumerate}
\end{theorem}
\begin{proof}[Proof Sketch]
    We first use Lemmas \ref{thm:loss-landscape} and  \ref{thm:tron:covariance} to relate \eqref{eq:tron} for ReLU regression problems to SGD for linear regression problems. Then we invoke Corollary 3.4 in \citet{wu2022power} to get the results.
\end{proof}

These bounds in Theorem \ref{thm:tron:gaussian} match those of SGD for high-dimensional linear regression shown in \citet{wu2022power} (and also \citet{zou2021benign,wu2022iterate}) and can be interpreted in a similar manner.
Specifically,
in the upper bound, the first error term shows that $\wB_N$ recovers the true model parameter geometrically at each dimension and is at most 
\[
\|\wB_0 - \wB_* \|_2^2 / (\gamma_0 \Neff),
\]
and the second error term is at most 
\[
\Big( \alpha \big\| \wB_0 - \wB_* \big\|^2_2 + \sigma^2 \Big) \cdot \frac{\Deff}{\Neff}.
\]
Provided a bounded signal-to-noise ratio and a constant initial stepsize (which might not be optimal), the expected risk decreases at a rate of $\Ocal(\Deff / \Neff)$. 
Moreover, the lower bound justifies the sharpness of the upper bound.

We remark that $\Deff$ is independent of the ambient dimension, and is small so long as the spectrum of $\HB$ decays fast. This enables \eqref{eq:tron} to achieve a small excess risk even in the overparameterized regime.

The following corollary provides three concrete examples.
\begin{mycorollary}\label{thm:examples}
Under the same conditions as Theorem \ref{thm:tron:gaussian}, suppose that $\gamma_0 = 1/(4 \alpha\tr(\HB))$, and $\|\wB_0 - \wB^* \|_2 $ is finite. Recall the eigenspectrum of $\HB$ is $(\lambda_k)_{k \ge 1}$.
\begin{enumerate}
    \item If $\lambda_k = k^{-(1+r)}$ for some constant $r > 0$, then the excess risk is $\Ocal\big(   N^{\frac{-r}{1+r}}\cdot \log^{\frac{r}{1+r}} (N) \big)$.
    \item If $\lambda_k = k^{-1} \log^{-r}(k+1)$ for some constant $r > 1$, then the excess risk is $\Ocal\big( \log^{-r} (N) \big)$.
    \item If $\lambda_k = 2^{-k}$, then the excess risk is $\Ocal\big( N^{-1}  \log^2 (N) \big)$.
\end{enumerate}
\end{mycorollary}

\paragraph{Iterate Averaging.}
Theorem \ref{thm:tron:gaussian} focuses on the last iterate of \eqref{eq:tron} with decaying stepsize~\eqref{eq:geometry-tail-decay-lr}. 
We remark that this theorem can also be extended to constant stepsize GLM-tron with iterate averaging. See Theorem \ref{append:thm:tron:iterate-avg} in Appendix \ref{append:sec:iterate-avg}, where we show matching upto constant factor upper and lower risk bounds for constant stepsize GLM-tron with iterate averaging.
It is proved similarly by invoking Lemmas \ref{thm:loss-landscape} and \ref{thm:tron:covariance} and related results from \citet{zou2021benign}.

\paragraph{Applications in the Classical Regime.}
In the next corollary,  we apply our instance-dependent risk bounds to the classical regime, i.e., finite dimension and/or bounded $\ell_2$-norm.

\begin{corollary}[Classical regime]\label{thm:well-specified:recover}
Under the setting of Theorem \ref{thm:tron:gaussian}, in addition assume that \( \sigma^2\lesssim 1,\ \|\wB_0 - \wB_*\|_2 \lesssim 1,\ \lambda_1 \lesssim 1.\)
We then have the following:
\begin{enumerate}[label=(\Alph*),nosep]
    \item \label{item:well-specified:slow-rate}
If $\tr(\HB)\lesssim 1$, then 
by choosing 
\( \gamma_0 \eqsim 1/{\sqrt{\Neff}}\) and \( k^* := \max\{k: \lambda_k \ge 1/\sqrt{\Neff}  \}\),
we have 
\[\Ebb \Delta(\wB_N) \lesssim \frac{1}{\sqrt{\Neff}} = \sqrt{\frac{\log(N)}{N}}.\]
\item \label{item:well-specified:fast-rate}
If $d$ is finite, then by choosing 
\(\gamma_0 \eqsim 1/\tr(\HB)\) and \(k^* = d,\)
we have 
\[\Ebb \Delta(\wB_N) \lesssim \frac{d}{\Neff} = \frac{d\log(N)}{N}.\]
\end{enumerate}
It is worth remarking that the $\log(N)$ factors in the above rates can be removed when considering constant-stepsize GLM-tron with iterate-averaging (see Theorem \ref{append:thm:tron:iterate-avg}). 
\end{corollary}

In Corollary \ref{thm:well-specified:recover}, the condition  $\|\wB_0 -\wB_*\|_2 \lesssim 1$  corresponds to the bounded $\ell_2$-norm condition of $\wB_*$ made in \citet{kakade2011efficient,frei2020agnostic} (by taking initialization $\wB_0 = 0$). 
The condition $\tr(\HB) \lesssim 1$ corresponds to the bounded $\ell_2$-norm condition of features made in \citet{kakade2011efficient,frei2020agnostic} (because $\Ebb[\|\xB\|_2^2 ] = \tr(\HB)$). 
Then Corollary \ref{thm:well-specified:recover}\ref{item:well-specified:slow-rate} matches the $\tilde{\Ocal}(1/\sqrt{N})$ rate for GLM-tron in \citet{kakade2011efficient}, and nearly matches the $\Ocal(1/\sqrt{N})$ rate for GD in \citet{frei2020agnostic}.
Corollary \ref{thm:well-specified:recover}\ref{item:well-specified:fast-rate} shows a faster $\tilde{\Ocal}(d/N)$ rate in the finite-dimensional regime.

\section{Misspecified ReLU Regression}\label{sec:misspecified}

In this part, we present our results for misspecified ReLU regression.
This setting is also known as the agnostic setting in literature \citep{goel2019time,diakonikolas2020approximation}.
We first define a misspecified noise as follows.

\begin{assumption}[Misspecified noise]\label{assump:noise:misspecified}
Denote the minimum population risk by 
\[\OPT := \min_{\wB'\in\Hbb} \risk(\wB').\]
Moreover, assume that there exists an optimal model parameter $\wB^* \in \arg\min_{\wB'\in\Hbb} \risk(\wB')$ such that 
    \begin{equation}\label{eq:misspecified:noise}
        \Ebb \big[ (y - \relu(\xB^\top\wB_*))^2\xB \xB^\top \big] \preceq \sigma^2\cdot \HB 
    \end{equation}
holds for some constant $\sigma^2 > 0$.
\end{assumption}

Different from the well-specified case,
Assumption~\ref{assump:noise:misspecified} does not directly impose any probability condition on the label-generating process. 
In particular, it captures the situation when $1-\OPT$ fraction of the label is generated without noise while the rest $\OPT$ fraction of the label is \emph{adversarially} given \citep{diakonikolas2020approximation}.

Moreover, we empathize that the condition \eqref{eq:misspecified:noise} in Assumption~\ref{assump:noise:misspecified} is very weak and conservative.  
In particular condition \eqref{eq:misspecified:noise} holds trivially when $y$ is bounded, $\|\wB_*\|_{\HB}$ is finite and $\xB$ satisfies the hypercontractivity condition in Assumption~\ref{assump:gaussian}\ref{item:gaussian:upper}, because:
\begin{align*}
    \text{l.h.s.\ of \eqref{eq:misspecified:noise}} &\preceq  2\Ebb \big[ y^2 \xB \xB^\top \big]
    +  2\Ebb \big[ ( \xB^\top\wB_*)^2 \xB \xB^\top \big] \\
    &\preceq \big(2 (\sup\{y\})^2 + 2\alpha \|\wB_*\|_{\HB}^2 \big)\cdot\HB. 
\end{align*}
The above requirements on $y$, $\wB_*$ and $\xB$ are already weaker than that required in the literature for learning miss-specified ReLU regression \citep{frei2020agnostic,diakonikolas2020approximation,goel2019time}.

In the misspecified setting, the label can correlate with data in an arbitrary manner. 
This breaks our nice Lemma~\ref{thm:tron:covariance} proved in the well-specified setting.
In order to analyze \eqref{eq:tron} in the misspecified setting, we extend the operator methods from considering PSD matrices to considering only the \emph{diagonals} of PSD matrices (see Section \ref{sec:proof} for more discussions).
With the new techniques, we obtain the following instance-dependent risk bound.

\begin{theorem}[Risk Bounds for GLM-tron]\label{thm:tron:gaussian:misspecified}
Suppose that Assumptions \ref{assump:symmetric}, \ref{assump:gaussian}\ref{item:gaussian:upper} and \ref{assump:noise:misspecified}  hold.
Let $\wB_{N}$ be the output of \eqref{eq:tron} with stepsize scheduler \eqref{eq:geometry-tail-decay-lr}.
Assume that $N > 100$. 
Let $\Neff := N / \log(N)$.
Then for $\gamma_0 < 1/(8\alpha(\tr(\HB)) )$, 
it holds that
\begin{align*}
   \Ebb \risk (\wB_{N})
    & \lesssim \OPT + 
    \bigg\| \prod_{t=1}^{N}\Big(\IB-\frac{\gamma_t}{2}\HB\Big)(\wB_0 - \wB_*) \bigg\|^2_\HB  \\
    & \quad + (1+\SNR)\cdot\sigma^2 \cdot  \frac{\Deff}{\Neff},
\end{align*}
where $\Deff$ is defined by \eqref{eq:effective-dim}, $k^* \ge 0$ is arbitrary, and 
\begin{align*}
    &\quad \ \SNR \\
    &:= {\alpha\Big(\OPT +\|\wB_*\|_{\HB}^2 + \big\| \wB_0 - \wB_* \big\|^2_{\frac{\IB_{0:k^*} }{\Neff \gamma_0} + \HB_{k^*:\infty}} \Big) } /  {\sigma^2} \\
    &\le \alpha (\OPT + \|\wB_*\|_{\HB}^2 + \|\wB_0 - \wB_*\|_{\HB}^2) / \sigma^2.
\end{align*}
\end{theorem}

Similar to the well-specified setting, 
Theorem \ref{thm:tron:gaussian:misspecified} allows \eqref{eq:tron} to achieve a constant-factor approximation even in the overparameterized regime, as long as the spectrum of $\HB$ decays fast such that $\Deff$ is small compared to $\Neff$.

\paragraph{Applications in the Finite-Dimensional Regime.}
The next corollary shows that, when applied to the finite-dimensional regime, our bound improves an existing bound, $\Ocal(\OPT + \sqrt{d/N})$, of GLM-tron for misspecified ReLU regression proved by \citet{diakonikolas2020approximation}.

\begin{corollary}[Finite-dimensional regime]\label{coro:gtron:gaussianLmisspecified}
Under the setting of Theorem \ref{thm:tron:gaussian:misspecified}, in addition assume that $d$ is finite and
\begin{align*}
 \sigma^2 \lesssim 1, \ 
 \|\wB_0-\wB_*\|_2\lesssim 1,\ \|\wB_*\|_2 \lesssim 1,\ \lambda_1 \lesssim 1.
\end{align*}
Then by choosing $\gamma_0 \eqsim 1/\tr(\HB)$ and $k^*=d$, we have
\begin{align*}
\Ebb \risk (\wB_N) \lesssim \OPT + \frac{d}{\Neff} = \OPT + \frac{d\log(N)}{N}.
\end{align*}
\end{corollary}

\section{Comparing GLM-tron with SGD}\label{sec:sgd}

In this part, we show some negative results for \eqref{eq:sgd} in ReLU regression with symmetric Bernoulli data.

\paragraph{Well-Specified Case.}
We first consider well-specified ReLU regression with symmetric Bernoulli data.
We provide the following risk lower bound for \eqref{eq:sgd}.
\begin{theorem}[Risk lower bound for SGD]\label{thm:sgd:bernoulli:lb}
Suppose that Assumptions \ref{assump:noise:well-specified}  and \ref{assump:bernoulli} hold.
Let $\wB_{N}$ be the output of \eqref{eq:sgd} with stepsize scheduler \eqref{eq:geometry-tail-decay-lr}. 
Assume that $N > 100$. 
Let $\Neff := N / \log(N)$.
Then for $ \gamma_0 < 1$, it holds that 
\begin{align*}
    \Ebb \excessrisk (\wB_{N})
    &\gtrsim 
    \big\| \wB_0 - \wB_* \big\|^2_{\prod_{t=1}^{N}(\IB-\gamma_t\HB)\HB} + 
     \sigma^2 \cdot\frac{\Deff}{\Neff} + \Psi,
\end{align*}
where $\Deff$ is defined by \eqref{eq:effective-dim} with $k^*$ defined by \eqref{eq:opt-index-sets},
and
\begin{equation*}
    \Psi := \Big\la\sum_{t=0}^{N-1} \gamma_t(1-\gamma_t)\prod_{k=t+1}^{N-1}(1-\gamma_k \HB)\HB,\ \FB_t \Big\ra
\end{equation*}
and $\FB_t \succeq 0$ is a PSD matrix. 
\end{theorem}

The excess risk lower bound for \eqref{eq:sgd} in Theorem \ref{thm:sgd:bernoulli:lb}
is in sharp contrast to the excess risk upper bound for \eqref{eq:tron} in Theorem \ref{thm:tron:bernoulli}:
the bias and variance error lower bounds for \eqref{eq:sgd} is comparable to the bias and variance error upper bounds for \eqref{eq:tron}; in addition, there is an extra non-negative error term $\Psi$ for \eqref{eq:sgd}.
This seems to suggest that \eqref{eq:sgd} is no better than \eqref{eq:tron}.
Our next theorem formalizes this observation. 

\begin{figure*}
    \centering
\subfigure[$\lambda_i\propto i^{-2}$, well-specified setting\label{fig:hd1}]{\includegraphics[width=0.3\linewidth]{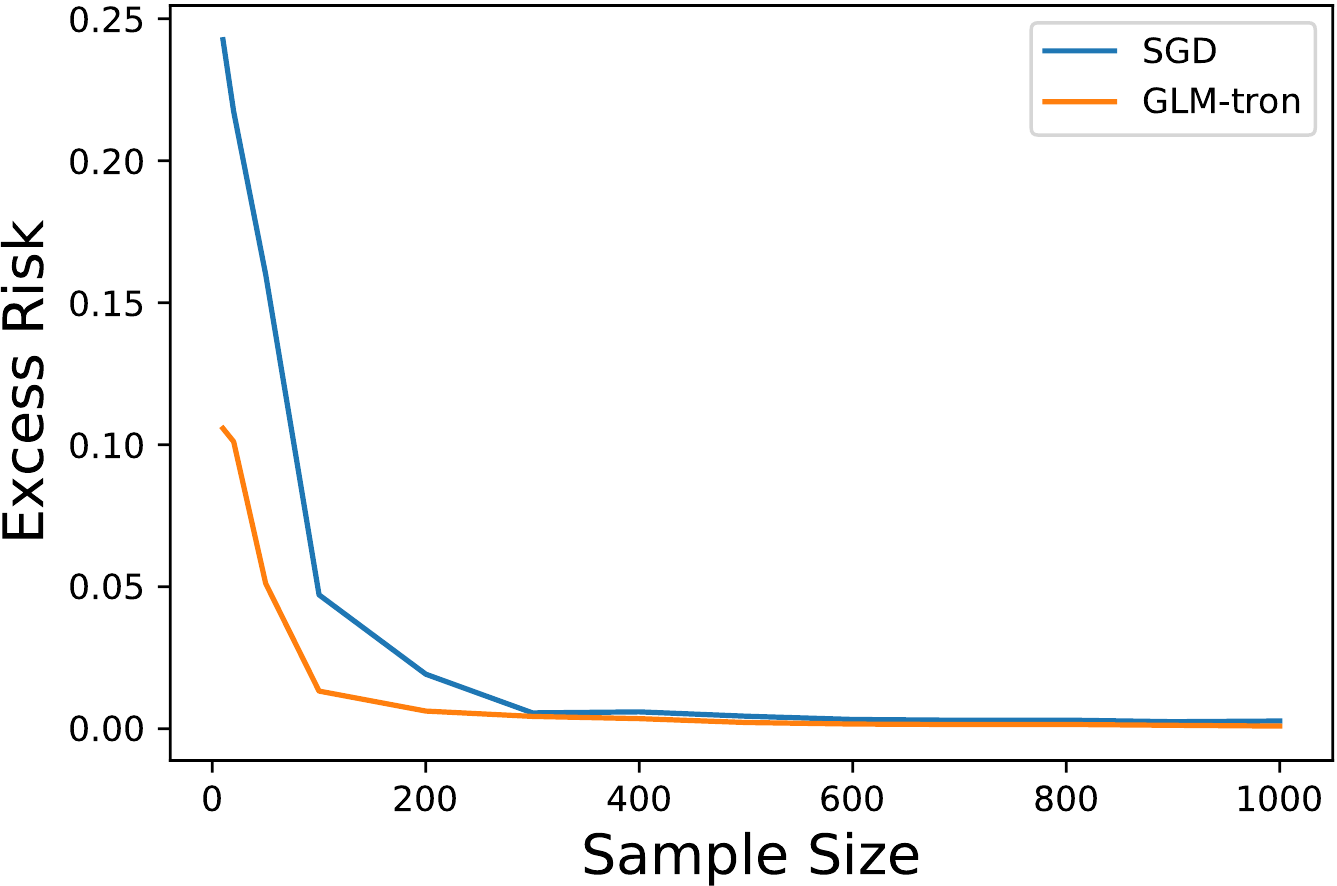}}
\subfigure[$\lambda_i\propto i^{-3}$, well-specified setting\label{fig:hd2}]{\includegraphics[width=0.3\linewidth]{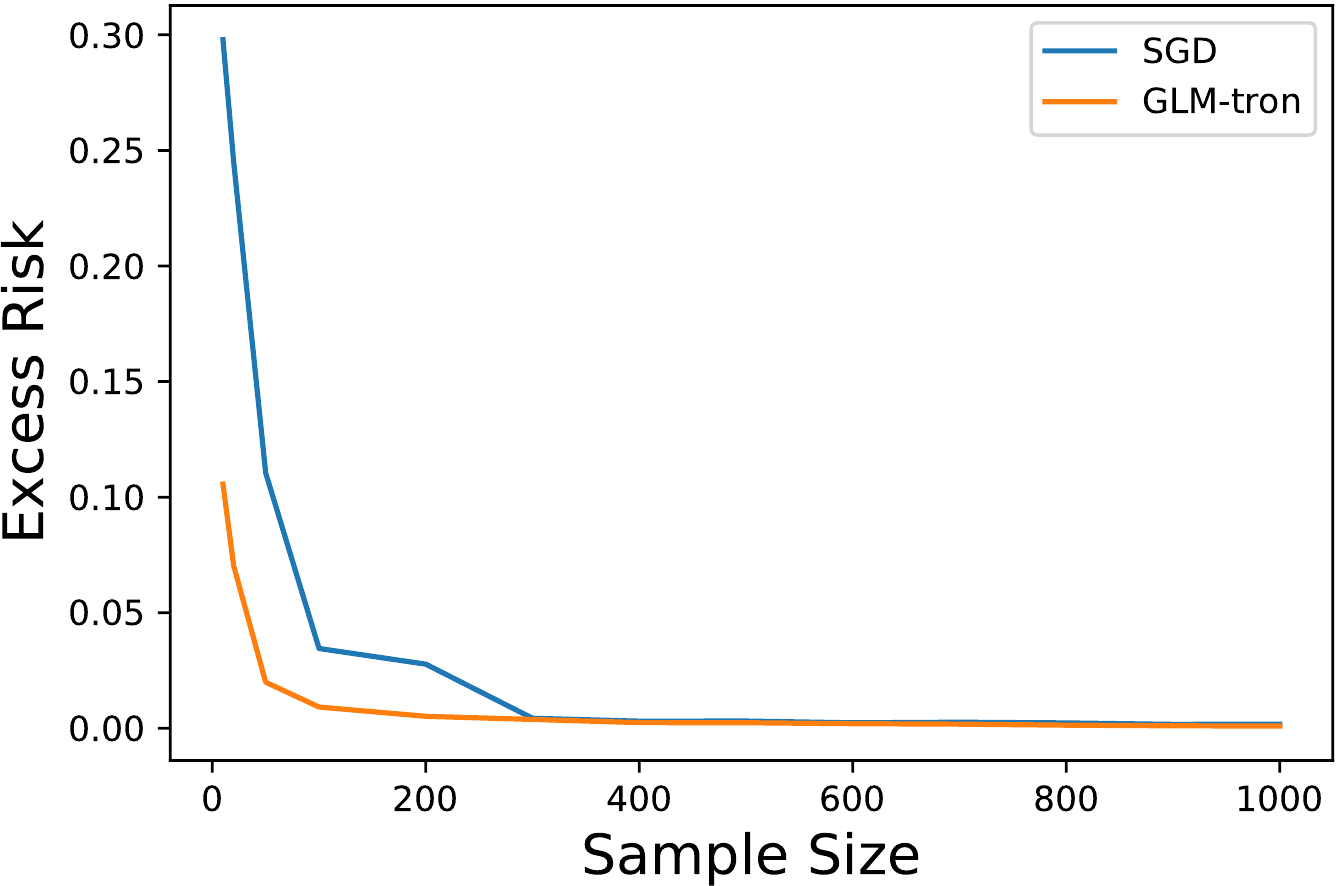}}
    \subfigure[2D illustration, noiseless setting\label{fig:2d}] {\includegraphics[width=0.31\linewidth]{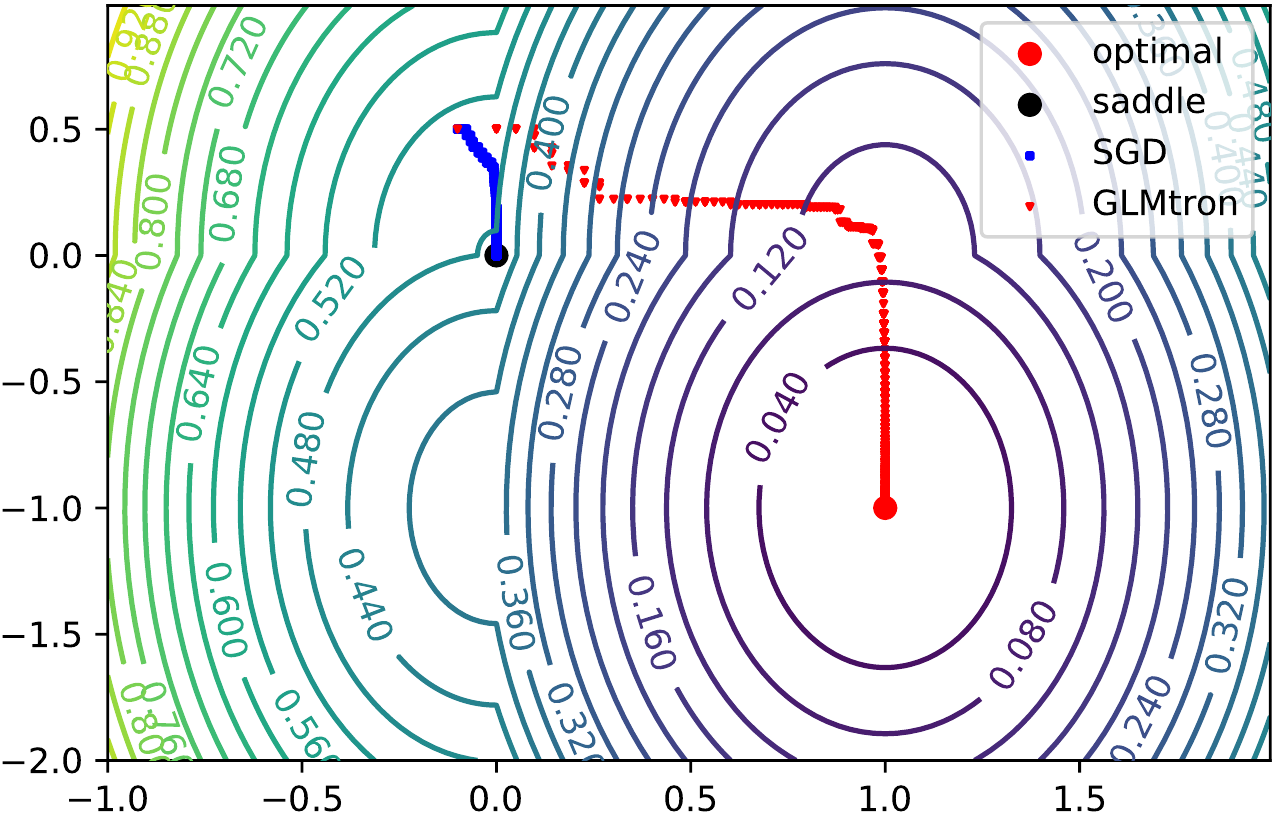}}
    \caption{\small
    (a) and (b): Excess risk comparison between \eqref{eq:sgd} and \eqref{eq:tron} in well-specified ReLU regression with symmetric Bernoulli data. Here $d=1,024$, $\sigma^2=0.01$ and $\wB_*=(i^{-1})_{i= 1}^d$.
    The eigen spectrum is $\lambda_i \propto i^{-2}$ and $\lambda_i\propto i^{-3}$ for (a) and (b), respectively.
    For each algorithm and each sample size, we do a grid search on the initial stepsize $\gamma_0 \in \{0.5, 0.25, 0.1, 0.075, 0.05, 0.025, 0.01\}$ and report the best excess risk. The plots are averaged over $20$ independent runs.
    (c): Training trajectories of \eqref{eq:sgd} and \eqref{eq:tron} on a 2D noiseless ReLU regression with symmetric Bernoulli data. Here $(\lambda_1, \lambda_2) = (0.8, 0.2)$ and $\wB_* = (1,-1)$.
}
    \label{fig:bernoulli}
\end{figure*}

\begin{theorem}[GLM-tron vs. SGD]\label{thm:tron-vs-sgd:every-case}
Fix an initialization $\wB_0$.
Consider a set of well-specified ReLU regression problems with symmetric Bernoulli data (denoted by $\Escr$) such that:
 Assumption~\ref{assump:noise:well-specified} and \ref{assump:bernoulli} hold and $\|\wB_0 - \wB_*\|_2^2 \lesssim \sigma^2$.
 Let $\wB^{\sgd}_{N}(\gamma^{\sgd}_0, \Pcal)$ and $\wB^{\tron}_{N}(\gamma^{\tron}_0, \Pcal)$ be the outputs of \eqref{eq:sgd} and \eqref{eq:tron} with the same stepsize scheduler \eqref{eq:geometry-tail-decay-lr}, initialization $\wB_0$, sample size $N>100$, and on the same problem instance $\Pcal \in \Escr$, respectively, where $\gamma^{\sgd}_0< 1$ and $\gamma^{\tron}_0<1/2$ denote their initial stepsizes, respectively.
 Then  for every problem $\Pcal \in \Escr$, it holds that
\begin{align*}  
& \min_{\gamma^{\tron}_0<1/2} \Ebb \excessrisk\big(\wB_N^{\tron}(\gamma^{\tron}_0, \Pcal)\big) \\
&\hspace{20mm} \lesssim \min_{\gamma^{\sgd}_0<1}\Ebb \excessrisk\big(\wB_N^{\sgd}(\gamma^{\sgd}_0, \Pcal)\big). 
\end{align*}
\end{theorem}

This theorem shows that for every problem instance in $\Ecal$, the excess risk achieved by \eqref{eq:sgd} is no better than that achieved by \eqref{eq:tron} ignoring constant factors. 

\paragraph{Noiseless Case.}
Our final result shows that for the noiseless ReLU regression with symmetric Bernoulli data, \eqref{eq:sgd} unavoidably suffers from a constant risk in expectation, while \eqref{eq:tron} can still obtain a small risk.

\begin{theorem}[Failure of SGD]\label{thm:tron-vs-sgd:bad-case}
Consider a noiseless ReLU regression problem with symmetric Bernoulli data, i.e., Assumptions \ref{assump:noise:well-specified} and \ref{assump:bernoulli} hold with $\sigma^2 = 0$.
Let $\Ebb_{\wB_*}$ denote the expectation over the randomness of flipping the sign in each component of $\wB_*$ uniformly and let $\Ebb_{\algo}$ denote the expectation over the randomness of an algorithm.
Let $N>100$ be the sample size. 
Then: 
\begin{enumerate}[label=(\Alph*), nosep]
    \item For $\wB^{\tron}_N$,  the \eqref{eq:tron} output  with stepsize scheduler \eqref{eq:geometry-tail-decay-lr} and initial stepsize $\gamma_0 < 1/2$, it holds that
\[\Ebb_{\wB_*} \Ebb_{\algo} \risk(\wB_N^{\tron}) \lesssim  \big\| \wB_0 - \wB_* \big\|^2_{\prod_{t=1}^{N}\big( \IB-\frac{\gamma_t}{2}\HB \big)\HB}.\]
\item For $\wB_N^{\sgd}$, the \eqref{eq:sgd} output with stepsize scheduler \eqref{eq:geometry-tail-decay-lr} and any initial stepsize $\gamma_0 < 1$, it holds that
\[ \Ebb_{\wB_*} \Ebb_{\algo}  \risk(\wB_N^{\sgd}) \ge \half \cdot \| \wB_*\|^2_{\HB}\ge \half \cdot \risk(0) .\]
\end{enumerate}
\end{theorem}

\paragraph{Simulations.} 
Furthermore, we empirically compare the performance of \eqref{eq:tron} and \eqref{eq:sgd} for ReLU regression with symmetric Bernoulli data.
Simulation results are presented in Figure \ref{fig:bernoulli}.
In the well-specified setting,  Figures \ref{fig:hd1} and \ref{fig:hd2} show that the excess risk of \eqref{eq:tron} is no worse than that of \eqref{eq:sgd}, even when both algorithms are tuned with their hyperparameters (initial stepsizes) respectively. 
This verifies our Theorem \ref{thm:tron-vs-sgd:every-case}. 
In the noiseless setting,  Figure \ref{fig:2d} clearly illustrates that \eqref{eq:sgd} can converge to a critical point with constant risk, while \eqref{eq:tron} successfully recovers the true parameters $\wB_*$. This verifies our Theorem \ref{thm:tron-vs-sgd:bad-case}. 

\section{Proof Sketch}\label{sec:proof}
We now overview our techniques for analyzing \eqref{eq:tron} iterates in both well-specified and misspecified cases.

For simplicity let us denote the label noise by $\epsilon_t:= y_t - \relu(\wB_*^\top \xB_t)$.
We first reformulate \eqref{eq:tron} as 
\begin{equation*}
\begin{aligned}
     &\ \quad   \wB_{t} - \wB_* \\
     &= \underbrace{\Big( \IB - \gamma_t \ind{\xB_t^\top \wB_{t-1}>0} \xB_t \xB_t^\top \Big) (\wB_{t-1}-\wB_*)}_{\cB}  \\
    &\quad  + \underbrace{\gamma_t \big(  \onebb [\xB_t^\top \wB_*>0] - \onebb [\xB_t^\top \wB_{t-1}>0]\big) \xB_t \xB_t^\top \wB_*}_{\fB} \\
    &\quad + \underbrace{\gamma_t\epsilon_t\xB_t}_{\nB},
\end{aligned}
\end{equation*}
where the three parts can be understood as a \emph{contraction} term ($\cB$), a \emph{fluctuation} term ($\fB$) and a \emph{noise} term ($\nB$), respectively.
So we have 
\begin{align*}
    \AB_t & :=\Ebb (\wB_t - \wB_*)^{\otimes 2} = \Ebb[(\cB + \fB + \nB)^{\otimes 2}] \\
    &= \Ebb [\cB^{\otimes 2} + \fB^{\otimes 2} + \nB^{\otimes 2} + \text{cross terms}].
\end{align*}
We begin with computing the three quadratic terms.
For the contraction term, by Assumption~\ref{assump:symmetric}
we have 
\begin{align*}
    &\ \quad \Ebb[\cB^{\otimes 2}] \\
    &= \AB_{t-1} - \frac{\gamma_t}{2}(\HB \AB_{t-1}^\top + \AB_{t-1} \HB^\top ) + \frac{\gamma^2_t}{2} \Mcal \circ \AB_{t-1}.
\end{align*}
For the fluctuation term, we have 
\begin{align*}
     &\quad\ \Ebb[\fB^{\otimes 2}] \\
     &= \gamma_t^2 \cdot \Ebb \big[  (\onebb [\xB_t^\top \wB_*>0] - \onebb [\xB_t^\top \wB_{t-1}>0])^2 \cdot \\
     &\hspace{50mm} (\xB_t^\top \wB_*)^2 \cdot \xB_t^{\otimes 2}\big] \\
    &= 2\gamma_t^2 \cdot \Ebb \big[  \onebb [\xB_t^\top \wB_*<0,\xB_t^\top \wB_{t-1}>0] \cdot \\
    &\hspace{50mm} (\xB_t^\top \wB_*)^2 \cdot \xB_t^{\otimes 2}\big],
\end{align*}
where in the last inequality we use Assumption~\ref{assump:symmetric}.
As for the noise term, we simply apply Assumption~\ref{assump:noise:well-specified} in the well-specified setting or Assumption~\ref{assump:noise:misspecified} in the misspecified setting to obtain 
\begin{equation*}
    \Ebb[\nB^{\otimes 2}] \preceq \gamma_t^2 \sigma^2 \HB.
\end{equation*}
In what follows, we utilize the symmetricity condition (Assumption~\ref{assump:symmetric}) to compute the cross terms. 

\paragraph{Well-Specified Setting.}
In the well-specified setting we have that $\epsilon_t$ is mean zero conditional on $\xB_t$, so all the cross terms involving $\nB$ is mean zero, then we have
\begin{equation*}
    \Ebb[\text{cross terms}] = \Ebb [ \cB \fB^\top + \fB \cB^\top ].
\end{equation*}
Moreover, under Assumption~\ref{assump:symmetric} it holds that 
\(\Ebb [\fB ] = 0,\)
so the part in $\cB$ that does not involve $\xB_t$ will disappear in the expected crossing terms, i.e., 
\begin{align*}
    &\ \Ebb[\text{cross terms}] = \Ebb [ \cB \fB^\top + \fB \cB^\top] \\
    &= -\gamma_t  \Ebb\big[ \ind{\xB_t^\top \wB_{t-1}>0}  \xB_t^\top (\wB_{t-1}-\wB_*)  (\xB_t \fB^\top + \fB \xB_t^\top) \big] \\
    &= 2\gamma_t^2  \Ebb\big[ \ind{\xB_t^\top \wB_{t-1}>0, \xB^\top_t \wB_*<0}\cdot \\
    &\hspace{30mm} \xB_t^\top (\wB_{t-1}-\wB_*) \cdot\xB_t^\top \wB_*\cdot \xB_t \xB_t^\top \big].
\end{align*}
Combining the cross term and $\Ebb[\fB^{\otimes 2}]$ we obtain 
\begin{align*}
    & \Ebb[\fB^{\otimes 2} + \text{cross terms}] = \Ebb [ \fB^{\otimes 2} + \cB \fB^\top + \fB \cB^\top]  \\
    &= 2\gamma_t^2  \Ebb\big[ \ind{\xB_t^\top \wB_{t-1}>0, \xB^\top_t \wB_*<0}\cdot \\
    &\hspace{30mm} \xB_t^\top \wB_{t-1}\cdot\xB_t^\top \wB_*\cdot \xB_t \xB_t^\top \big]\\
    &\preceq 0,
\end{align*}
where the last inequality is because the random variable inside the expectation is always non-positive. 

Putting everything together, we 
have shown that 
\begin{align}
    \AB_t 
    &= \Ebb [\cB^{\otimes 2} + \fB^{\otimes } + \nB^{\otimes 2} + \text{cross terms}] \notag \\
    &= \Ebb [\cB^{\otimes 2} + \fB^{\otimes } + \nB^{\otimes 2} + \cB\fB^\top + \fB \cB^\top] \label{eq:sketch:well-specified:bound}\\
    &\preceq \AB_{t-1} - \frac{\gamma_t}{2}\cdot (\HB \AB_{t-1}^\top + \AB_{t-1}\HB^\top) \notag \\
    &\hspace{10mm} + \frac{\gamma^2_t}{2}\cdot \Mcal \circ \AB_{t-1} + \gamma_t^2 \sigma^2 \cdot \HB.\notag 
\end{align}
This matrix recursion has been well-understood thanks to the works by \citet{zou2021benign,wu2022iterate,wu2022power}.

\paragraph{Misspecified Setting.}
Now we consider the misspecified setting. 
Compared to the well-specified setting, the difference is that
the part of the cross terms that involve $\epsilon_t$ is no longer zero mean, as $\epsilon_t$ could correlate with $\xB_t$ in an arbitrary manner.
The extra work is to understand this part of the cross terms:
\begin{align*}
   &\ \quad \Ebb[ \cB \nB^\top + \nB \cB^\top + \fB \nB^\top + \nB\fB^\top ] \\
    &= \underbrace{\gamma_t \Ebb\big[\epsilon_t \big((\wB_{t-1}-\wB_*)\xB_t^\top + \xB_t (\wB_{t-1}-\wB_*)\big)\big]}_{\text{leading order}} \\
    &\quad+ \underbrace{2\gamma_t^2 \Ebb[\mathtt{IndFunc1}\cdot \epsilon_t\cdot \xB_t^\top(\wB_{t-1}-\wB_*)\cdot \xB_t\xB_t^\top]}_{\text{higher order 1}} \\
    &\quad+ \underbrace{2\gamma_t^2 \Ebb[\mathtt{IndFunc2}\cdot \epsilon_t\cdot \xB_t^\top\wB_* \cdot \xB_t\xB_t^\top]}_{\text{higher order 2}},
\end{align*}
where $\mathtt{IndFunc1}$ and $\mathtt{IndFunc2}$ are two functions of indicators, both bounded between $-1$ and $1$.
For the first higher order term, notice the following by Cauchy inequality:
\begin{align*}
& \ \mathtt{IndFunc1}\cdot \epsilon_t\cdot \xB_t^\top(\wB_{t-1}-\wB_*) \\
&\hspace{20mm} \le \half \big( \epsilon_t^2 + (\xB_t^\top(\wB_{t-1}-\wB_*) )^2 \big),
\end{align*}
so we have 
\begin{align*}
    &\quad \ \text{higher order 1} \\ 
    &\preceq \gamma_t^2\cdot \Ebb [ \epsilon_t^2 \xB_t\xB_t^\top +  (\xB_t^\top(\wB_{t-1}-\wB_*) )^2\cdot \xB_t\xB_t^\top ] \\
    &\preceq \gamma_t^2 \sigma^2 \HB + \gamma_t^2 \Mcal \circ \AB_{t-1},
\end{align*}
where in the last inequality we use Assumption~\ref{assump:noise:misspecified}.
We bound the second higher order term in the same manner:
\begin{align*}
    \text{higher order 2}
    &\preceq \gamma_t^2 \cdot \Ebb [ \epsilon_t^2 \cdot \xB_t\xB_t^\top +  (\xB_t^\top\wB_* )^2\cdot \xB_t\xB_t^\top ] \\
    &\preceq \gamma_t^2 \sigma^2 \cdot \HB + \alpha \gamma_t^2 \|\wB_*\|_{\HB}^2 \cdot \HB,
\end{align*}
where the last inequality is by Assumptions \ref{assump:noise:misspecified} and \ref{assump:gaussian}\ref{item:gaussian:upper}.

The leading order term needs some special treatments. 
In fact, it is hard to sharply control the leading order term by a PSD matrix.
Alternatively, it is possible to sharply bound the \emph{diagonal} of the leading order term by a diagonal matrix (here we assume that $\HB$ is diagonal, without loss of generality).
The following bound is proved in Lemma~\ref{lemma:tron:agnostic:cross} in Appendix \ref{append:sec:misspecified}:
\begin{align*}
    \diag(\text{leading order}) \preceq \frac{\gamma_t}{2} \cdot \HB \diag(\AB_{t-1})  + 2\gamma_t\cdot \XiB, 
\end{align*}
where $\XiB$ is a fixed diagonal PSD matrix and $\tr(\XiB) \le \OPT$.

Putting things together with \eqref{eq:sketch:well-specified:bound}, we have 
\begin{align*}
    & \quad  \diag(\AB_t) \\
    &= \diag( \Ebb [\cB^{\otimes 2} + \fB^{\otimes } + \nB^{\otimes 2} + \cB\fB^\top + \fB \cB^\top] ) \\
    &\qquad + \diag(\Ebb[\cB \nB^\top + \nB \cB^\top + \fB \nB^\top + \nB\fB^\top  ]) \\
    &\preceq \diag(\AB_{t-1}) - \gamma_t \HB \diag(\AB_{t-1}) \\
    &\ + \frac{\gamma^2_t}{2}\diag( \Mcal \circ \AB_{t-1}) + \gamma_t^2 \sigma^2  \HB + \gamma_t^2 \sigma^2 \HB \\
    &\  + \gamma_t^2 \diag(\Mcal \circ \AB_{t-1}) +\gamma_t^2 \sigma^2 \HB + \alpha \gamma_t^2 \|\wB_*\|^2_{\HB} \HB \\
    &\ + \frac{\gamma_t}{2}  \HB \diag(\AB_{t-1})  + 2\gamma_t \XiB \\
    &\preceq \Big( \IB- \frac{\gamma_t}{2} \HB \Big)\diag(\AB_{t-1}) + {2\gamma^2_t}\diag( \Mcal \circ \AB_{t-1}) \\
    &\  + 3\gamma_t^2 (\sigma^2+\alpha\|\wB_*\|_{\HB}^2)  \HB + 2\gamma_t \XiB.
\end{align*}
The remaining efforts are to bound the above recursion using techniques developed from \citet{zou2021benign,wu2022iterate,wu2022power}.
It is crucial to remark that $\tr(\XiB) \le \OPT$, which ensures that the cumulation of the extra ``noise term'', $2\gamma_t \XiB$, would cause an additive error of at most $\Ocal(\OPT)$ in the final risk bound.

\section{Conclusion}\label{sec:conclusion}
We consider the problem of learning high-dimensional ReLU regression with well-specified or misspecified noise. 
In the well-specified setting, we provide instance-wise sharp excess risk upper and lower bounds for GLM-tron, that can be applied in the overparameterized regime. 
In the misspecified setting, we also provide sharp instance-dependent risk upper bound for GLM-tron.
In addition, negative results are shown for SGD in well-specified or noiseless ReLU regression with symmetric Bernoulli data, suggesting that GLM-tron might be more effective in ReLU regression.

\section*{Acknowledgements}
We would like to thank the anonymous reviewers and area chairs for their helpful comments. 
This work has been made possible in part by a gift from the Chan Zuckerberg Initiative Foundation to establish the Kempner Institute for the Study of Natural and Artificial Intelligence.
JW and VB are partially supported by the National Science Foundation awards \#2244870, \#2107239, and \#2244899. 
ZC and QG are partially supported by the National Science Foundation awards IIS-1906169 and IIS-2008981. 
SK acknowledges funding from the Office of Naval Research under award N00014-22-1-2377 and the National Science Foundation Grant under award \#CCF-2212841.
The views and conclusions contained in this paper are those of the authors and should not be interpreted as representing any funding agencies.

\bibliographystyle{icml2023}
\bibliography{ref}

\appendix
\onecolumn

\section{Weaker Symmetricity Assumptions}\label{append:sec:symmetricity}

In fact, Assumption~\ref{assump:symmetric} can be relaxed into some moment symmetricity conditions:
\begin{assumption}[Moment symmetricity conditions]\label{assump:symmetric:moment}
    Assume that 
    \begin{enumerate}[label=(\Alph*),leftmargin=*]
    \item \label{item:symmetric:moment-2} For every $\uB \in \Hbb$, it holds that 
    \[\Ebb \big[\xB \xB^\top \cdot \ind{\xB^\top \uB > 0} \big] = \Ebb \big[\xB \xB^\top \cdot \ind{\xB^\top \uB < 0} \big]. \]
    \item \label{item:symmetric:moment-2-2} For every $\uB \in \Hbb$ and $\vB \in \Hbb$, it holds that 
    \[\Ebb \big[\xB \xB^\top \cdot \ind{\xB^\top \uB > 0, \xB^\top \vB>0} \big] = \Ebb \big[\xB \xB^\top \cdot \ind{\xB^\top \uB < 0, \xB^\top \vB<0} \big]. \]
    \item \label{item:symmetric:moment-4-1} For every $\uB \in \Hbb$, it holds that 
    \[\Ebb \big[  \xB^{\otimes 4} \cdot \ind{\xB^\top \uB > 0} \big] = \Ebb \big[ \xB^{\otimes 4} \cdot \ind{\xB^\top \uB < 0} \big]. \]
    \item \label{item:symmetric:moment-4-2} For every $\uB \in \Hbb$ and $\vB\in \Hbb$, it holds that 
    \[\Ebb \big[ (\xB^\top \vB)^2 \xB \xB^\top \cdot \ind{\xB^\top \uB > 0, \xB^\top \vB>0} \big] = \Ebb \big[ (\xB^\top \vB)^2 \xB \xB^\top \cdot \ind{\xB^\top \uB < 0, \xB^\top \vB<0} \big]. \]
    \end{enumerate}
\end{assumption}
Clearly all the conditions in Assumption~\ref{assump:symmetric:moment} holds when Assumption~\ref{assump:symmetric} is true.
Assumption~\ref{assump:symmetric:moment}\ref{item:symmetric:moment-2} is crucial to our analysis. 
Assumption~\ref{assump:symmetric:moment}\ref{item:symmetric:moment-2-2} is only useful for deriving lower bounds. Note that Assumption~\ref{assump:symmetric:moment}\ref{item:symmetric:moment-2-2} implies Assumption~\ref{assump:symmetric:moment}\ref{item:symmetric:moment-2}.
Assumption~\ref{assump:symmetric:moment}\ref{item:symmetric:moment-4-1} is only useful for deriving lower bounds, too. 
Assumption~\ref{assump:symmetric:moment}\ref{item:symmetric:moment-4-2} is only made for technical simplicity; without using Assumption~\ref{assump:symmetric:moment}\ref{item:symmetric:moment-4-2} one can still derive an upper bound for GLM-tron, the only difference will be replacing $\sigma^2$ in the current upper bound with $\sigma^2 + \alpha \| \wB_* \|_{\HB}^2$.

\paragraph{Some Moments Results.}
The following moments results are direct consequences of Assumption~\ref{assump:symmetric:moment}.
\begin{lemma}\label{lemma:moments}
The following holds:
\begin{enumerate}[label=(\Alph*),leftmargin=*]
    \item \label{item:moment-2}
Under Assumption~\ref{assump:symmetric:moment} \ref{item:symmetric:moment-2}, it holds that: for every vector $\uB \in \Hbb$,
\[\Ebb \big[\xB \xB^\top \cdot \ind{\xB^\top \uB > 0} \big] = \half\cdot  \Ebb \big[\xB \xB^\top] =: \half \cdot \HB.\]
\item \label{item:moment-4}
Under Assumption  \ref{assump:symmetric:moment} \ref{item:symmetric:moment-4-1}, it holds that: for every vector $\uB\in \Hbb$,
\[
    \Ebb \big[\xB^{\otimes 4} \cdot \ind{\xB^\top \uB > 0} \big] = \half \cdot \Ebb \big[\xB^{\otimes 4}\big] =: \half \cdot \Mcal.
\]
\end{enumerate}
\end{lemma}
\begin{proof}[Proof of Lemma~\ref{lemma:moments}]
By Assumption~\ref{assump:symmetric:moment}\ref{item:symmetric:moment-2}, we have 
\begin{align*}
    \Ebb \big[\xB \xB^\top \ind{\xB^\top \uB > 0} \big] 
    &= \Ebb \big[(-\xB) (-\xB)^\top \ind{(-\xB)^\top \uB > 0} \big]= \Ebb \big[\xB \xB^\top \ind{\xB^\top \uB < 0} \big].
\end{align*}
Moreover, notice that 
\[
\Ebb \big[\xB \xB^\top \ind{\xB^\top \uB > 0} \big] +  \Ebb \big[\xB \xB^\top \ind{\xB^\top \uB < 0} \big]
= \Ebb \big[\xB \xB^\top].
\]
The above two equations together imply that 
\[
\Ebb \big[\xB \xB^\top \ind{\xB^\top \uB > 0} \big] = \half \Ebb \big[\xB \xB^\top].
\]
Similarly, we can prove the second equality in the lemma. 
\end{proof}

\section{Well-Specified Setting}\label{append:sec:well-specified}

In this section, we focus on the well-specified setting and always assume Assumption~\ref{assump:noise:well-specified} holds.

\subsection{Proof of Lemma~\ref{thm:loss-landscape}}
We will prove a slightly stronger lemma.
\begin{lemma}[Loss landscape, restated Lemma~\ref{thm:loss-landscape}]\label{append:thm:loss-landscape}
Suppose that Assumption~\ref{assump:noise:well-specified} holds.
Consider~\eqref{eq:excess-risk}, we have:
\begin{enumerate}[label=(\Alph*),leftmargin=*]
    \item \(\excessrisk(\wB) \le \| \wB - \wB_*\|_{\HB}^2;\)
    \item if in addition Assumption~\ref{assump:symmetric:moment}\ref{item:symmetric:moment-2-2} holds, then 
    \(
 \excessrisk(\wB) \ge \frac{1}{4} \cdot \|\wB - \wB_*\|_{\HB}^2.
\)
\end{enumerate}
\end{lemma}
\begin{proof}
Under Assumption~\ref{assump:noise:well-specified}, it holds that 
\[\Delta(\wB) = \Ebb \big(\relu(\xB^\top \wB) - \relu(\xB^\top \wB_*) \big)^2. \]

The upper bound follows from the fact that $\relu(\cdot)$ is $1$-Lipschitz, i.e., $|\relu(a)-\relu(b)| \le |a-b|$.

For the lower bound, we first expand the excess risk to obtain that 
\begin{align*}
    &\ \Ebb \big(\relu(\xB^\top \wB) - \relu(\xB^\top \wB_*) \big)^2  \\
    &= \Ebb \big(\xB^\top \wB \cdot \ind{\xB^\top \wB >0 } - \xB^\top \wB_*\cdot \ind{\xB^\top \wB_*>0} \big)^2 \\
    &= \Ebb \big[ \wB^\top \xB \xB^\top \wB \cdot \ind{\xB^\top \wB >0 } \big]  + \Ebb \big[ \wB_*^\top \xB \xB^\top \wB_* \cdot \ind{\xB^\top \wB_* >0 } \big]\\
    &\quad- 2\Ebb \big[ \wB^\top \xB \xB^\top \wB_* \cdot \ind{\xB^\top \wB >0, \xB^\top \wB_*>0 } \big].
\end{align*}
In the above equation, we use Assumption~\ref{assump:symmetric:moment}\ref{item:symmetric:moment-2-2} to obtain that 
\begin{align*}
    &\ \Ebb \big(\relu(\xB^\top \wB) - \relu(\xB^\top \wB_*) \big)^2  \\
    &= \Ebb \big[ \wB^\top \xB \xB^\top \wB \cdot \ind{\xB^\top \wB <0 } \big]  + \Ebb \big[ \wB_*^\top \xB \xB^\top \wB_* \cdot \ind{\xB^\top \wB_* <0 } \big]\\
    &\quad- 2\Ebb \big[ \wB^\top \xB \xB^\top \wB_* \cdot \ind{\xB^\top \wB <0, \xB^\top \wB_*<0 } \big] \\
    &= \Ebb \big(\relu(-\xB^\top \wB) - \relu(-\xB^\top \wB_*) \big)^2.
\end{align*}
Moreover, notice the following by Cauchy inequality:
\begin{align*}
    &\big(\xB^\top \wB - \xB^\top \wB_* \big)^2 \\
    &=  \big( \xB^\top \wB \ind{\xB^\top \wB > 0} - \xB^\top \wB_*\ind{\xB^\top \wB_* > 0} 
    + \xB^\top \wB \ind{\xB^\top \wB < 0} - \xB^\top \wB_*\ind{\xB^\top \wB_* < 0} \big)^2 \\
    &\le 2 \big(\xB^\top \wB \ind{\xB^\top \wB > 0} - \xB^\top \wB_*\ind{\xB^\top \wB_* > 0} \big)^2 
    +  2 \big(\xB^\top \wB \ind{\xB^\top \wB < 0} - \xB^\top \wB_*\ind{\xB^\top \wB_* < 0} \big)^2 \\
    &= 2 \big(\relu(\xB^\top \wB) - \relu(\xB^\top \wB_*) \big)^2 
    +  2 \big(\relu(-\xB^\top \wB)  - \relu(-\xB^\top \wB_*)\big)^2.
\end{align*}
Then taking an expectation on both sides we obtain that
\begin{align*}
    \Ebb \big(\xB^\top \wB - \xB^\top \wB_* \big)^2
    &\le 2\Ebb \big(\relu(\xB^\top \wB) - \relu(\xB^\top \wB_*) \big)^2 
    +  2\Ebb \big(\relu(-\xB^\top \wB)  - \relu(-\xB^\top \wB_*)\big)^2 \\
    &= 4 \Ebb \big(\relu(\xB^\top \wB) - \relu(\xB^\top \wB_*) \big)^2,
\end{align*}
which concludes the proof.
\end{proof}

\subsection{Proof of Lemma~\ref{thm:tron:covariance}}
We will prove a stronger result.
\begin{lemma}[Generic bounds on the GLM-tron iterates, restated Lemma~\ref{thm:tron:covariance}]\label{append:thm:tron:covariance}
    Suppose that Assumption~\ref{assump:noise:well-specified} holds. 
Consider \eqref{eq:tron}. Then:
    \begin{enumerate}[label=(\Alph*)]
        \item If in addition Assumptions \ref{assump:symmetric:moment}\ref{item:symmetric:moment-2} and \ref{assump:symmetric:moment}\ref{item:symmetric:moment-4-2} hold,  then \(
        \AB_{t+1} \preceq \bigg(\Ical - \frac{\gamma_t}{2}\cdot \Tcal(2\gamma_t)\bigg) \circ  \AB_{t-1}  +  \gamma_t^2 \sigma^2 \HB;
  \)
  \item If in addition Assumptions \ref{assump:symmetric:moment}\ref{item:symmetric:moment-2}, \ref{assump:symmetric:moment}\ref{item:symmetric:moment-4-1} and \ref{assump:symmetric:moment}\ref{item:symmetric:moment-4-2} hold, then \(\AB_{t+1} \succeq \Big(\Ical - \frac{\gamma_t}{2}\cdot \Tcal\big(\frac{\gamma_t}{2}\big)\Big) \circ  \AB_t  +  \frac{\gamma_t^2 \sigma^2}{4} \cdot \HB.\)
    \end{enumerate}
\end{lemma}
\begin{proof}
From \eqref{eq:tron} we have
\begin{align*}
    \wB_{t}
    &= \wB_{t-1} - \gamma_t\cdot \big( \relu(\xB_t^\top \wB_{t-1}) - y_t \big) \xB_t \\
    &= \wB_{t-1} - \gamma_t \ind{\xB_t^\top \wB_{t-1}>0}\cdot \xB_t \xB_t^\top \wB_{t-1}  + \gamma_t \ind{\xB_t^\top \wB_*>0}\cdot \xB_t \xB_t^\top \wB_* +\gamma_t\epsilon_t\xB_t\\
    &= \wB_{t-1} - \gamma_t \ind{\xB_t^\top \wB_{t-1}>0}\cdot \xB_t \xB_t^\top (\wB_{t-1}-\wB_*)  \\
    &\quad + \gamma_t \big(  \ind{\xB_t^\top \wB_*>0} - \ind{\xB_t^\top \wB_{t-1}>0}\big)\cdot \xB_t \xB_t^\top \wB_*+\gamma_t\epsilon_t\xB_t,
\end{align*}
which implies that 
\begin{equation}\label{eq:tron:wt-w*}
\begin{aligned}
    \wB_{t} - \wB_* 
    & = \Big( \IB - \gamma_t \ind{\xB_t^\top \wB_{t-1}>0} \xB_t \xB_t^\top \Big) (\wB_{t-1}-\wB_*)  \\
    &\quad + \gamma_t \big(  \ind{\xB_t^\top \wB_*>0} - \ind{\xB_t^\top \wB_{t-1}>0}\big) \xB_t \xB_t^\top \wB_* + \gamma_t\epsilon_t\xB_t.
\end{aligned}
\end{equation}
Let us consider the expected outer product:
\begin{equation}\label{eq:gaussian:tron:out-product}
\begin{aligned}
    & \Ebb \big( \wB_{t} - \wB_*  \big)^{\otimes 2}\\ 
    &= \Ebb \underbrace{\Big( \IB - \gamma_t \ind{\xB_t^\top \wB_{t-1}>0}\cdot \xB_t \xB_t^\top \Big)^{\otimes 2} \circ  (\wB_{t-1}-\wB_*)^{\otimes 2}}_{(\texttt{quadratic term 1})} \\
    &\quad +\gamma_t^2 \cdot \Ebb  \underbrace{\big(  \ind{\xB_t^\top \wB_*>0} - \ind{\xB_t^\top \wB_{t-1}>0}\big)^2 \cdot \xB_t \xB_t^\top \wB_*\wB_*^\top \xB_t \xB_t^\top}_{(\texttt{quadratic term 2})}\\
    &\quad + \gamma_t\cdot \Ebb \underbrace{ \big(  \ind{\xB_t^\top \wB_*>0} - \ind{\xB_t^\top \wB_{t-1}>0}\big)\cdot \xB_t \xB_t^\top \wB_*(\wB_{t-1}-\wB_*)^\top \Big( \IB - \gamma_t \ind{\xB_t^\top \wB_{t-1}>0}\cdot \xB_t \xB_t^\top \Big) }_{(\texttt{crossing term 1})} \\
    &\quad +\gamma_t\cdot \Ebb \underbrace{\big(  \ind{\xB_t^\top \wB_*>0} - \ind{\xB_t^\top \wB_{t-1}>0}\big)\cdot \Big( \IB - \gamma_t \ind{\xB_t^\top \wB_{t-1}>0}\cdot \xB_t \xB_t^\top \Big) (\wB_{t-1}-\wB_*)\wB_*^\top \xB_t \xB_t^\top}_{(\texttt{crossing term 2})} \\
    &\quad + \gamma_t^2\cdot \Ebb \big( \epsilon_t^2 \xB_t\xB_t^\top\big),
\end{aligned}
\end{equation}
where the crossing terms involving $\epsilon_t$ has zero expectation because $\Ebb [ \epsilon_t | \xB_t ] = 0$. 

For the second quadratic term in \eqref{eq:gaussian:tron:out-product}, notice that 
\[
\big(  \ind{\xB_t^\top \wB_*>0} - \ind{\xB_t^\top \wB_{t-1}>0}\big)^2
= \ind{\xB_t^\top \wB_{t-1}>0, \xB_t^\top \wB_*<0} + \ind{\xB_t^\top \wB_{t-1}<0, \xB_t^\top \wB_*>0},
\]
then we have
\begin{align}
    &\ \Ebb (\texttt{quadratic term 2}) \notag \\
    &= \Ebb \bigg( \big(  \ind{\xB_t^\top \wB_*>0} - \ind{\xB_t^\top \wB_{t-1}>0}\big)^2  \cdot \big(\xB_t^\top \wB_*\big)^2  \cdot \xB_t \xB_t^\top \bigg) \notag \\
    &=  \Ebb \bigg(\Big(  \ind{\xB_t^\top \wB_{t-1}>0, \xB_t^\top \wB_*<0} + \ind{\xB_t^\top \wB_{t-1}<0, \xB_t^\top \wB_*>0}\Big)\cdot \big(\xB_t^\top \wB_*\big)^2  \cdot \xB_t \xB_t^\top  \bigg) \label{eq:gaussian:tron:quad-term2:ver0} \\
    &= 2\cdot \Ebb \Big(  \ind{\xB_t^\top \wB_{t-1}>0,\xB_t^\top \wB_*<0} \cdot \big(\xB_t^\top \wB_*\big)^2  \cdot \xB_t \xB_t^\top  \Big), \label{eq:gaussian:tron:quad-term2}
\end{align}
where the last equation is by Assumption~\ref{assump:symmetric:moment}\ref{item:symmetric:moment-4-2}.
For the crossing terms in \eqref{eq:gaussian:tron:out-product} we have that 
\begin{align}
    & \quad (\texttt{crossing term 1}) + (\texttt{crossing term 2}) \notag \\
    &= \big(  \ind{\xB_t^\top \wB_*>0} - \ind{\xB_t^\top \wB_{t-1}>0}\big)\cdot \bigg( \xB_t \xB_t^\top \wB_*(\wB_{t-1}-\wB_*)^\top+ (\wB_{t-1}-\wB_*)\wB_*^\top \xB_t \xB_t^\top \bigg) \notag \\
    &\quad - 2 \gamma_t \big(  \ind{\xB_t^\top \wB_*>0} - \ind{\xB_t^\top \wB_{t-1}>0}\big)\cdot\ind{\xB_t^\top \wB_{t-1}>0}\cdot \xB_t^\top \wB_* \cdot \xB_t^\top (\wB_{t-1}-\wB_*)\cdot \xB_t \xB_t^\top \notag \\
    &= \big(  \ind{\xB_t^\top \wB_*>0} - \ind{\xB_t^\top \wB_{t-1}>0}\big)\cdot \bigg( \xB_t \xB_t^\top \wB_*(\wB_{t-1}-\wB_*)^\top+ (\wB_{t-1}-\wB_*)\wB_*^\top \xB_t \xB_t^\top \bigg) \notag \\
    &\quad + 2 \gamma_t \ind{\xB_t^\top \wB_{t-1}>0, \xB_t^\top \wB_*<0}\cdot \xB_t^\top \wB_* \cdot \xB_t^\top (\wB_{t-1}-\wB_*)\cdot \xB_t \xB_t^\top,\label{eq:gaussian:tron:cross-term}
\end{align}
where in the last equality we use 
\begin{align*}
    &\quad  - \big(  \ind{\xB_t^\top \wB_*>0} - \ind{\xB_t^\top \wB_{t-1}>0}\big)\cdot\ind{\xB_t^\top \wB_{t-1}>0} \\
&= \ind{\xB_t^\top \wB_{t-1} >0} - \ind{\xB_t^\top \wB_*>0}\cdot\ind{\xB_t^\top \wB_{t-1}>0} \\
&= \ind{\xB_t^\top \wB_{t-1}>0,\xB_t^\top \wB_*<0}.
\end{align*}
Now we take expectation on both sides of \eqref{eq:gaussian:tron:cross-term}.
By Assumption~\ref{assump:symmetric:moment}\ref{item:symmetric:moment-2} (or Lemma~\ref{lemma:moments}\ref{item:moment-2}) the first term in \eqref{eq:gaussian:tron:cross-term} has zero expectation, therefore we obtain
\begin{align}
    & \quad \Ebb\big(   (\texttt{crossing term 1}) + (\texttt{crossing term 2}) \big) \notag \\
    &= 2 \gamma_t \cdot \Ebb\bigg( \ind{\xB_t^\top \wB_{t-1}>0, \xB_t^\top \wB_*<0}\cdot \xB_t^\top \wB_* \cdot \xB_t^\top (\wB_{t-1}-\wB_*)\cdot \xB_t \xB_t^\top\bigg) \notag \\
    &= 2 \gamma_t \cdot \Ebb\bigg(\ind{\xB_t^\top \wB_{t-1}>0, \xB_t^\top \wB_*<0}\cdot \xB_t^\top \wB_* \cdot \xB_t^\top \wB_{t-1}\cdot \xB_t \xB_t^\top\bigg) \notag \\
    &\quad  -  2 \gamma_t\cdot \Ebb\bigg( \ind{\xB_t^\top \wB_{t-1}>0, \xB_t^\top \wB_*<0}\cdot \big(\xB_t^\top \wB_*\big)^2\cdot \xB_t \xB_t^\top\bigg) , \label{eq:gaussian:tron:cross-term:2}
\end{align}
Now considering \eqref{eq:gaussian:tron:out-product} and applying \eqref{eq:gaussian:tron:quad-term2} and \eqref{eq:gaussian:tron:cross-term:2}, we obtain 
\begin{align}
      \Ebb \big( \wB_{t} - \wB_*  \big)^{\otimes 2} 
      &=  \Ebb\Big( \IB - \gamma_t \ind{\xB_t^\top \wB_{t-1}>0}\cdot \xB_t \xB_t^\top \Big)^{\otimes 2} \circ  (\wB_{t-1}-\wB_*)^{\otimes 2}+ \gamma_t^2 \sigma^2 \HB \notag \\
      &\quad + 2 \gamma_t^2 \cdot \Ebb\bigg(\ind{\xB_t^\top \wB_{t-1}>0, \xB_t^\top \wB_*<0}\cdot \xB_t^\top \wB_* \cdot \xB_t^\top \wB_{t-1}\cdot \xB_t \xB_t^\top\bigg)   .\label{eq:eq:gaussian:tron:out-product:1}
\end{align}

\paragraph{An Upper Bound.}
In \eqref{eq:eq:gaussian:tron:out-product:1}, we can use the indicator function to show that 
\[
\ind{\xB_t^\top \wB_{t-1}>0, \xB_t^\top \wB_*<0}\cdot \xB_t^\top \wB_* \cdot \xB_t^\top \wB_{t-1} \le 0,
\]
so we have 
\begin{align}
      \Ebb \big( \wB_{t} - \wB_*  \big)^{\otimes 2} 
      &\preceq  \Ebb\Big( \IB - \gamma_t \ind{\xB_t^\top \wB_{t-1}>0}\cdot \xB_t \xB_t^\top \Big)^{\otimes 2} \circ  (\wB_{t-1}-\wB_*)^{\otimes 2}+ \gamma_t^2 \sigma^2 \HB  \notag \\
      &= \Ebb \big( \wB_{t-1} - \wB_*  \big)^{\otimes 2} \notag \\
      &\quad  -\gamma_t \cdot \Ebb \Big( \ind{\xB_t^\top \wB_{t-1}>0}\cdot \xB_t \xB_t^\top \Big)\cdot  \Ebb \big( \wB_{t-1} - \wB_*  \big)^{\otimes 2} \notag \\
      &\quad - \gamma_t \cdot \Ebb\big( \wB_{t-1} - \wB_*  \big)^{\otimes 2}\cdot \Ebb \Big( \ind{\xB_t^\top \wB_{t-1}>0}\cdot \xB_t \xB_t^\top\Big) \notag \\
     &\quad + \gamma_t^2\cdot \Ebb\Big( \ind{\xB_t^\top \wB_{t-1}>0}\cdot \xB_t \xB_t^\top  \cdot \Ebb \big( \wB_{t-1} - \wB_*  \big)^{\otimes 2}\cdot \xB_t \xB_t^\top\Big)\notag \\
     &\quad + \gamma^2_t \sigma^2 \HB. \label{eq:gaussian:tron:out-product:upper}
\end{align}
By Assumption~\ref{assump:symmetric:moment}\ref{item:symmetric:moment-2} (or Lemma~\ref{lemma:moments}\ref{item:moment-2}) we have
\begin{align*}
     \Ebb \Big( \ind{\xB_t^\top \wB_{t-1}>0}\cdot \xB_t \xB_t^\top  \Big) = \half \HB,
\end{align*}
moreover 
\[ 
     \Ebb \Big( \ind{\xB_t^\top \wB_{t-1}>0}\cdot \xB_t \xB_t^\top \otimes \xB_t \xB_t^\top\Big)\preceq \Ebb \Big(  \xB_t \xB_t^\top \otimes \xB_t \xB_t^\top\Big) = \Mcal.\]
Then under notations of $\AB_t$, $\Tcal$ and $\Mcal$, \eqref{eq:gaussian:tron:out-product:upper} can be written as  
\begin{align*}
    \AB_{t} 
    &\preceq \AB_{t-1} - \frac{\gamma_t}{2} \big( \HB \AB_{t-1} + \AB_{t-1} \HB) + {\gamma^2_t} \Mcal \circ \AB_{t-1} + \gamma_t^2 \sigma^2 \HB \\
    &= \bigg(\Ical - \frac{\gamma_t}{2}\cdot \Tcal(2\gamma_t)\bigg) \circ  \AB_{t-1}  +  \gamma_t^2 \sigma^2 \HB.
\end{align*}

\paragraph{A Lower Bound.}
We now derive a lower bound for \eqref{eq:eq:gaussian:tron:out-product:1}.
We first notice the following fact:
for every two vectors $\vB$ and $\uB$, it holds that
\begin{equation}\label{eq:cross-term-lb}
\uB \vB^\top + \vB \uB^\top
= \half \big(  (\uB + \vB)^{\otimes 2}- (\uB - \vB)^{\otimes 2} \big)
\succeq - \half (\uB - \vB)^{\otimes 2}. 
\end{equation}   
Applying \eqref{eq:cross-term-lb}, we obtain that 
\begin{align*}
    & 2 \gamma_t^2 \cdot \Ebb\bigg(\ind{\xB_t^\top \wB_{t-1}>0, \xB_t^\top \wB_*<0}\cdot \xB_t^\top \wB_* \cdot \xB_t^\top \wB_{t-1}\cdot \xB_t \xB_t^\top\bigg)  \\ 
    &= \gamma_t^2 \cdot \Ebb\bigg(\ind{\xB_t^\top \wB_{t-1}>0, \xB_t^\top \wB_*<0}\cdot \xB_t \xB_t^\top \cdot \big( \wB_* \wB_{t-1}^\top + \wB_{t-1} \wB_*^\top \big) \cdot \xB_t \xB_t^\top\bigg)  \\ 
    &\succeq -\frac{\gamma_t^2}{2}\cdot \Ebb\bigg(\ind{\xB_t^\top \wB_{t-1}>0, \xB_t^\top \wB_*<0}\cdot \xB_t \xB_t^\top \cdot (\wB_{t-1}- \wB_{*})(\wB_{t-1}- \wB_{*})^\top \cdot \xB_t \xB_t^\top\bigg) \\
    &\succeq -\frac{\gamma_t^2}{2}\cdot \Ebb\bigg(\ind{\xB_t^\top \wB_{t-1}>0}\cdot \xB_t \xB_t^\top  \cdot \Ebb (\wB_{t-1}- \wB_{*})^{\otimes 2} \cdot \xB_t \xB_t^\top\bigg).
\end{align*}
We now bring this into \eqref{eq:eq:gaussian:tron:out-product:1}, then we get
\begin{align}
      \Ebb \big( \wB_{t} - \wB_*  \big)^{\otimes 2} 
      &\succeq  \Ebb\Big( \IB - \gamma_t \ind{\xB_t^\top \wB_{t-1}>0}\cdot \xB_t \xB_t^\top \Big)^{\otimes 2} \circ  (\wB_{t-1}-\wB_*)^{\otimes 2} + \gamma^2_t\sigma^2 \HB \notag \\
      &\quad -\frac{\gamma_t^2}{2}\cdot \Ebb\bigg(\ind{\xB_t^\top \wB_{t-1}>0}\cdot \xB_t \xB_t^\top  \cdot \Ebb (\wB_{t-1}-\wB_*)^{\otimes 2} \cdot \xB_t \xB_t^\top\bigg)\notag \\
      &= \Ebb \big( \wB_{t-1} - \wB_*  \big)^{\otimes 2} \notag \\
      &\quad  -\gamma_t \cdot \Ebb \Big( \ind{\xB_t^\top \wB_{t-1}>0}\cdot \xB_t \xB_t^\top \Big)\cdot  \Ebb \big( \wB_{t-1} - \wB_*  \big)^{\otimes 2} \notag \\
      &\quad - \gamma_t \cdot \Ebb\big( \wB_{t-1} - \wB_*  \big)^{\otimes 2}\cdot \Ebb \Big( \ind{\xB_t^\top \wB_{t-1}>0}\cdot \xB_t \xB_t^\top\Big) \notag \\
     &\quad + \gamma_t^2\cdot \Ebb\Big( \ind{\xB_t^\top \wB_{t-1}>0}\cdot \xB_t \xB_t^\top  \cdot \Ebb \big( \wB_{t-1} - \wB_*  \big)^{\otimes 2}\cdot \xB_t \xB_t^\top\Big) \notag\\
      &\quad + \gamma_t^2\sigma^2 \HB \notag \\
     &\quad - \frac{\gamma_t^2}{2}\cdot \Ebb\bigg(\ind{\xB_t^\top \wB_{t-1}>0}\cdot \xB_t \xB_t^\top  \cdot \Ebb (\wB_*- \wB_{t-1})^{\otimes 2} \cdot \xB_t \xB_t^\top\bigg) \notag \\
     &= \Ebb \big( \wB_{t-1} - \wB_*  \big)^{\otimes 2} \notag \\
      &\quad  -\gamma_t \cdot \Ebb \Big( \ind{\xB_t^\top \wB_{t-1}>0}\cdot \xB_t \xB_t^\top \Big)\cdot  \Ebb \big( \wB_{t-1} - \wB_*  \big)^{\otimes 2} \notag \\
      &\quad - \gamma_t \cdot \Ebb\big( \wB_{t-1} - \wB_*  \big)^{\otimes 2}\cdot \Ebb \Big( \ind{\xB_{t-1}^\top \wB_{t-1}>0}\cdot \xB_t \xB_t^\top\Big) \notag \\
     &\quad + \frac{\gamma_t^2}{2}\cdot \Ebb\Big( \ind{\xB_t^\top \wB_{t-1}>0}\cdot \xB_t \xB_t^\top  \cdot \Ebb \big( \wB_{t-1} - \wB_*  \big)^{\otimes 2}\cdot \xB_t \xB_t^\top\Big) \notag \\
     &\quad +\gamma^2_t\sigma^2\HB.\label{eq:eq:gaussian:tron:out-product:lower}
\end{align}
By Assumptions \ref{assump:symmetric:moment}\ref{item:symmetric:moment-2} and \ref{assump:symmetric:moment}\ref{item:symmetric:moment-4-1} (or Lemma~\ref{lemma:moments}\ref{item:moment-4}) we have
\begin{align*}
     \Ebb \Big( \ind{\xB_t^\top \wB_{t-1}>0}\cdot \xB_t \xB_t^\top  \Big) = \half \HB,\quad 
     \Ebb \Big( \ind{\xB_t^\top \wB_{t-1}>0}\cdot \xB_t \xB_t^\top \otimes \xB_t \xB_t^\top\Big) = \half \Mcal.
\end{align*}
Then under notations of $\AB_t$, $\Tcal$ and $\Mcal$,  \eqref{eq:eq:gaussian:tron:out-product:lower} can be written as  
\begin{align*}
    \AB_{t} 
    &\succeq \AB_{t-1} - \frac{\gamma_t}{2} \big( \HB \AB_{t-1} + \AB_{t-1} \HB) + \frac{\gamma_t^2}{4} \Mcal \circ \AB_{t-1} + \gamma_t^2 \sigma^2 \HB \\
    &= \bigg(\Ical - \frac{\gamma_t}{2}\cdot{\Tcal}\bigg(\frac{\gamma_t}{2}\bigg)\bigg) \circ  \AB_{t-1} +  \gamma_t^2 \sigma^2 \HB.
\end{align*}
We have completed the proof.
\end{proof}

\subsection{Proof of Theorem \ref{thm:tron:bernoulli}}

\paragraph{Notations.}
In this section, we always assume that $\HB$ is diagonal.
For a PSD matrix $\AB$, 
we use
\( \mathring{\AB}\)
to refer to the diagonal of $\AB$.

\begin{proof}[Proof of Theorem \ref{thm:tron:bernoulli}]
The proof is by combing Lemma~\ref{append:thm:tron:covariance}, Lemma~\ref{append:thm:loss-landscape} and the analysis for one-hot data in \citet{zou2021benefits}.

Note that for symmetric Bernoulli distribution, or under Assumption~\ref{assump:bernoulli}, it holds that (see also the proof of Lemma A.1 in \citet{zou2021benefits}): for any PSD matrix $\AB$, 
\begin{equation}\label{eq:bernolli:fourth-moment}
    \Mcal \circ \AB = \Ebb (\xB^\top \AB \xB)\cdot \xB \xB^\top
= \diag(\HB \AB) = \HB \mathring{\AB}.
\end{equation}

\paragraph{Upper Bound.}
We first show the upper bound.
By Lemma~\ref{append:thm:tron:covariance} and \eqref{eq:bernolli:fourth-moment} we have 
\begin{align*}
    \AB_{t} 
    &\preceq \AB_{t-1} - \frac{\gamma_t}{2} \big( \HB \AB_{t-1} + \AB_{t-1} \HB) + {\gamma^2_t} \Mcal \circ \AB_{t-1} + \gamma_t^2 \sigma^2 \HB \\
    &= \AB_{t-1} - \frac{\gamma_t}{2} \big( \HB \AB_{t-1} + \AB_{t-1} \HB) + {\gamma^2_t}\HB \mathring{\AB}_{t-1} + \gamma_t^2 \sigma^2 \HB.
\end{align*}
Taking diagonal on both sides we get
\begin{align*}
    \mathring\AB_{t} 
    &\preceq \mathring\AB_{t-1} - \gamma_t \HB \mathring\AB_{t-1}  + {\gamma^2_t}\HB \mathring{\AB}_{t-1} + \gamma_t^2 \sigma^2 \HB\\
    &\preceq \Big(\IB - \frac{\gamma_t}{2}\cdot \HB \Big) \cdot \mathring\AB_{t-1}   + \gamma_t^2 \sigma^2 \HB,
\end{align*}
where we use the assumption that $\gamma < 1/2$.
Solving the above recursion and apply Lemma~\ref{lemma:tron:misspecified:prod-sum-bounds}, we obtain
\begin{align*}
    \mathring\AB_N &\preceq \prod_{t=1}^N \Big(\IB - \frac{\gamma_t}{2}\cdot \HB \Big) \cdot \mathring\AB_{0}   +\sigma^2 \sum_{t=1}^N  \gamma_t^2 \prod_{k=t+1}^N \Big(\IB - \frac{\gamma_k}{2}\cdot \HB \Big)\HB \\
    &\preceq \prod_{t=1}^N \Big(\IB - \frac{\gamma_t}{2}\cdot \HB \Big) \cdot \mathring\AB_{0}   +\frac{\sigma^2}{8} \cdot \bigg( \frac{1}{\Neff}\HB_{0:k}^{-1} + \Neff \gamma_0^2 \HB_{k:\infty}\bigg).
\end{align*}
Taking inner product with $\HB$ gives the upper bound on the excess risk.

\paragraph{Lower Bound.}
We next show the lower bound.
By Lemma~\ref{append:thm:tron:covariance} and \eqref{eq:bernolli:fourth-moment} we have 
\begin{align*}
    \AB_{t} 
    &\succeq \AB_{t-1} - \frac{\gamma_t}{2} \big( \HB \AB_{t-1} + \AB_{t-1} \HB) + \frac{\gamma_t^2}{4} \Mcal \circ \AB_{t-1} +  {\gamma^2_t}\sigma^2 \HB \\
    &\succeq \AB_{t-1} - \frac{\gamma_t}{2} \big( \HB \AB_{t-1} + \AB_{t-1} \HB)  + \gamma_t^2 \sigma^2 \HB.
\end{align*}
Taking diagonal on both sides we get
\begin{align*}
    \mathring\AB_{t} 
    &\succeq \big(\IB  - \gamma_t \HB \big) \cdot \mathring\AB_{t-1}  + \gamma_t^2 \sigma^2 \HB.
\end{align*}
Solving the above recursion and apply Lemma~\ref{lemma:tron:misspecified:prod-sum-bounds}, we obtain
\begin{align*}
    \mathring\AB_N &\succeq \prod_{t=1}^N \Big(\IB - {\gamma_t} \HB \Big) \cdot \mathring\AB_{0}   +\sigma^2 \sum_{t=1}^N  \gamma_t^2 \prod_{k=t+1}^N \Big(\IB - {\gamma_k} \HB \Big)\HB \\
    &\succeq \prod_{t=1}^N \Big(\IB - {\gamma_t} \HB \Big) \cdot \mathring\AB_{0}   +\frac{\sigma^2}{400} \cdot \bigg( \frac{1}{\Neff} \HB^{-1}_{0:k^*} + \Neff \gamma_0^2 \HB_{k^*:\infty}\bigg),
\end{align*}
where $k^*:= \max\{k: \lambda_k \ge 1/(\gamma_0 \Neff)\}$.
Taking inner product with $\HB$ gives the lower bound on the excess risk.
\end{proof}

\subsection{Proof of Theorem \ref{thm:tron:gaussian}}
We first restate Corollary 3.4 in \citet{wu2022power} under our notations.

\begin{corollary*}[Corollary 3.4 in \citet{wu2022power}, restated]
Consider a sequence of PSD matrices $(\AB_{t})_{t =0}^N$ that describes the covariance of the SGD iterates for linear regression, i.e.,
\begin{equation*}
\AB_0 := (\wB_0 - \wB_*)^{\otimes 2},\quad 
\AB_{t} := \Ebb (\IB- \gamma_t \xB \xB^\top) \AB_{t-1} (\IB- \gamma_t \xB \xB^\top) + \gamma_t^2\cdot \sigma^2 \cdot \HB ,\ t= 1,\dots,N,
\end{equation*}
where $(\gamma_t)_{t=0}^N$ is a stepsize scheduler as defined in \eqref{eq:geometry-tail-decay-lr}.
Assume that $N > 100$. 
Let $\Neff := N / \log(N)$.
\begin{enumerate}[label=(\Alph*)]
    \item
If Assumption~\ref{assump:gaussian}\ref{item:gaussian:upper} holds,
then for $\gamma_0 < 1/(4\alpha(\tr(\HB)) )$ 
it holds that
\begin{align*}
    \la \HB, \AB_N \ra
    &\lesssim 
    \bigg\| \prod_{t=1}^{N}\Big(\IB-{\gamma_t}\HB\Big)(\wB_0 - \wB_*) \bigg\|^2_\HB  + \Big( \alpha \big\| \wB_0 - \wB_* \big\|^2_{\frac{\IB_{0:k^*} }{\Neff \gamma_0} + \HB_{k^*:\infty}} + \sigma^2 \Big) \cdot \frac{k^* + \Neff^2 \gamma^2_0 \cdot \sum_{i>k^*} \lambda_i^2}{\Neff},
\end{align*}
where $k^*\ge 0$ is an arbitrary index.
\item 
If Assumption~\ref{assump:gaussian}\ref{item:gaussian:lower} holds,
then for $\gamma_0 < 1/(4\alpha(\tr(\HB)) )$ it holds that
\begin{align*}
    \la \HB, \AB_N \ra
    &\gtrsim 
    \bigg\| \prod_{t=1}^{N}\Big(\IB-{\gamma_t}\HB\Big)(\wB_0 - \wB_*) \bigg\|^2_\HB + 
     \big( \beta \|\wB_0 - \wB_* \|^2_{\HB_{k^*:\infty}} + \sigma^2 \big) \cdot  \frac{k^* + \Neff^2 \gamma^2_0 \cdot \sum_{i>k^*} \lambda_i^2}{\Neff},
\end{align*}
where $k^*:= \max\{k: \lambda_k \ge 1/(\gamma_0 \Neff)\}$.
\end{enumerate}
\end{corollary*}
\begin{proof}
    See Corollary 3.4 in \citet{wu2022power}.
\end{proof}

We restate Theorem \ref{thm:tron:gaussian} in a slightly stronger version. 
\begin{theorem}[Risk Bounds for GLM-tron, restated Theorem \ref{thm:tron:gaussian}]\label{append:thm:tron:gaussian}
Suppose that Assumption  \ref{assump:noise:well-specified} holds.
Let $\wB_{N}$ be the output of \eqref{eq:tron} with stepsize scheduler \eqref{eq:geometry-tail-decay-lr}.
Assume that $N > 100$. 
Let $\Neff := N / \log(N)$.
\begin{enumerate}[label=(\Alph*)]
    \item
If in addition Assumption~\ref{assump:gaussian}\ref{item:gaussian:upper} and 
Assumption~\ref{assump:symmetric:moment}\ref{item:symmetric:moment-2}\ref{item:symmetric:moment-4-2} hold,
then for $\gamma_0 < 1/(4\alpha(\tr(\HB)) )$ it holds that
\begin{align*}
    \Ebb \excessrisk (\wB_{N})
    &\lesssim 
    \bigg\| \prod_{t=1}^{N}\Big(\IB-\frac{\gamma_t}{2}\HB\Big)(\wB_0 - \wB_*) \bigg\|^2_\HB  + \Big( \alpha \big\| \wB_0 - \wB_* \big\|^2_{\frac{\IB_{0:k^*} }{\Neff \gamma_0} + \HB_{k^*:\infty}} + \sigma^2 \Big) \cdot \frac{k^* + \Neff^2 \gamma^2_0 \cdot \sum_{i>k^*} \lambda_i^2}{\Neff},
\end{align*}
where $k^*\ge 0$ is an arbitrary index.
\item
If in addition Assumption~\ref{assump:gaussian}\ref{item:gaussian:lower} and Assumption~\ref{assump:symmetric:moment} hold, then
for $\gamma_0 < 1/\lambda_1$,
it holds that
\begin{align*}
    \Ebb \excessrisk (\wB_{N})
    &\gtrsim 
    \bigg\| \prod_{t=1}^{N}\Big(\IB-\frac{\gamma_t}{2}\HB\Big)(\wB_0 - \wB_*) \bigg\|^2_\HB + 
     \big( \beta \|\wB_0 - \wB_* \|^2_{\HB_{k^*:\infty}} + \sigma^2 \big) \cdot  \frac{k^* + \Neff^2 \gamma^2_0 \cdot \sum_{i>k^*} \lambda_i^2}{\Neff},
\end{align*}
where $k^*:= \max\{k: \lambda_k \ge 1/(\gamma_0 \Neff)\}$.
\end{enumerate}
\end{theorem}
\begin{proof}
    We first use Lemma~\ref{append:thm:loss-landscape} and Lemma~\ref{append:thm:tron:covariance} to relate GLM-tron for ReLU regression problems to SGD for linear regression problems. Then we invoke Corollary 3.4 in \citet{wu2022power} (see above) to get the results.
\end{proof}

\subsection{Proof of Corollary \ref{thm:examples}}
\begin{proof}[Proof of Corollary \ref{thm:examples}]
For all these examples one can verify that $\tr(\HB) \eqsim 1$. Therefore $\gamma_0 \eqsim 1$.

We can verify that
\begin{align*}
 \bigg\| \prod_{t=1}^{N}\Big(\IB-\frac{\gamma_t}{2}\HB\Big)(\wB_0 - \wB_*) \bigg\|^2_\HB 
 &\le 
    \bigg\| \Big(\IB-\frac{\gamma_0}{2}\HB\Big)^{\Neff} (\wB_0 - \wB_*) \bigg\|^2_{\HB} \\
&= \sum_{i}   \lambda_i \cdot \bigg( 1-\frac{\gamma_0}{2}\cdot \lambda_i\bigg)^{2\Neff} \cdot (\wB_0[i] - \wB_*[i])^2 \\
&\lesssim \sum_{i}   \lambda_i \cdot \frac{1}{\gamma_0 \lambda_i \Neff} \cdot (\wB_0[i] - \wB_*[i])^2 \\
    &\eqsim \frac{\big\|  (\wB_0 - \wB^*) \big\|^2_{2} }{\gamma_0 \Neff} \\
   & \eqsim \frac{1}{\Neff} \eqsim \frac{\log (N)}{N},
\end{align*}
and that 
\begin{align*}
\big\| \wB_0 - \wB_* \big\|^2_{\frac{\IB_{0:k^*} }{\Neff \gamma_0} + \HB_{k^*:\infty}}
    \lesssim \big\| \wB_0 - \wB_* \big\|^2_2 \lesssim 1.
\end{align*}
Therefore in Theorem \ref{thm:tron:gaussian} we have 
\begin{align*}
   \Ebb \excessrisk (\wB_{N})
    &\lesssim 
    \bigg\| \prod_{t=1}^{N}\Big(\IB-\frac{\gamma_t}{2}\HB\Big)(\wB_0 - \wB_*) \bigg\|^2_\HB  + \Big( \alpha \big\| \wB_0 - \wB_* \big\|^2_{\frac{\IB_{0:k^*} }{\Neff \gamma_0} + \HB_{k^*:\infty}} + \sigma^2 \Big) \cdot \frac{k^* + \Neff^2 \gamma^2_0 \cdot \sum_{i>k^*} \lambda_i^2}{\Neff} \\ 
    &\lesssim \frac{1}{\Neff} + \frac{k^* + \Neff^2 \gamma^2_0 \cdot \sum_{i>k^*} \lambda_i^2}{\Neff} \\ 
    &\lesssim \frac{k^* + \Neff^2  \cdot \sum_{i>k^*} \lambda_i^2}{\Neff}. 
\end{align*}
We next examine each case.
Recall that $k^*:= \max\{k: \lambda_k \ge 1/(\gamma_0 \Neff)\}$.
\begin{enumerate}
    \item By definitions we have 
    \[k^* \eqsim (\Neff)^{\frac{1}{1+r}},\]
    therefore we have 
    \begin{align*}
    k^* + \Neff^2 \cdot \sum_{i>k^*} \lambda_i^2
    &\eqsim  k^* + (\Neff)^2 \cdot (k^*)^{-1-2r} \\
    & \eqsim ( \Neff )^{\frac{1}{1+r}}.
    \end{align*}
    This implies that 
    \[ \Ebb \excessrisk (\wB_{N}) \lesssim (\Neff)^{\frac{-r}{1+r}} \eqsim (N / \log(N))^{\frac{-r}{1+r}}.\]
    
    \item By definitions we have 
    \[k^* \eqsim \Neff\cdot \log^{-r} (\Neff), \]
    therefore we have 
    \begin{align*}
         k^* + \Neff^2  \cdot \sum_{i>k^*} \lambda_i^2
    &\eqsim  k^* + (\Neff)^2 \cdot (k^*)^{-1} \log^{-2r}(k^*) \\
    & \eqsim \Neff\cdot \log^{-r} (\Neff).
    \end{align*}
    This implies that \[ \Ebb \excessrisk (\wB_{N}) \lesssim  \log^{-r} (\Neff) \eqsim  \log^{-r} (N / \log(N))\eqsim \log^{-r}(N).\]

    \item By definitions we have 
    \[k^* \eqsim  \log (\Neff), \]
    therefore we have 
    \begin{align*}
       k^* + \Neff^2  \cdot \sum_{i>k^*} \lambda_i^2
    &\eqsim  k^* + (\Neff)^2 \cdot 2^{-k^*} \\
    & \eqsim \log(\Neff).
    \end{align*}
    This implies that \[ \Ebb \excessrisk (\wB_{N}) \lesssim  \log(\Neff) / \Neff \eqsim  \log^{2}(N) / N.\]
\end{enumerate}
We have completed the proof.
\end{proof}

\subsection{Iterate Average}\label{append:sec:iterate-avg}
We may also consider constant-stepsize GLM-tron with iterate averaging, i.e., \eqref{eq:tron} is run with constant stepsize $\gamma$ and outputs the average of the iterates:
\begin{equation}\label{eq:tron:average}
    \bar{\wB}_N := \frac{1}{N} \sum_{t=0}^{N-1} \wB_t.
\end{equation}

\begin{lemma}[Iterate averaging]\label{lemma:tron:average}
Suppose that Assumption~\ref{assump:noise:well-specified} and Assumption~\ref{assump:symmetric:moment}\ref{item:symmetric:moment-2} hold.
    For $\bar{\wB}_N$ defined in \eqref{eq:tron:average}, we have that
    \begin{align*}
        \Ebb \la \HB, (\bar{\wB}_N-\wB_*)^{\otimes 2}\ra 
        &\le \frac{1}{\gamma N^2}\bigg\la \IB - \bigg(\IB - \frac{\gamma}{2}\HB\bigg)^{N},\ 
        \sum_{t=0}^N \AB_t\bigg\ra; \\
        \Ebb \la \HB, (\bar{\wB}_N-\wB_*)^{\otimes 2}\ra 
        &\ge \frac{1}{2\gamma N^2}\bigg\la \IB - \bigg(\IB - \frac{\gamma}{2}\HB\bigg)^{N/2},\ 
        \sum_{t=0}^{N/2} \AB_t\bigg\ra.
    \end{align*}
\end{lemma}
\begin{proof}
In \eqref{eq:tron:wt-w*}, we take conditional expectation to obtain 
\begin{align*}
    \Ebb[ \wB_{t} - \wB_* | \wB_{t-1}]
    & =  \Ebb\bigg[ \Big( \IB - \gamma_t \ind{\xB_t^\top \wB_{t-1}>0} \xB_t \xB_t^\top \Big) (\wB_{t-1}-\wB_*) | \wB_{t-1}\bigg]  \\
    &\quad + \gamma_t \cdot \Ebb\bigg[\big(  \ind{\xB_t^\top \wB_*>0} - \ind{\xB_t^\top \wB_{t-1}>0}\big) \xB_t \xB_t^\top \wB_*| \wB_{t-1}\bigg]  + \gamma_t\Ebb [\epsilon_t\xB_t | \wB_{t-1}] \\
    &= \Ebb\bigg[ \Big( \IB - \gamma_t \ind{\xB_t^\top \wB_{t-1}>0} \xB_t \xB_t^\top \Big) (\wB_{t-1}-\wB_*) | \wB_{t-1}\bigg]  \\
    &= \bigg(\IB - \frac{\gamma}{2}\HB \bigg)(\wB_{t-1} - \wB_*),
\end{align*}
where the second equation is due to Assumption~\ref{assump:symmetric:moment}\ref{item:symmetric:moment-2} (or Lemma~\ref{lemma:moments}\ref{item:moment-2}) and Assumption~\ref{assump:noise:well-specified},
and the third equation is due to Assumption~\ref{assump:symmetric:moment}\ref{item:symmetric:moment-2} (or Lemma~\ref{lemma:moments}\ref{item:moment-2}).
Applying the above recursively we obtain that: for $t > s$,
\begin{equation*}
    \Ebb[ \wB_{t} - \wB_* | \wB_{s}]
    = \bigg(\IB - \frac{\gamma}{2}\HB \bigg)^{t-s}(\wB_{s} - \wB_*),
\end{equation*}
which also implies that
\begin{equation}
    \Ebb[ (\wB_{t} - \wB_*)\otimes(\wB_{s}-\wB_{*})]
    = \bigg(\IB - \frac{\gamma}{2}\HB \bigg)^{t-s}\cdot \Ebb (\wB_{s} - \wB_*)^{\otimes 2}= \bigg(\IB - \frac{\gamma}{2}\HB \bigg)^{t-s}\cdot \AB_s.
\end{equation}
Now let us consider $\Ebb (\bar{\wB}_N-\wB_*)^{\otimes 2}$:
\begin{align*}
    &\ \Ebb (\bar{\wB}_N-\wB_*)^{\otimes 2} \\
    &= \frac{1}{N^2} \cdot\bigg(  \Ebb \sum_{t=0}^{N-1} (\wB_t-\wB_*)^{\otimes 2}  + \Ebb \sum_{s=0}^{N-1} \sum_{t=s+1}^{N-1} \Big( (\wB_t-\wB_*) \otimes (\wB_s-\wB_*) +  (\wB_s-\wB_*) \otimes (\wB_t-\wB_*) \Big)\bigg) \\
    &= \frac{1}{N^2} \cdot \Bigg(  \sum_{s=0}^{N-1} \AB_s+ \sum_{s=0}^{N-1} \sum_{t=s+1}^{N-1}\bigg(\bigg(\IB - \frac{\gamma}{2}\HB \bigg)^{t-s} \cdot \AB_s + \AB_s\cdot \bigg(\IB - \frac{\gamma}{2}\HB \bigg)^{t-s}  \bigg) \Bigg).
\end{align*}
The remaining proof simply follows from \citet{zou2021benign}.

\end{proof}

We next present the risk bounds for constant-stepsize GLM-tron with iterate averaging as follows.

\begin{theorem}[Risk Bounds for constant-stepsize GLM-tron]\label{append:thm:tron:iterate-avg}
Suppose that Assumption  \ref{assump:noise:well-specified} holds.
Consider $\bar{\wB}_{N}$ defined in \eqref{eq:tron:average}, i.e., the iterate average of constant stepsize~\eqref{eq:tron}. 
Suppose $N > 100$. 
\begin{enumerate}[label=(\Alph*)]
    \item
If in addition Assumption~\ref{assump:gaussian}\ref{item:gaussian:upper} and 
Assumption~\ref{assump:symmetric:moment}\ref{item:symmetric:moment-2}\ref{item:symmetric:moment-4-2} hold,
then for $\gamma < 1/(4\alpha(\tr(\HB)) )$ it holds that
\begin{align*}
    \Ebb \excessrisk (\bar\wB_{N})
    &\lesssim 
    \frac{1}{N^2\gamma^2}\cdot \big\| \wB_0 - \wB_* \big\|^2_{\HB_{0:k^*}^{-1}} +  \big\| \wB_0 - \wB_* \big\|^2_{\HB_{k^*:\infty}} \\
    &\quad + \bigg( \alpha \cdot \frac{\big\| \wB_0 - \wB_* \big\|^2_{\IB_{0:k^*}} +  N\gamma \cdot \big\| \wB_0 - \wB_* \big\|^2_{\HB_{k^*:\infty} } }{ N\gamma } + \sigma^2 \bigg) \cdot \frac{k^* + \Neff^2 \gamma^2_0 \cdot \sum_{i>k^*} \lambda_i^2}{\Neff},
\end{align*}
where $k^*\ge 0$ is an arbitrary index.
\item
If in addition Assumption~\ref{assump:gaussian}\ref{item:gaussian:lower} and Assumption~\ref{assump:symmetric:moment} hold, then
for $\gamma_0 < 1/\lambda_1$,
it holds that
\begin{align*}
    \Ebb \excessrisk (\wB_{N})
    &\gtrsim 
    \frac{1}{N^2\gamma^2}\cdot \big\| \wB_0 - \wB_* \big\|^2_{\HB_{0:k^*}^{-1}} +  \big\| \wB_0 - \wB_* \big\|^2_{\HB_{k^*:\infty}} \\
    &\quad + \bigg( \beta \cdot \frac{\big\| \wB_0 - \wB_* \big\|^2_{\IB_{0:k^*}} +  N\gamma \cdot \big\| \wB_0 - \wB_* \big\|^2_{\HB_{k^*:\infty} } }{ N\gamma } + \sigma^2 \bigg) \cdot \frac{k^* + \Neff^2 \gamma^2_0 \cdot \sum_{i>k^*} \lambda_i^2}{\Neff},
\end{align*}
where $k^*:= \max\{k: \lambda_k \ge 1/(\gamma_0 \Neff)\}$.
\end{enumerate}
\end{theorem}
\begin{proof}
    We first use Lemma~\ref{append:thm:loss-landscape} and Lemma~\ref{append:thm:tron:covariance} to relate GLM-tron for ReLU regression problems to SGD for linear regression problems. Then we invoke Lemma~\ref{lemma:tron:average} and the proof of Theorems 2.1 and 2.2 in \citet{zou2021benign} to get the results.
\end{proof}


\subsection{Proof of Corollary \ref{thm:well-specified:recover}}
\begin{proof}[Proof of Corollary \ref{thm:well-specified:recover}]
According to the stepsize scheduler \eqref{eq:geometry-tail-decay-lr} and the assumptions, we have that 
\begin{align*}
    \Ebb \excessrisk(\wB_N) &\lesssim \|e^{-0.5 \Neff \gamma_0 \HB} \HB(\wB_0 - \wB_*)\|^2_{2} + \frac{k^* + \Neff^2\gamma_0^2\sum_{i>k^*}\lambda_i^2}{\Neff} \\
    &\lesssim \frac{1}{\Neff \gamma_0} + \frac{k^* + \Neff^2\gamma_0^2\sum_{i>k^*}\lambda_i^2}{\Neff},
\end{align*}
where $k^*$ can be arbitrary.

For the first part, we choose $\gamma_0 = 1/\sqrt{\Neff}$ and $k^* = \max\{k: \lambda_k > 1/\sqrt{\Neff}\}$, then 
from $\tr(\HB) \lesssim 1$ we know that 
\[k^* \lesssim \sqrt{\Neff},\quad \sum_{i>k^*}\lambda_i^2 \lesssim \frac{1}{\sqrt{\Neff}}.\]
Then we have
\begin{align*}
    \Ebb \excessrisk(\wB_N) 
    &\lesssim \frac{1}{\Neff \gamma_0} + \frac{k^* + \Neff^2\gamma_0^2\sum_{i>k^*}\lambda_i^2}{\Neff} 
    \lesssim \frac{1}{\sqrt{\Neff}} + \frac{\sqrt{\Neff} + \Neff^2 \cdot \frac{1}{\Neff} \cdot\frac{1}{\sqrt{\Neff}} }{\Neff} \eqsim \frac{1}{\sqrt{\Neff}}.
\end{align*}

As for the second part, we choose \(\gamma \eqsim  1/\tr(\HB)\), and \( k^* := d\), then 
\begin{align*}
    \Ebb \excessrisk(\wB_N) 
    &\lesssim \frac{1}{\Neff \gamma_0} + \frac{k^* + \Neff^2\gamma_0^2\sum_{i>k^*}\lambda_i^2}{\Neff} 
    \lesssim\frac{\tr(\HB)}{\Neff}+\frac{d}{\Neff} \eqsim \frac{d}{\Neff}.
\end{align*}

\end{proof}

\section{Misspecified Setting}\label{append:sec:misspecified}
In this part, we consider the misspecified setting and assume Assumption~\ref{assump:noise:misspecified}.

\paragraph{Notations.}
In this section, we assume that $\HB$ is diagonal.
For 
\( \AB_t:= \Ebb (\wB_t - \wB_*)^2,\)
we use
\( \mathring{\AB}_t\)
to refer to the diagonal of $\AB_t$.
For simplicity, we will use
\[\epsilon_t := y_t - \relu(\wB_*^\top \xB_t)\]
to refer to the misspecified noise in this section.

One technique we used for dealing with misspecified cases is to study the diagonal, instead of the matrix itself, of the expected outer product of the error iterates.
The following lemma is useful for translating inequalities about PSD matrices to inequalities about their diagonals. 

\begin{lemma}\label{lemma:diagonal}
    For every pair of symmetric matrices $\AB$ and $\BB$,  $\AB \preceq \BB$ implies $\mathring{\AB} \preceq \mathring{\BB}$.
\end{lemma}
\begin{proof}
We only need to show that $\diag(\BB- \AB)$ is PSD.
This holds because every diagonal entry of a PSD matrix must be non-negative. 
\end{proof}

\subsection{Risk Landscape}
We first show the following lemma about an upper bound on the risk.
\begin{lemma}[Risk landscape, misspecified case]\label{lemma:misspecified:loss-landscape}
Under Assumption~\ref{assump:noise:misspecified}, it holds that
    \[\risk (\wB) \le 2\cdot \|\wB - \wB_* \|_{\HB}^2 + 2\cdot \OPT.\]
\end{lemma}
\begin{proof}
We prove the conclusion as follows:
    \begin{align*}
        \risk (\wB) 
        &:= \Ebb \big( \relu({\wB}^\top\xB) - y\big)^2 \\
        &= \Ebb \big( \relu({\wB}^\top\xB) - \relu({\wB}_*^\top\xB) + \relu({\wB}_*^\top\xB) - y\big)^2  \\
        &\le 2\cdot \Ebb \big( \relu({\wB}^\top\xB) - \relu({\wB}_*^\top\xB)\big)^2 + 2\cdot \Ebb \big(\relu({\wB}_*^\top\xB) - y\big)^2  \\
        &\le 2\cdot \Ebb \big( {\wB}^\top\xB - {\wB}_*^\top\xB \big)^2 + 2\cdot \Ebb \big(\relu({\wB}_*^\top\xB) - y\big)^2  \\
        &= 2\cdot \|\wB - \wB_* \|_{\HB}^2 + 2\cdot \OPT,
    \end{align*}
    where in the last inequality we use the fact that $\relu(\cdot)$ is $1$-Lipschitz.
\end{proof}

\subsection{Iterate Bounds}

\begin{lemma}[Iterate upper bound]\label{thm:misspecified:covariance-bound}
    Suppose that Assumption~\ref{assump:noise:misspecified}, Assumption~\ref{assump:symmetric:moment}\ref{item:symmetric:moment-2} and Assumption~\ref{assump:gaussian}\ref{item:gaussian:upper} hold, then the following holds for \eqref{eq:tron}:
    \[
        \mathring{\AB}_{t} \preceq \Big(\IB -  \frac{\gamma_t}{2} \cdot \HB\Big)\cdot \mathring{\AB}_{t-1}  + 2\alpha\gamma_t^2 \cdot \la\HB, \mathring\AB_{t-1}\ra\cdot \HB + 3\gamma_t^2 (\sigma^2 + \alpha\| \wB_*\|_{\HB}^2) \cdot \HB + 2\gamma_t \cdot \XiB,
\]
where $\XiB$ is a diagonal deterministic matrix and $\tr(\XiB) \le \OPT$.
\end{lemma}
\begin{proof}
We first consider the expected outer product of \eqref{eq:tron:wt-w*} in the misspecified setting:
\begin{align}
\AB_t &:= \Ebb \big( \wB_{t} - \wB_*  \big)^{\otimes 2} \notag \\ 
&  \left.  
    \begin{array}{l}
    = \Ebb \Big( \IB - \gamma_t \ind{\xB_t^\top \wB_{t-1}>0}\cdot \xB_t \xB_t^\top \Big)^{\otimes 2} \circ  (\wB_{t-1}-\wB_*)^{\otimes 2} \\
     \quad+\gamma_t^2 \cdot \Ebb  \big(  \ind{\xB_t^\top \wB_*>0} - \ind{\xB_t^\top \wB_{t-1}>0}\big)^2 \cdot \xB_t \xB_t^\top \wB_*\wB_*^\top \xB_t \xB_t^\top\\
    \quad+ \gamma_t\cdot \Ebb  \big(  \ind{\xB_t^\top \wB_*>0} - \ind{\xB_t^\top \wB_{t-1}>0}\big)\cdot \xB_t \xB_t^\top \wB_*(\wB_{t-1}-\wB_*)^\top \Big( \IB - \gamma_t \ind{\xB_t^\top \wB_{t-1}>0}\cdot \xB_t \xB_t^\top \Big)  \\
    \quad+\gamma_t\cdot \Ebb \big(  \ind{\xB_t^\top \wB_*>0} - \ind{\xB_t^\top \wB_{t-1}>0}\big)\cdot \Big( \IB - \gamma_t \ind{\xB_t^\top \wB_{t-1}>0}\cdot \xB_t \xB_t^\top \Big) (\wB_{t-1}-\wB_*)\wB_*^\top \xB_t \xB_t^\top  
    \end{array}
\right\}=: \SB \notag \\
&  \left.  
    \begin{array}{l}
     \quad+ \gamma_t^2\cdot \Ebb \big[ \epsilon_t^2 \xB_t\xB_t^\top\big] \\
     \quad+ \gamma_t \cdot \Ebb \Big[ \epsilon_t\Big( (\wB_{t-1}-\wB_*)\xB^\top +  \xB(\wB_{t-1}-\wB_*)^\top \Big) \Big] \\
     \quad-  2\gamma_t^2 \cdot \Ebb \Big[ \epsilon_t \ind{\xB^\top_t\wB_{t-1} >0}\cdot \xB_t^\top (\wB_{t-1} - \wB_*) \cdot \xB_t \xB_t^\top \Big] \\
     \quad+ 2\gamma_t^2 \cdot \Ebb \Big[ \epsilon_t \big( \ind{\xB^\top_t\wB_{*} >0} - \ind{\xB^\top_t\wB_{t-1}>0} \big)\cdot \xB_t^\top \wB_* \cdot \xB_t \xB_t^\top \Big], 
    \end{array}
\right\}=: \NB \notag 
\end{align}
where we decompose $\AB_t$ into a signal part and a noise part, i.e., $\AB_t := \SB + \NB$.
We next upper bound these two parts separately.

\paragraph{Signal Part.}
The analysis of this part is similar to the derivation of \eqref{eq:gaussian:tron:out-product:upper} in the proof of Theorem \ref{thm:tron:covariance}.
However this time we only use Assumption~\ref{assump:symmetric:moment}\ref{item:moment-2} and do not use Assumption~\ref{assump:symmetric:moment}\ref{item:symmetric:moment-4-2}.
In specific, 
under Assumption~\ref{assump:symmetric:moment}\ref{item:moment-2}, \eqref{eq:gaussian:tron:quad-term2:ver0} and \eqref{eq:gaussian:tron:cross-term:2} still hold, and applying which to the signal part $\SB$  we obtain 
\begin{align*}
      \SB
      &=  \Ebb\Big( \IB - \gamma_t \ind{\xB_t^\top \wB_{t-1}>0}\cdot \xB_t \xB_t^\top \Big)^{\otimes 2} \circ  (\wB_{t-1}-\wB_*)^{\otimes 2} \notag \\
      &\quad + 2 \gamma_t^2 \cdot \Ebb\bigg(\ind{\xB_t^\top \wB_{t-1}>0, \xB_t^\top \wB_*<0}\cdot \xB_t^\top \wB_* \cdot \xB_t^\top \wB_{t-1}\cdot \xB_t \xB_t^\top\bigg) \notag \\
      &\quad  -  2 \gamma_t^2 \cdot \Ebb\bigg( \ind{\xB_t^\top \wB_{t-1}>0, \xB_t^\top \wB_*<0}\cdot \big(\xB_t^\top \wB_*\big)^2\cdot \xB_t \xB_t^\top\bigg) \notag \\
      & \quad + \gamma_t^2 \cdot \Ebb \bigg(\Big(  \ind{\xB_t^\top \wB_{t-1}>0, \xB_t^\top \wB_*<0} + \ind{\xB_t^\top \wB_{t-1}<0, \xB_t^\top \wB_*>0}\Big)\cdot \big(\xB_t^\top \wB_*\big)^2  \cdot \xB_t \xB_t^\top  \bigg).  \notag 
\end{align*}
In the above, the second term is always non-positive due to the property of the indicator function; and the third and fourth terms together is equal to
\begin{align*}
    & \gamma_t^2 \cdot \Ebb \bigg(\Big( \ind{\xB_t^\top \wB_{t-1}<0, \xB_t^\top \wB_*>0}- \ind{\xB_t^\top \wB_{t-1}>0, \xB_t^\top \wB_*<0} \Big)\cdot \big(\xB_t^\top \wB_*\big)^2  \cdot \xB_t \xB_t^\top  \bigg) \\
    &\le \gamma_t^2 \cdot \Ebb \big( (\xB_t^\top \wB_*)^2  \cdot \xB_t \xB_t^\top  \big)
    = \gamma_t^2 \cdot\Mcal\circ (\wB_*\wB_*^\top), 
\end{align*}
so the signal part can be bounded by 
\begin{align*}
    \SB
      &\preceq  \Ebb\Big( \IB - \gamma_t \ind{\xB_t^\top \wB_{t-1}>0}\cdot \xB_t \xB_t^\top \Big)^{\otimes 2} \circ  (\wB_{t-1}-\wB_*)^{\otimes 2}
      + \gamma_t^2 \cdot\Mcal\circ (\wB_*\wB_*^\top).
\end{align*}
Now use Assumption~\ref{assump:symmetric:moment}\ref{item:symmetric:moment-2} (or Lemma~\ref{lemma:moments}\ref{item:moment-2}) and Assumption~\ref{assump:gaussian}\ref{item:gaussian:upper}, we obtain
\begin{align}
    \SB 
    &\preceq \AB_{t-1} - \frac{\gamma_t}{2} \big( \HB \AB_{t-1} + \AB_{t-1} \HB) + \gamma^2_t \Mcal \circ \AB_{t-1} + \gamma_t^2 \cdot\Mcal\circ (\wB_*\wB_*^\top) \notag \\
    &\preceq \AB_{t-1} - \frac{\gamma_t}{2} \big( \HB \AB_{t-1} + \AB_{t-1} \HB) + \gamma^2_t \Mcal \circ \AB_{t-1} + \alpha \gamma_t^2 \|\wB_*\|_{\HB}^2\cdot \HB. \label{eq:tron:agnostic:signal}
\end{align}

\paragraph{Noise Part.}
For the noise part, 
we apply Cauchy inequality to obtain
\begin{align*}
    \NB &: = \gamma_t^2\cdot \Ebb \big[ \epsilon_t^2 \xB_t\xB_t^\top\big] 
    + \gamma_t \cdot \Ebb \Big[ \epsilon_t\Big( (\wB_{t-1}-\wB_*)\xB^\top +  \xB(\wB_{t-1}-\wB_*)^\top \Big) \Big] \\
     &\quad  -  2\gamma_t^2 \cdot \Ebb \Big[ \epsilon_t \ind{\xB^\top_t\wB_{t-1} >0}\cdot \xB_t^\top (\wB_{t-1} - \wB_*) \cdot \xB_t \xB_t^\top \Big] \\
    &\quad+ 2\gamma_t^2 \cdot \Ebb \Big[ \epsilon_t \big( \ind{\xB^\top_t\wB_{*}>0} - \ind{\xB^\top_t\wB_{t-1}>0} \big)\cdot \xB_t^\top \wB_* \cdot \xB_t \xB_t^\top \Big] \\
     &\preceq  \gamma_t^2\cdot \Ebb \big[ \epsilon_t^2 \xB_t\xB_t^\top\big] 
    + \gamma_t \cdot \Ebb \Big[ \epsilon_t\Big( (\wB_{t-1}-\wB_*)\xB^\top +  \xB(\wB_{t-1}-\wB_*)^\top \Big) \Big] \\
     & \quad +  \gamma_t^2 \cdot \Ebb \Big[ \Big( \epsilon_t^2 +   \big(\xB_t^\top (\wB_{t-1} - \wB_*) \big)^2\Big) \cdot \xB_t \xB_t^\top \Big] 
     + \gamma_t^2 \cdot \Ebb \Big[ \big( \epsilon_t^2 + ( \xB_t^\top \wB_*)^2\big) \cdot \xB_t \xB_t^\top \Big].
\end{align*}
Next we apply Assumption~\ref{assump:gaussian}\ref{item:gaussian:upper} and Assumption~\ref{assump:noise:misspecified} to obtain 
\begin{align*}
    \NB &\preceq   \gamma_t^2\sigma^2 \cdot \HB
     + \gamma_t \cdot \Ebb \Big[ \epsilon_t\cdot \Big( (\wB_{t-1}-\wB_*)\xB^\top +  \xB(\wB_{t-1}-\wB_*)^\top \Big) \Big] \\
     & \quad +  \gamma_t^2 \cdot\Big( \sigma^2 \cdot \HB + \Mcal \circ \AB_{t-1}\Big)
     + \gamma_t^2 \cdot \Big( \sigma^2 \cdot \HB + \alpha \tr(\HB \wB_*\wB_*^\top) \cdot\HB \Big) \\
     &= 3\gamma_t^2\sigma^2 \cdot \HB + \alpha\gamma_t^2 \|\wB_*\|_{\HB}^2\cdot \HB +  \gamma_t^2 \cdot \Mcal\circ \AB_{t-1}  
     + \gamma_t \cdot \Ebb \Big[ \epsilon_t\cdot\Big( (\wB_{t-1}-\wB_*)\xB^\top +  \xB(\wB_{t-1}-\wB_*)^\top \Big) \Big].
\end{align*}
Next, we take diagonal over the above inequality and apply Lemma~\ref{lemma:diagonal} and Lemma~\ref{lemma:tron:agnostic:cross}, then we obtain 
\begin{align}
    \mathring{\NB}
    &\preceq 3\gamma_t^2\sigma^2 \cdot \HB + \alpha\gamma_t^2 \|\wB_*\|_{\HB}^2\cdot \HB + \gamma_t^2 \cdot \diag(\Mcal \circ \AB_{t-1}) \notag\\
    &\quad + \gamma_t \cdot \Ebb \Big[ \epsilon_t\cdot \diag \Big( (\wB_{t-1}-\wB_*)\xB^\top +  \xB(\wB_{t-1}-\wB_*)^\top \Big) \Big] \notag \\
    &\preceq 3\gamma_t^2\sigma^2 \cdot \HB + \alpha\gamma_t^2 \|\wB_*\|_{\HB}^2\cdot \HB + \gamma_t^2 \cdot \diag(\Mcal \circ \AB_{t-1}) + \frac{\gamma_t}{2} \cdot \HB\mathring\AB_{t-1} + 2\gamma_t\cdot \XiB,\label{eq:tron:agnostic:noise}
\end{align}
where $\XiB$ is a deterministic diagonal PSD matrix and that $\tr(\XiB)\le \OPT$.

\paragraph{Combining Two Parts.}
Combining the diagonal of \eqref{eq:tron:agnostic:signal} with \eqref{eq:tron:agnostic:noise}, we have
\begin{align*}
    \mathring{\AB}_t 
    &= \mathring{\SB} + \mathring\NB  \\
    &\preceq \mathring{\AB}_{t-1} - \gamma_t \cdot\HB\mathring{\AB}_{t-1} + {\gamma_t^2} \cdot \diag ( \Mcal \circ {\AB}_{t-1} ) +  \alpha\gamma_t^2 \|\wB_*\|_{\HB}^2\cdot \HB \\
    &\quad + 3\gamma_t^2\sigma^2 \cdot \HB + \alpha\gamma_t^2 \|\wB_*\|_{\HB}^2\cdot \HB + \gamma_t^2 \cdot\diag(\Mcal \circ \AB_{t-1} ) + \frac{\gamma_t}{2}\cdot \HB \mathring\AB_{t-1}+ 2\gamma_t \cdot\XiB  \\
    &\preceq \Big(\IB -  \frac{\gamma_t}{2} \cdot \HB\Big)\cdot \mathring{\AB}_{t-1} + 2\gamma_t^2 \cdot \diag(\Mcal \circ\AB_{t-1}) + 3\gamma_t^2 (\sigma^2 + \alpha\| \wB_*\|_{\HB}^2) \cdot \HB + 2\gamma_t \cdot \XiB \\
    &\preceq \Big(\IB -  \frac{\gamma_t}{2} \cdot \HB\Big)\cdot \mathring{\AB}_{t-1}  + 2\alpha\gamma_t^2 \cdot \la\HB, \mathring\AB_{t-1}\ra\cdot \HB + 3\gamma_t^2 (\sigma^2 + \alpha\| \wB_*\|_{\HB}^2) \cdot \HB + 2\gamma_t \cdot \XiB,
\end{align*}
where in the last inequality we applied Assumption~\ref{assump:gaussian}\ref{item:gaussian:upper}.
We have completed the proof.
\end{proof}

\begin{lemma}\label{lemma:tron:agnostic:cross}
In the setting of Lemma~\ref{thm:misspecified:covariance-bound}, 
it holds that
\begin{align*}
    \gamma_t \cdot \Ebb \Big[ \epsilon_t\cdot \diag \Big( (\wB_{t-1}-\wB_*)\xB_t^\top +  \xB_t(\wB_{t-1}-\wB_*)^\top \Big) \Big] 
    \preceq \frac{\gamma_t}{2} \cdot \HB \mathring\AB_{t-1} + 2\gamma_t \cdot\XiB,
\end{align*}
where $\XiB$ is a fixed diagonal matrix and that $\tr(\XiB) \le \OPT$.
\end{lemma}
\begin{proof}
Define a fixed vector
\[\aB := \Ebb[\epsilon_t \HB^{-\half}\xB_t]. \]
Recall that $\HB$ is a diagonal matrix, so $\HB$ commutes with any diagonal matrix. Then we have
\begin{align*}
    \Ebb \Big[ \epsilon_t\cdot \diag \Big( (\wB_{t-1}-\wB_*)\xB_t^\top +  \xB_t(\wB_{t-1}-\wB_*)^\top \Big) \Big] 
   &= \Ebb \Big[ 2\epsilon_t\cdot \diag \Big( (\wB_{t-1}-\wB_*)\xB_t^\top \Big) \Big] \\
   &= \Ebb \Big[ 2\cdot \diag \Big( \HB^{\half}(\wB_{t-1}-\wB_*) \cdot \epsilon_t\xB_t^\top\HB^{-\half} \Big) \Big] \\
   &= \Ebb \Big[ 2\cdot \diag \Big( \HB^{\half}(\wB_{t-1}-\wB_*) \cdot \aB^\top \Big) \Big],
\end{align*}
where in the last equation we take (conditional) expectation over the fresh randomness introduced by $\epsilon_t$ and $ \xB_t$.
Now use the fact that: for every two vectors $\uB, \vB$ it holds that
\[\uB \vB^\top + \vB \uB^\top \preceq \uB\uB^\top + \vB \vB^\top, \]
we then obtain
\begin{align*}
   & \Ebb \Big[ \epsilon_t\cdot \diag \Big( (\wB_{t-1}-\wB_*)\xB_t^\top +  \xB_t(\wB_{t-1}-\wB_*)^\top \Big) \Big] \\
    &= \Ebb \Big[ 2\cdot \diag \Big( \frac{1}{\sqrt{2}}\HB^{\half}(\wB_{t-1}-\wB_*) \cdot \sqrt{2}\aB^\top \Big) \Big]\\
   &\preceq \Ebb\bigg[ \diag\bigg( \half \HB^{\half}(\wB_{t-1}-\wB_*)(\wB_{t-1}-\wB_*)^\top\HB^{\half} + 2\aB \aB^\top   \bigg)\bigg] \\
   &= \half \cdot \diag(\HB \AB_{t-1})+2\cdot\diag\big(\aB \aB^\top\big).
\end{align*}
Moreover, notice that
\begin{align*}
    \aB^\top \aB = \Ebb [\epsilon_t \xB_t^\top \HB^{-\half} \aB]
    \le \half \Ebb[ \epsilon_t^2 + \aB^\top \HB^{-\half }\xB_t \xB_t^\top \HB^{-\half}\aB ] 
    = \half \OPT  + \half \aB^\top \aB ,
\end{align*}
which implies that $\aB^\top \aB \le \OPT$, so it holds that 
\[\tr(\diag\big( \aB \aB^\top\big)) = \tr(\aB\aB^\top) = \aB^\top \aB \le \OPT.\]
We have completed the proof by setting $\XiB := \diag(\aB\aB^\top)$ and noting that $\diag(\HB \AB_{t-1}) = \HB \mathring{\AB}_{t-1}$.
\end{proof}

\subsection{Proof of Theorem \ref{thm:tron:gaussian:misspecified}}
We will prove the following slightly stronger version.
\begin{theorem}[Risk Bounds for GLM-tron, restated Theorem \ref{thm:tron:gaussian:misspecified}]\label{append:thm:tron:gaussian:misspecified}
Suppose that Assumption~\ref{assump:noise:misspecified}, Assumption~\ref{assump:symmetric:moment}\ref{item:symmetric:moment-2} and Assumption~\ref{assump:gaussian}\ref{item:gaussian:upper} hold.
Let $\wB_{N}$ be the output of \eqref{eq:tron} with stepsize scheduler \eqref{eq:geometry-tail-decay-lr}.
Assume that $N > 100$. 
Let $\Neff := N / \log(N)$.
Then for $\gamma_0 < 1/(8\alpha(\tr(\HB)) )$, it holds that
\begin{align*}
   \Ebb[ \risk (\wB_{N})]
    & \lesssim \OPT + 
    \bigg\| \prod_{t=0}^{N-1}\Big(\IB-\frac{\gamma_t}{2}\HB\Big)(\wB_0 - \wB_*) \bigg\|^2_\HB  \\
    & \quad + \bigg({\alpha\Big(\OPT +\|\wB_*\|_{\HB}^2 + \big\| \wB_0 - \wB_* \big\|^2_{\frac{\IB_{0:k^*} }{\Neff \gamma} + \HB_{k^*:\infty}} \Big) + \sigma^2 }\bigg) \cdot  \frac{k^* + \Neff^2\gamma_0^2 \sum_{i>k^*}\lambda_i^2}{\Neff},
\end{align*}
where $k^* \ge 0$ can be any index.
\end{theorem}
\begin{proof}
First of all, by Lemma~\ref{lemma:misspecified:loss-landscape}, it holds that 
\begin{align*}
    \Ebb[ \risk (\wB_{N})]
    &\le 2 \cdot \Ebb \|\wB_N - \wB_* \|^2_{\HB} + 2\cdot \OPT \\
    &= 2\cdot \la \HB, \mathring{\AB} \ra + 2\cdot \OPT.
\end{align*}
Now consider the recursion of  $\mathring{\AB}_t$ given in Lemma~\ref{thm:misspecified:covariance-bound}.
Note that $\mathring{\AB}_t$ is related to $\mathring{\AB}_{t-1}$ through a linear operator,
therefore 
$\mathring{\AB}_t$ can be understood as the sum of two iterates, i.e., $\mathring{\AB}_t := \mathring{\BB}_t + \mathring{\CB}_t $, where 
\begin{align*}
    \begin{cases}
        \mathring{\BB}_t \preceq \big( \IB  - \frac{\gamma_t}{2} \cdot \HB\big) \cdot \mathring{\BB}_{t-1} + 2\alpha \gamma_t^2\cdot  \la \HB, \mathring\BB_{t-1}\ra \cdot \HB;  \\
        \mathring{\BB}_0 := \diag( (\wB_0 - \wB_*)^{\otimes 2} ),
    \end{cases}
\end{align*}
and
\begin{align*}
    \begin{cases}
        \mathring{\CB}_t \preceq \big(\IB - \frac{\gamma_t}{2} \cdot \HB\big) \cdot \mathring{\CB}_{t-1} + 2\alpha \gamma_t^2 \cdot \la \HB, \mathring\CB_{t-1}\ra\cdot \HB + 3\gamma_t^2 (\sigma^2 + \alpha\| \wB_*\|_{\HB}^2) \cdot \HB + 2\gamma_t\cdot \XiB; \\
        \mathring{\CB}_0 := 0.
    \end{cases}
\end{align*}
Then we have 
\begin{align*}
    \Ebb[ \risk (\wB_{N})]
    &\le 2\cdot \la \HB, \mathring{\BB} \ra+2\cdot \la \HB, \mathring{\CB} \ra + 2\cdot \OPT.
\end{align*}

\paragraph{Bounding the Bias Error $\la \HB, \mathring{\BB} \ra$.}
Note that $\mathring\BB_t$ is exactly the diagonal of the bias iterate in  \citet{wu2022iterate,wu2022power}, ignoring a difference in constant factors in the stepsizes.
So by the proof of the bias part of Corollary 3.3 in \citet{wu2022power}, we have 
\begin{align*}
   \la \HB, \mathring{\BB} \ra
    & \lesssim 
    \bigg\| \prod_{t=1}^{N}\Big(\IB-\frac{\gamma_t}{2}\HB\Big)(\wB_0 - \wB_*) \bigg\|^2_\HB  + \alpha\cdot \big\| \wB_0 - \wB_* \big\|^2_{\frac{\IB_{0:k^*} }{\Neff \gamma} + \HB_{k^*:\infty}}   \cdot  \frac{k^* + \Neff^2\gamma_0^2 \sum_{i>k^*}\lambda_i^2}{\Neff}.
\end{align*}

\paragraph{Bounding the Variance Error $\la \HB, \mathring{\CB} \ra$.}
However $\mathring{\CB}_t$ is slightly different from the variance iterate in \citet{wu2022iterate,wu2022power}, as the noise structure is different due to the appearance of $\XiB$. But a similar analysis idea applies here.

We first derive a crude upper bound on $\mathring\CB_t$ in Lemma~\ref{lemma:tron:misspecified:crude-var-bound}:
 \[\mathring{\CB}_t \preceq \rho\gamma \cdot \IB + 4\cdot \HB^{-1}\XiB,\quad \text{where}\ \rho := \frac{16\alpha\OPT + 6(\sigma^2+\alpha\|\wB_*\|^2_{\HB})}{1-4\gamma\alpha\tr(\HB)}, \quad t\ge 0.\]
Then we establish a sharper bound based on Lemma~\ref{lemma:tron:misspecified:crude-var-bound} as follows:
\begin{align*}
    \mathring{\CB}_t 
        &\preceq \Big(\IB -  \frac{\gamma_t}{2} \cdot \HB \Big)\cdot \mathring{\CB}_{t-1} + 2\alpha\gamma_t^2 \cdot \la \HB,  \mathring\CB_{t-1}\ra \cdot \HB + 3\gamma_t^2 (\sigma^2 + \alpha\| \wB_*\|_{\HB}^2) \cdot \HB + 2\gamma_t\cdot \XiB \\ 
        &\preceq \Big(\IB -  \frac{\gamma_t}{2}\cdot \HB \Big)\cdot \mathring{\CB}_{t-1} + 2\alpha \gamma_t^2 \big( \rho\gamma\tr(\HB) + 4\OPT\big)\cdot  \HB + 3\gamma_t^2 (\sigma^2 + \alpha\| \wB_*\|_{\HB}^2) \cdot \HB + 2\gamma_t\cdot \XiB \\
        &= \bigg(\IB -  \frac{\gamma_t}{2} \cdot\HB \bigg)\cdot\mathring{\CB}_{t-1} + \big( 2\alpha\rho\gamma\tr(\HB) + 8\alpha\OPT + 3 (\sigma^2 + \alpha\| \wB_*\|_{\HB}^2) \big)\cdot \gamma_t^2\cdot \HB  + 2\gamma_t\cdot \XiB \\
        &\preceq \bigg(\IB -  \frac{\gamma_t}{2} \HB \bigg)\mathring{\CB}_{t-1} + \big(  16\alpha\OPT + 6 (\sigma^2 + \alpha\| \wB_*\|_{\HB}^2) \big)\cdot \gamma_t^2\cdot \HB  + 2\gamma_t\cdot \XiB,
\end{align*}
where the second inequality is by Lemma~\ref{lemma:tron:misspecified:crude-var-bound}; and in the last inequality we use the assumption that
\[\gamma < \frac{1}{8\alpha\tr(\HB)},\]
so that 
\[\rho := \frac{16\alpha\OPT + 6(\sigma^2+\alpha \|\wB_*\|^2_{\HB})}{1-4\gamma\alpha\tr(\HB)} \le 32\alpha\OPT + 12(\sigma^2+\alpha\|\wB_*\|^2_{\HB}),\]
which together imply 
\[2\alpha \rho \gamma \tr(\HB) \le \frac{\rho}{4} \le 8 \alpha\OPT + 3(\sigma^2+\alpha\|\wB_*\|^2_{\HB}).\]

We then solve the recursion and obtain 
\begin{align*}
    \mathring{\CB}_N 
    \preceq \big(  16\alpha\OPT + 6 (\sigma^2 + \alpha\| \wB_*\|_{\HB}^2) \big)\cdot \sum_{t=1}^N \gamma_t^2\prod_{i=t+1}^N \bigg(\IB- \frac{\gamma_t}{2}\HB\bigg) \cdot \HB  + 2 \sum_{t=1}^N \gamma_t \prod_{i=t+1}^N \bigg(\IB- \frac{\gamma_t}{2}\HB\bigg) \cdot \XiB.
\end{align*}
Finally we use Lemma~\ref{lemma:tron:misspecified:prod-sum-bounds} and obtain 
\begin{align*}
    \mathring{\CB}_N 
    \preceq 8 \big(  16\alpha\OPT + 6 (\sigma^2 + \alpha\| \wB_*\|_{\HB}^2) \big)\cdot \bigg( \frac{1}{\Neff} \HB^{-1}_{0:k} + \Neff \gamma^2 \HB_{k:\infty}\bigg)  + 32 \HB^{-1} \XiB.
\end{align*}
So it holds that
\begin{align*}
\big\la \HB, \mathring{\CB}_N \big\ra 
&\le  8 \big(  16\alpha\OPT + 6 (\sigma^2 + \alpha\| \wB_*\|_{\HB}^2) \big)\cdot \frac{k^* + \Neff^2\gamma_0^2 \sum_{i>k^*}\lambda_i^2}{\Neff} + 32 \tr(\XiB) \\
&\le 8 \big(  16\alpha\OPT + 6 (\sigma^2 + \alpha\| \wB_*\|_{\HB}^2) \big)\cdot \frac{k^* + \Neff^2\gamma_0^2 \sum_{i>k^*}\lambda_i^2}{\Neff} + 32 \OPT.
\end{align*}

Putting everything together completes the proof.
\end{proof}

\subsection{Some Auxiliary Lemmas}

\begin{lemma}[A crude variance upper bound]\label{lemma:tron:misspecified:crude-var-bound}
Consider a sequence of variance iterates defined as follows:
\begin{align*}
    \begin{cases}
        \mathring{\CB}_t \preceq \mathring{\CB}_{t-1} - \frac{\gamma_t}{2} \cdot \HB\mathring{\CB}_{t-1} + 2\alpha\gamma_t^2 \cdot \la\HB, \mathring\CB_{t-1}\ra \cdot \HB + 3\gamma_t^2 (\sigma^2 + \alpha\| \wB_*\|_{\HB}^2) \cdot \HB + 2\gamma_t\cdot \XiB; \\
        \mathring{\CB}_0 := 0,
    \end{cases}
\end{align*}
where $\XiB$ is deterministic and $\tr(\XiB) \le \OPT$. 
Then for $\gamma < 1/(4\alpha\tr(\HB))$, it holds that
    \[\mathring{\CB}_t \preceq \rho\gamma \cdot \IB + 4\cdot \HB^{-1}\XiB,\quad \text{where}\ \rho := \frac{16\alpha\OPT + 6(\sigma^2+\alpha\|\wB_*\|^2_{\HB})}{1-4\gamma\alpha\tr(\HB)}, \quad t\ge 0.\]
\end{lemma}
\begin{proof}
    We show it by induction.
    For $t=0$ the conclusion holds because $\mathring{\CB}_0 = 0$.
    Now suppose that 
    \[\mathring{\CB}_{t-1} \preceq \rho\gamma  \IB + 4\HB^{-1}\XiB,\]
    then 
    \[ \la\HB,\ \mathring{\CB}_{t-1} \ra
    \le  \rho\gamma \tr(\HB)  + 4\tr(\XiB) 
    \le\rho\gamma \tr(\HB) + 4 \OPT .\]
    Then 
    \begin{align*}
        \mathring{\CB}_t 
        &\preceq \bigg(\IB -  \frac{\gamma_t}{2} \HB \bigg)\mathring{\CB}_{t-1} + 2\alpha \gamma_t^2 \cdot \la \HB, \mathring\CB_{t-1}\ra\cdot \HB + 3\gamma_t^2 (\sigma^2 + \alpha \| \wB_*\|_{\HB}^2) \cdot \HB + 2\gamma_t\cdot \XiB \\
        &\preceq   \bigg(\IB -  \frac{\gamma_t}{2} \HB \bigg)\big( \rho\gamma \IB + 4\HB^{-1}\XiB\big) + 2\alpha\gamma_t^2\big(\rho \gamma \tr(\HB) + 4 \OPT \big)\cdot \HB +  3\gamma_t^2 (\sigma^2 + \alpha\| \wB_*\|_{\HB}^2) \cdot \HB + 2\gamma_t\cdot \XiB\\
        &=  \big( \rho\gamma \IB + 4\HB^{-1}\XiB \big) + \gamma_t \HB \cdot \bigg( - \frac{\rho\gamma}{2} + 2\alpha \gamma_t  \big(\rho \tr(\HB) + 4 \OPT \big) + 3\gamma_t (\sigma^2 + \alpha\| \wB_*\|_{\HB}^2)  \bigg) \\
        &\le \big( \rho\gamma \IB + 4\HB^{-1}\XiB \big) + \gamma_t \HB \cdot \bigg( - \frac{\rho\gamma}{2} + 2\alpha \gamma  \big(\rho \tr(\HB) + 4 \OPT \big) + 3\gamma (\sigma^2 + \alpha\| \wB_*\|_{\HB}^2)  \bigg) \\
        &= \rho\gamma \IB + 4\HB^{-1}\XiB.
    \end{align*}
    We have completed the proof.
\end{proof}

\begin{lemma}[Some technical bounds]\label{lemma:tron:misspecified:prod-sum-bounds}
It holds that 
\begin{enumerate}[label=(\Alph*)]
    \item     \(\sum_{t=1}^N \gamma_t^2\prod_{i=t+1}^N \bigg(\IB- \frac{\gamma_t}{2}\HB\bigg) \cdot \HB\preceq 8\cdot \bigg( \frac{1}{\Neff} \HB^{-1}_{0:k} + \Neff \gamma_0^2 \HB_{k:\infty}\bigg). \)
    \item \(\sum_{t=1}^N \gamma_t \prod_{i=t+1}^N \bigg(\IB- \frac{\gamma_t}{2}\HB\bigg) \preceq 16\cdot \HB^{-1}.\)
    \item For $k^*:= \max\{k: \lambda_k \ge 1/(\gamma_0 \Neff)\}$, it holds that \[\sum_{t=1}^N \gamma_t^2\prod_{i=t+1}^N \big(\IB- {\gamma_t}\HB\big) \cdot \HB\succeq \frac{1}{400}\cdot \bigg( \frac{1}{\Neff} \HB^{-1}_{0:k^*} + \Neff \gamma_0^2 \HB_{k^*:\infty}\bigg). \]
\end{enumerate}
\end{lemma}
\begin{proof}
    The first result is from the proof of Theorem 5 in \citet{wu2022iterate}.
    The third result is from the proof of Theorem 7 in \citet{wu2022iterate}.
    The second result can be proved in a similar manner. 
By definition, we have 
\begin{align*}
    \sum_{t=1}^N \gamma_t \prod_{i=t+1}^N \bigg(\IB- \frac{\gamma_t}{2}\HB\bigg)
    &= \sum_{\ell=0}^{L-1} \frac{\gamma}{2^\ell}\cdot \sum_{i=1}^{\Neff} \bigg(\IB-\frac{\gamma}{2^{\ell+1}} \HB \bigg)^{\Neff - i}\cdot\prod_{j=\ell+1}^{L-1} \bigg( \IB - \frac{\gamma}{2^{j+1}}\HB \bigg)^{\Neff} \\
    &= 2\HB^{-1}\cdot \sum_{\ell=0}^{L-1}  \Bigg( \IB- \bigg(\IB-\frac{\gamma}{2^{\ell+1}} \HB \bigg)^{\Neff }\Bigg) \cdot\prod_{j=\ell+1}^{L-1} \bigg( \IB - \frac{\gamma}{2^{j+1}}\HB \bigg)^{\Neff} \\
    & \preceq 2\HB^{-1}\cdot \sum_{\ell=0}^{L-1}  \bigg( \Neff \cdot\frac{\gamma}{2^{\ell + 1}} \HB \bigg)\cdot \prod_{j=\ell+1}^{L-1} \bigg( \IB - \frac{\gamma}{2^{j+1}}\HB \bigg)^{\Neff} \\
    &=: 2\Neff \HB^{-1} \cdot f(\gamma \HB),
\end{align*}
where 
\[f(x) := \sum_{\ell=0}^{L-1} \frac{x}{2^{\ell + 1}} \cdot  \prod_{j=\ell+1}^{L-1} \bigg( 1-\frac{x}{2^{j+1}}\bigg)^{\Neff},\quad 0< x< 1.\]
We then upper bound $f(x)$ as follows:
\begin{itemize}
    \item For $x\in (0, 4/\Neff)$ it holds that
    \[
f(x) \le \sum_{\ell=0}^{L-1} \frac{x}{2^{\ell+1}} \le x \le \frac{4}{\Neff}.
    \]

\item As for $x\in [4 / \Neff, 1]$, there is an 
\[\ell^* := \lfloor \log (\Neff x ) \rfloor - 2  \in [0,\ L-1 ), \] 
such that 
    \[ {2^{\ell^* + 2}}/{\Neff} \le x < {2^{\ell^*+3}}/{\Neff}.\]
    by which and the definition of $f(x)$ we obtain:
    \begin{align}
        f(x)
        &= \sum_{\ell=0}^{\ell^* } \frac{x}{2^{\ell+1}}  \cdot \prod_{j=\ell+1}^{L-1} \bigg(1-\frac{x}{2^{j+1}}\bigg)^{\Neff} + \sum_{\ell=\ell^* + 1}^{L-1} \frac{x}{2^{\ell+1}} \cdot \prod_{j=\ell+1}^{L-1} \bigg(1-\frac{x}{2^{j+1}}\bigg)^{\Neff} \notag\\
        &\le \sum_{\ell=0}^{\ell^* } \frac{x}{2^{\ell+1}}  \cdot \bigg(1-\frac{x}{2^{\ell+2}}\bigg)^{\Neff} + \sum_{\ell=\ell^*  +1}^{L-1} \frac{x}{2^{\ell+1}} \cdot 1 \notag\\
        &\le \sum_{\ell=0}^{\ell^*} \frac{2^{\ell^*-\ell+2}}{\Neff}  \cdot \bigg(1-\frac{2^{\ell^*-\ell}}{\Neff}\bigg)^{\Neff} + \sum_{\ell=\ell^* + 1}^{L-1} \frac{2^{\ell^*-\ell+2}}{\Neff}  \notag\\
        &\le \frac{4}{\Neff} \cdot \sum_{\ell=0}^{\ell^*} 2^{\ell^*-\ell}\cdot e^{-2^{\ell^*-\ell}} + \frac{4}{\Neff} \notag\\
        &\le \frac{4}{\Neff}\cdot 1  +\frac{4}{\Neff} 
        = \frac{8}{\Neff}.  \notag
    \end{align}
\end{itemize}
In sum we have shown $f(x) \le 8 / \Neff$ for $x \in (0,1)$.
Therefore
\[\sum_{t=1}^N \gamma_t \prod_{i=t+1}^N \bigg(\IB- \frac{\gamma_t}{2}\HB\bigg) = 2\Neff \HB^{-1} \cdot f(\gamma \HB) \preceq 2\Neff \HB^{-1} \cdot \frac{8}{\Neff} = 16 \HB^{-1}. \]
We have completed the proof.
\end{proof}

\subsection{Proof of Corollary \ref{coro:gtron:gaussianLmisspecified}}
\begin{proof}[Proof of Corollary \ref{coro:gtron:gaussianLmisspecified}]
According to the stepsize scheduler \eqref{eq:geometry-tail-decay-lr} and the assumptions, we have that 
\begin{align*}
    \Ebb \risk(\wB_N) &\lesssim \OPT + \|e^{-0.5 \Neff \gamma_0 \HB} \HB(\wB_0 - \wB_*)\|^2_{2} + \frac{k^* + \Neff^2\gamma_0^2\sum_{i>k^*}\lambda_i^2}{\Neff} \\
    &\lesssim \OPT + \frac{1}{\Neff \gamma_0} + \frac{k^* + \Neff^2\gamma_0^2\sum_{i>k^*}\lambda_i^2}{\Neff},
\end{align*}
where $k^*$ can be arbitrary.
We then simply choose $k^* = d$, and $\gamma_0 \eqsim 1/\tr(\HB)$, then 
\begin{align*}
    \Ebb \risk(\wB_N) \lesssim \OPT + \frac{\tr(\HB)}{\Neff } + \frac{d}{\Neff} \lesssim \OPT + \frac{d}{\Neff},
\end{align*}
where we use that $\lambda_1 \lesssim 1$.    
\end{proof}

\section{GLM-tron versus SGD}\label{append:sec:sgd}
In this section, we compare GLM-tron and SGD in learning well-specified ReLU regression with symmetric Bernoulli data. 
We assume that Assumption~\ref{assump:noise:well-specified} and Assumption~\ref{assump:bernoulli} hold in this part.

\paragraph{Notations.}
In this section, we assume that $\HB$ is diagonal.
For 
\( \AB_t:= \Ebb (\wB_t - \wB_*)^2,\)
we use
\( \mathring{\AB}_t\)
to refer to the diagonal of $\AB_t$.
For simplicity, we will use
\[\epsilon_t := y_t - \relu(\wB_*^\top \xB_t)\]
to refer to the additive noise in this section.

\subsection{Proof of Theorem \ref{thm:sgd:bernoulli:lb}}
\begin{proof}[Proof of Theorem \ref{thm:sgd:bernoulli:lb}]

Consider \eqref{eq:sgd}.
\begin{align*}
    \wB_{t}  
    &= \wB_{t-1} - \gamma_t \big(\relu(\xB_t^\top\wB_{t-1})-\relu(\xB^\top_t \wB_*)-\epsilon_t\big) \cdot \xB_t \ind{\xB_t^\top \wB_{t-1} > 0} \\
     &= \wB_{t-1} - \gamma_t \big(\xB_t^\top\wB_{t-1}\ind{\xB_t^\top\wB_{t-1}>0}-\xB_t^\top \wB_*\ind{\xB_t^\top \wB_*>0}-\epsilon_t\big) \cdot \xB_t \ind{\xB_t^\top \wB_{t-1} > 0} \\
    &= \wB_{t-1} - \gamma_t \xB_t \xB_t^\top \ind{\xB_t^\top \wB_{t-1} > 0}\wB_{t-1} + \gamma_t \xB_t \xB_t^\top \ind{\xB_t^\top \wB_{t-1} > 0, \xB_t^\top \wB_*>0}\wB_* \\
    &\quad + \gamma_t \ind{\xB_t^\top \wB_{t-1} > 0}\epsilon_t \xB_t \\
    &= \wB_{t-1} - \gamma_t \xB_t \xB_t^\top \ind{\xB_t^\top \wB_{t-1} > 0}(\wB_{t-1}-\wB_*) - \gamma_t \xB_t \xB_t^\top \ind{\xB_t^\top \wB_{t-1} > 0, \xB_t^\top \wB_*<0}\wB_* \\
    &\quad+ \gamma_t \ind{\xB_t^\top \wB_{t-1} > 0}\epsilon_t \xB_t,
\end{align*}
which implies that 
\begin{align*}
    \wB_{t} - \wB_* 
    &= \big(\IB - \gamma_t \xB_t \xB_t^\top \ind{\xB_t^\top \wB_{t-1} > 0} \big)(\wB_{t-1} - \wB_*)\\
    &\quad  - \gamma_t \xB_t \xB_t^\top \ind{\xB_t^\top \wB_{t-1} > 0, \xB_t^\top \wB_* < 0} \wB_* 
    + \gamma_t \ind{\xB_t^\top \wB_{t-1} > 0}\epsilon_t \xB_t.
\end{align*}
Let us compute the expected outer product:
\begin{equation}\label{eq:bernoulli:sgd:out-product}
\begin{aligned}
    &\ \Ebb \big( \wB_{t} - \wB_* \big)^{\otimes 2}  \\
    &=  \Ebb \underbrace{\Big( \big(\IB - \gamma_t \xB_t \xB_t^\top \ind{\xB_t^\top \wB_{t-1} > 0} \big)(\wB_{t-1} - \wB_*) \Big)^{\otimes 2}}_{\texttt{quadratic term 1}}  \\
    &\quad + \gamma_t^2\cdot \Ebb \underbrace{ \big( \ind{\xB_t^\top \wB_{t-1} > 0, \xB_t^\top \wB_* < 0} \cdot (\xB_t^\top \wB_*)^2 \cdot \xB_t^{\otimes 2} \big)}_{\texttt{quadratic term 2}} \\
    &\quad - \gamma_t \cdot \Ebb \underbrace{ \bigg(\ind{\xB_t^\top \wB_{t-1} > 0, \xB_t^\top \wB_* < 0} \xB_t^\top \wB_*\cdot \big(\IB - \gamma_t \xB_t \xB_t^\top \ind{\xB_t^\top \wB_{t-1} > 0} \big)( \wB_{t-1}-\wB_*)\xB_t^\top \bigg) }_{\texttt{crossing term 1}} \\
    &\quad - \gamma_t \cdot \Ebb \underbrace{ \bigg(\ind{\xB_t^\top \wB_{t-1} > 0, \xB_t^\top \wB_* < 0} \xB_t^\top \wB_*\cdot \xB_t( \wB_{t-1}-\wB_*)^\top \big(\IB - \gamma_t \xB_t \xB_t^\top \ind{\xB_t^\top \wB_{t-1} > 0} \big)\bigg) }_{\texttt{crossing term 2}} \\
    &\quad + \gamma^2_t \cdot \Ebb \big(\ind{\xB_t^\top \wB_{t-1} > 0}\epsilon_t^2\cdot  \xB_t^{\otimes 2}\big),
\end{aligned}
\end{equation}
where the crossing terms involving $\epsilon$ has zero expectation because $\Ebb[\epsilon_t | \xB_t] = 0$.

Now we use Assumption~\ref{assump:bernoulli} and compute each part in \eqref{eq:bernoulli:sgd:out-product}.
Notice that under Assumption~\ref{assump:bernoulli}, $\xB_t \in \{\pm \eB_i\}_{i\ge 1}$, then one can verify that 
\begin{equation}\label{eq:bernoulli:key-identity}
    \text{for every $\uB \in \Hbb$},\ \diag(\uB \xB_t^\top) = \diag( \xB_t \uB^\top) = \xB_t^\top \uB \cdot \xB_t\xB_t^\top.
\end{equation}

By \eqref{eq:bernoulli:key-identity} we see that 
\begin{align*}
    &\ \diag\bigg( \gamma_t^2\cdot \Ebb (\texttt{quadratic term 2}) - \gamma_t\cdot \Ebb (\texttt{crossing term 1})  - \gamma_t\cdot \Ebb (\texttt{crossing term 2})  \bigg)  \\
    &= \gamma_t^2\cdot \Ebb  \big( \ind{\xB_t^\top \wB_{t-1} > 0, \xB_t^\top \wB_* < 0} \cdot (\xB_t^\top \wB_*)^2 \cdot \xB_t\xB_t^\top \big) \\
    &\quad - 2\gamma_t \cdot \Ebb \bigg(\ind{\xB_t^\top \wB_{t-1} > 0, \xB_t^\top \wB_* < 0} \xB_t^\top \wB_*\cdot \diag\Big( ( \wB_{t-1}-\wB_*) \xB_t^\top \Big)\bigg) \\
    &\quad + 2\gamma_t^2 \cdot \Ebb \bigg(\ind{\xB_t^\top \wB_{t-1} > 0, \xB_t^\top \wB_* < 0} \xB_t^\top \wB_*\cdot \xB_t^\top  ( \wB_{t-1}-\wB_*) \cdot \xB_t \xB_t^\top \bigg) \\
    &= \gamma_t^2\cdot \Ebb  \Big( \ind{\xB_t^\top \wB_{t-1} > 0, \xB_t^\top \wB_* < 0} \cdot (\xB_t^\top \wB_*)^2 \cdot \xB_t\xB_t^\top \Big) \\
    &\quad + (2\gamma_t-2\gamma_t^2) \cdot \Ebb \Big(\ind{\xB_t^\top \wB_{t-1} > 0, \xB_t^\top \wB_* < 0} \xB_t^\top \wB_*\cdot \xB_t^\top  (\wB_*- \wB_{t-1}) \cdot \xB_t\xB_t^\top\Big),
\end{align*}
where in the last equality we use \eqref{eq:bernoulli:key-identity}.
Define 
\begin{equation*}
    F(\wB) :=  \Ebb_{\xB} \Big(\ind{\xB^\top \wB > 0, \xB^\top \wB_* < 0} \xB^\top \wB_*\cdot \xB^\top  (\wB_*-\wB) \cdot \xB\xB^\top \Big),
\end{equation*}
where the expectation is only taken with respect to the randomness of $\xB$.
Then by the property of the indicator function,  $\ind{\xB_t^\top \wB_{t-1} > 0, \xB_t^\top \wB_* < 0} $, we observe that 
\[
0\preceq \Ebb_{\xB_t}  \Big( \ind{\xB_t^\top \wB_{t-1} > 0, \xB_t^\top \wB_* < 0} \cdot (\xB_t^\top \wB_*)^2 \cdot \xB_t\xB_t^\top \Big)\preceq F(\wB_{t-1}). 
\]
So when $0<\gamma_t<1$ it holds that 
\begin{equation}\label{eq:bernoulli:sgd:quad2+cross}
\begin{aligned}
    &\quad \diag\bigg( \gamma_t^2\cdot \Ebb (\texttt{quadratic term}) -\gamma_t\cdot \Ebb (\texttt{crossing term 1})
    -\gamma_t\cdot \Ebb (\texttt{crossing term 2}) \bigg)\\
    &= \gamma_t^2\cdot \Ebb  \Big( \ind{\xB_t^\top \wB_{t-1} > 0, \xB_t^\top \wB_* < 0} \cdot (\xB_t^\top \wB_*)^2 \cdot \xB_t\xB_t^\top \Big)  + (2\gamma_t-2\gamma_t^2) \cdot \Ebb [ F(\wB_{t-1}) ] \\
    &\begin{cases}
        \le (2\gamma_t - \gamma_t^2) \cdot \Ebb [F(\wB_{t-1})] \le 2\gamma_t \cdot \Ebb [F(\wB_{t-1})] ; \\
        \ge (2\gamma_t - 2\gamma_t^2) \cdot \Ebb [F(\wB_{t-1})] = 2\gamma_t(1-\gamma_t)\cdot \Ebb [F(\wB_{t-1})].
    \end{cases}
\end{aligned}
\end{equation}

Similarly, we calculate the diagonal of the expectation of $(\texttt{quadratic term 1})$ in \eqref{eq:bernoulli:sgd:out-product}:
\begin{align}
    &\ \diag\big( \Ebb (\texttt{quadratic term 1}) \big) \notag \\
    &= \diag\big( \Ebb (\wB_{t-1}-\wB_*)^{\otimes 2} \big)
    -2\gamma_t\cdot \diag\bigg( \Ebb \Big( \ind{\xB_t^\top \wB_{t-1}>0} \cdot \xB_t^\top (\wB_{t-1}-\wB_*) \cdot \xB_t (\wB_{t-1}-\wB_*)^\top \Big) \bigg) \notag \\
    &\quad + \gamma_t^2 \cdot  \Ebb \Big(\ind{\xB_t^\top \wB_{t-1}>0} \cdot \big(\xB_t^\top (\wB_{t-1}-\wB_*) \big)^2 \cdot \xB_t\xB_t^\top \Big)  \notag \\
    &= \mathring{\AB}_{t-1}
     + (\gamma_t^2-2\gamma_t) \cdot \Ebb \Big(\ind{\xB_t^\top \wB_{t-1}>0} \cdot \big(\xB_t^\top (\wB_{t-1}-\wB_*) \big)^2 \cdot \xB_t\xB_t^\top \Big) \notag \\
     &\begin{cases}
     \preceq \mathring\AB_{t-1}
     -\gamma_t \cdot \Ebb \Big(\ind{\xB_t^\top \wB_{t-1}>0} \cdot \big(\xB_t^\top (\wB_{t-1}-\wB_*) \big)^2 \cdot \xB_t\xB_t^\top \Big); \\
        \succeq \mathring\AB_{t-1}
     -2\gamma_t \cdot \Ebb \Big(\ind{\xB_t^\top \wB_{t-1}>0} \cdot \big(\xB_t^\top (\wB_{t-1}-\wB_*) \big)^2 \cdot \xB_t\xB_t^\top \Big), 
     \end{cases}\notag
\end{align}
where in the second equality we use \eqref{eq:bernoulli:key-identity} and in the inequality we use $0<\gamma_t<1$.
We now use Assumption~\ref{assump:bernoulli} to obtain that 
\begin{align*}
    \Ebb \Big(\ind{\xB_t^\top \wB_{t-1}>0} \cdot \big(\xB_t^\top (\wB_{t-1}-\wB_*) \big)^2 \cdot \xB_t\xB_t^\top \Big)
    = \sum_i \frac{\lambda_i}{2} \cdot \Ebb \big(\wB_{t-1}[i]-\wB_*[i]\big)^2\cdot \eB_i 
    = \half \cdot \HB \mathring\AB_{t-1}.
\end{align*}
So we have 
\begin{equation}\label{eq:bernoulli:sgd:quad1}
    \diag\big( \Ebb (\texttt{quadratic term 1}) \big)
    \begin{cases}
    \preceq (\IB -\frac{\gamma_t}{2}\cdot \HB )\cdot \mathring\AB_{t-1} ; \\
        \succeq (\IB-\gamma_t\cdot \HB)\cdot \mathring\AB_{t-1}.
    \end{cases}
\end{equation}

Bring \eqref{eq:bernoulli:sgd:quad2+cross}, \eqref{eq:bernoulli:sgd:quad1} and that 
\[\Ebb \big(\ind{\xB_t^\top \wB_{t-1} > 0}\epsilon_t^2\cdot  \xB_t\xB_t^\top \big) = \frac{\sigma^2}{2}\cdot \HB\]
into \eqref{eq:bernoulli:sgd:out-product} we obtain
\begin{equation*}
\begin{aligned}
\mathring\AB_t
&= \diag\big( 
     \Ebb ( \wB_{t} - \wB_* )^{\odot 2}   \big) \\
&=      \diag\big( \Ebb (\texttt{quadratic term 1}) \big)
+ \Ebb \big(\ind{\xB_t^\top \wB_{t-1} > 0}\epsilon_t^2\cdot  \xB_t\xB_t^\top \big)  \\ 
&\quad  + \diag\bigg( \gamma_t^2\cdot \Ebb (\texttt{quadratic term 2}) - \gamma_t\cdot \Ebb (\texttt{crossing term 1})  - \gamma_t\cdot \Ebb (\texttt{crossing term 2})  \bigg) \\
& 
     \begin{cases}
     \preceq  (\IB-\frac{\gamma_t}{2}\cdot \HB)\cdot \mathring\AB_{t-1} + \frac{\gamma_t^2\sigma^2}{2}\HB + 2\gamma_t \cdot \Ebb[F(\wB_{t-1})], \\
    \succeq  (\IB-\gamma_t\cdot \HB)\cdot \mathring\AB_{t-1} + \frac{\gamma_t^2\sigma^2}{2}\HB+ 2\gamma_t(1-\gamma_t) \cdot \Ebb[F(\wB_{t-1})].
     \end{cases}
\end{aligned}
\end{equation*}
By solving the above recursion we obtain 
\begin{align*}
    \mathring\AB_N
        &\preceq \prod_{t=1}^N (\IB-\frac{\gamma_t}{2}\cdot \HB)\cdot\mathring\AB_0 + \frac{\sigma^2}{2}\sum_{t=1}^{N} \gamma^2_t\prod_{k=t+1}^{N}(1-\frac{\gamma_k}{2} \HB)\HB + 2\sum_{t=1}^{N} \gamma_t\prod_{k=t+1}^{N}(1-\frac{\gamma_k}{2} \HB)\Ebb[F(\wB_{t})], \\
        \mathring\AB_N
        &\succeq \prod_{t=1}^N (\IB-{\gamma_t}\cdot \HB)\cdot\mathring\AB_0 + \frac{\sigma^2}{2}\sum_{t=1}^{N} \gamma^2_t\prod_{k=t+1}^{N}(1-{\gamma_k} \HB)\HB+ 2\sum_{t=1}^{N} \gamma_t (1-\gamma_t) \prod_{k=t+1}^{N}(1-{\gamma_k} \HB)\Ebb[F(\wB_{t})].
\end{align*}
The remain efforts are taking inner product with $\HB$ and using Lemma~\ref{lemma:tron:misspecified:prod-sum-bounds} to show that
\begin{align*}
    \quad \la\HB, \mathring\AB_N\ra 
        &\lesssim \Big\la \HB, \prod_{t=1}^N \Big(\IB-\frac{\gamma_t}{2}\cdot \HB\Big)\cdot\mathring\AB_0\Big \ra  + {\sigma^2}\Big\la \HB, \sum_{t=1}^{N} \gamma^2_t\prod_{k=t+1}^{N}\Big(1-\frac{\gamma_k}{2} \HB\Big)\HB\Big\ra  \\
        &\quad+ \Big\la \HB, \sum_{t=1}^{N} \gamma_t\prod_{k=t+1}^{N}\Big( 1-\frac{\gamma_k}{2} \HB\Big)\Ebb[F(\wB_{t})] \Big\ra\\
        &\lesssim \big\| \wB_0 - \wB_* \big \|_{\prod_{t=1}^N (\IB-\frac{\gamma_t}{2} \HB)\HB}^2  + \sigma^2\cdot \frac{k^*+\Neff^2\gamma_0^2\sum_{i>k^*}\lambda_i^2}{\Neff} \\
        &\quad + \sum_{t=1}^{N} \Big\la \gamma_t\prod_{k=t+1}^{N}\Big( 1-\frac{\gamma_k}{2} \HB\Big)\HB, \Ebb[F(\wB_{t})] \Big\ra,
\end{align*}
and
\begin{align*}
   \la\HB, \mathring\AB_N\ra 
        &\gtrsim \Big\la\HB,  \prod_{t=1}^N (\IB-{\gamma_t}\HB)\cdot\mathring\AB_0\Big\ra  + \sigma^2\sum_{t=1}^{N} \gamma^2_t\prod_{k=t+1}^{N}(1-{\gamma_k}\HB)\HB \\
        &\quad + \sum_{t=1}^{N} \gamma_t(1-\gamma_t)\prod_{k=t+1}^{N}(1-{\gamma_k} \HB)\Ebb[F(\wB_{t})], \\
        &\gtrsim  \big\| \wB_0 - \wB_* \big \|_{\prod_{t=1}^N (\IB-\gamma_t \HB)\HB}^2  + \sigma^2\cdot \frac{k^*+\Neff^2\gamma_0^2\sum_{i>k^*}\lambda_i^2}{\Neff} \\
        &\quad + \sum_{t=1}^{N} \Big\la \gamma_t(1-\gamma_t)\prod_{k=t+1}^{N}\big( 1-\gamma_k\HB\big)\HB, \Ebb[F(\wB_{t})] \Big\ra.
\end{align*}
We have completed the proof.
\end{proof}

\subsection{Proof of Theorem \ref{thm:tron-vs-sgd:every-case}}
\begin{proof}[Proof of Theorem \ref{thm:tron-vs-sgd:every-case}]
    We now compare the risk upper bound for \eqref{eq:tron} shown in Theorem \ref{thm:tron:bernoulli} and the risk lower bound for \eqref{eq:sgd} shown in Theorem \ref{thm:sgd:bernoulli:lb}.
Denote $\gamma_0^{\sgd}<1$ as the initial stepsize for \eqref{eq:sgd}, and 
\begin{align*}
   & \bias^{\sgd}(\gamma_0) := \big\| (\wB_0 - \wB_*) \big\|^2_{\prod_{t=1}^{N}(\IB-\gamma_t\HB)\HB},\\
   & \variance^{\sgd}(\gamma_0) := 
     \sigma^2  \cdot  \frac{\#\{i: \lambda_i \ge \frac{1}{\Neff \gamma_0 }\}+\Neff^2\gamma_0^2\sum_{\lambda_i < \frac{1}{\Neff \gamma_0}} \lambda_i^2}{\Neff},
\end{align*}
then Theorem \ref{thm:sgd:bernoulli:lb} implies that for every $\gamma_0^{\sgd}<1$,
\begin{align*}
    \Ebb \excessrisk(\wB_N^{\sgd}) \gtrsim \bias(\gamma_0^{\sgd}) + \variance(\gamma_0^{\sgd}) + \Psi,
\end{align*}
where $\Psi \ge 0$.
Similarly, 
denote $\gamma_0^{\tron}<1/2$ as the initial stepsize for \eqref{eq:tron}, and 
\begin{align*}
    \bias^{\tron}(\gamma_0) := \big\| (\wB_0 - \wB_*) \big\|^2_{\prod_{t=1}^{N}(\IB-\frac{\gamma_t}{2}\HB)\HB}, \qquad 
    \variance^{\tron}(\gamma_0, k) := 
     \sigma^2  \cdot  \frac{k + \Neff^2\gamma_0^2\sum_{i>k} \lambda_i^2}{\Neff},
\end{align*}
then Theorem \ref{thm:tron:bernoulli} implies that for every $\gamma_0^{\tron}<1/2$,
\begin{align*}
    \Ebb \excessrisk(\wB_N^{\tron}) \lesssim \bias(\gamma_0^{\tron}) + \variance(\gamma_0^{\tron}, k),
\end{align*}
where $k \ge 0$ can be an arbitrary index. 

    We prove the theorem by discussing two cases on whether or not the \eqref{eq:sgd} initial stepsize is large or not.

    \paragraph{SGD with Small Initial Stepsize.}
    If the initial stepsize for \eqref{eq:sgd} is $\gamma_0^{\sgd} < 1/8$, then one can choose an initial stepsize $\gamma_0^{\tron} = 2\gamma_0^{\sgd} < 1/4$ for \eqref{eq:tron}. Then we have 
    \( \bias^{\tron}(\gamma_0^{\tron}) = \bias^{\sgd}(\gamma_0^{\sgd}).\)
    Moreover by choosing $k := \#\{i: \lambda_i \ge {1}/{(\Neff \gamma_0^{\sgd}})\}$, we have 
    \( \variance^{\tron}(\gamma_0^{\tron}, k) = \variance^{\sgd}(\gamma_0^{\sgd}).\)
    These together imply that 
    \[\Ebb \excessrisk(\wB_N^{\tron}) \lesssim \Ebb \excessrisk(\wB_N^{\tron}) \]
    holds if $\gamma_0^{\sgd} < 1/8$.

    \paragraph{SGD with Large Initial Stepsize.}
    Now we discuss the case when $\gamma_0^{\sgd} > 1/8$. In this case we choose $\gamma_0^{\tron} = 1/8 < 1/4$. 
    \begin{itemize}
        \item 
    If $\Neff < \frac{1}{8\lambda_1}$, i.e., $\gamma_0^{\sgd}\lambda_i \le \gamma_0^{\sgd}\lambda_1 \le \frac{1}{8\Neff}$, which implies that 
    \((\IB-{\gamma_0^{\sgd}}\HB)^{2\Neff} \succeq (1-\frac{1}{8\Neff}\IB )^{2\Neff} \succeq 0.01\cdot \IB \),
    then
    \begin{align*}
        &\bias^{\sgd}(\gamma_0^{\sgd}) = \big\| (\wB_0 - \wB_*) \big\|^2_{\prod_{t=1}^{N}(\IB-{\gamma_t}\HB)\HB}
        \ge \big\| (\wB_0 - \wB_*) \big\|^2_{(\IB-{\gamma_0^{\sgd}}\HB)^{2\Neff}\HB}\\
        &  \ge 0.01\cdot \big\| (\wB_0 - \wB_*) \big\|^2_{\HB}.
    \end{align*}
    So we have 
    \(\bias^{\tron}(\gamma_0^{\tron}) \lesssim \bias^{\sgd}(\gamma_0^{\sgd}).\)
    Similarly we choose $k := \#\{i: \lambda_i \ge {1}/{(\Neff \gamma_0^{\sgd}})\}$, we have 
    \( \variance^{\tron}(\gamma_0^{\tron}, k) \lesssim \variance^{\sgd}(\gamma_0^{\sgd}),\)
    because $\gamma_0^{\tron} = 1/8 \le \gamma_0^{\sgd}$.

\item If $\Neff > \frac{1}{8\lambda_1}$, then it holds that $\lambda_1 > 1/({8\Neff}) = 1/(\Neff \gamma_0^{\tron})$ then we must have 
\begin{align*}
    \variance^{\tron}(\gamma_0^{\tron}, k) =\sigma^2  \cdot  \frac{k + \Neff^2(\gamma_0^{\tron})^2\sum_{i>k} \lambda_i^2}{\Neff} 
    \ge \frac{\sigma^2}{\Neff},
\end{align*}
for every $k\ge 0$.
But we also have $\big(\IB-\frac{\gamma_0^{\tron}}{2}\HB\big)^{\Neff} \le \frac{2}{\gamma_0^{\tron}\Neff}\HB^{-1} = \frac{16}{\Neff} \HB^{-1}$, which implies that
    \begin{align*}
        \bias^{\tron}(\gamma_0^{\tron}) &= \big\| (\wB_0 - \wB_*) \big\|^2_{\prod_{t=1}^{N}(\IB-\frac{\gamma_t}{2}\HB)\HB}\le \big\| (\wB_0 - \wB_*) \big\|^2_{\big(\IB-\frac{\gamma_0^{\tron}}{2}\HB\big)^{\Neff}\HB} \\
        &\le \frac{16}{\Neff}\cdot \|\wB_0 - \wB_*\|_2^2  \lesssim \variance^{\tron}(\gamma_0^{\tron}, k), 
    \end{align*}
    for every $k\ge 0$.
    Therefore we have 
    \begin{align*}
        \Ebb \excessrisk(\wB_N^{\tron}) \lesssim \bias(\gamma_0^{\tron}) + \variance(\gamma_0^{\tron}, k) 
        \lesssim \variance(\gamma_0^{\tron}, k) 
        \lesssim \variance(\gamma_0^{\sgd}) 
        \lesssim \Ebb \excessrisk(\wB_N^{\sgd}), 
    \end{align*}
    where the third inequality is by choosing $k := \#\{i: \lambda_i \ge {1}/{(\Neff \gamma_0^{\sgd}})\}$ and the fact that $\gamma_0^{\tron} = 1/8 \le \gamma_0^{\sgd}$.
    
    \end{itemize}

    Putting everything together, we have completed the proof.
\end{proof}

\subsection{Proof of Theorem \ref{thm:tron-vs-sgd:bad-case}}
\begin{proof}[Proof of Theorem \ref{thm:tron-vs-sgd:bad-case}]

We only need to show that for SGD \eqref{eq:sgd} it holds that 
\[\Ebb_{\wB_*}\Ebb_{\algo} \risk(\wB_N) \ge \half \cdot \risk(\wB_0).\]

According to the SGD iterate \eqref{eq:sgd} and the noiseless assumption ($\sigma^2=0$), we can write the gradient as 
\begin{align*}
    \gB_t &:= \big(\relu(\xB_t^\top \wB_{t-1})  - y_t\big)\cdot \ind{\xB_{t}^\top \wB_{t-1}>0}\cdot \xB_t \\
    &= \big( \xB_t^\top \wB_{t-1} \ind{\xB_{t}^\top \wB_{t-1} >0} - \xB_t^\top \wB_* \ind{\xB_{t}^\top \wB_{*}>0}\big) \cdot \ind{\xB_{t}^\top \wB_{t-1} >0}\cdot \xB_t.
\end{align*}
Let us focus on the $i$-th component from now on.
Without loss of generality, we assume for now that $\wB_*[i] \ge 0$.
We use Assumption~\ref{assump:bernoulli} to obtain
\begin{align*}
    \gB_t [i] =
    \begin{cases}
    0, & \xB_t \notin \{\pm \eB_i \} \ \text{or}\ \xB^\top_t \wB_{t-1} \le 0; \\
    \wB_{t-1}[i] - \wB_*[i], & \wB_{t-1}[i] \ge 0 \ \text{and}\ \xB_t = \eB_i; \\
    \wB[i], & \wB_{t-1}[i] < 0  \ \text{and}\ \xB_t = - \eB_i.
    \end{cases}
\end{align*}
So the $i$-th component of the SGD iterate is updated by
\begin{align}
    \wB_{t}[i] 
    &= \begin{cases}
    \wB_{t-1}[i], &  \xB_t \notin \{\pm \eB_i \} \ \text{or}\ \xB^\top_t \wB_{t-1} \le 0; \\
    \wB_{t-1}[i] - \gamma (\wB_t[i] - \wB_*[i]),& \wB_{t-1}[i] \ge 0 \ \text{and}\ \xB_t = \eB_i; \\
    \wB_{t-1}[i] - \gamma \wB_{t-1}[i], & \wB_{t-1}[i] < 0  \ \text{and}\ \xB_t = - \eB_i.
    \end{cases} \notag 
\end{align}
Recall that $\gamma < 1/\tr(\HB) = 1$.
We next show that: if $\wB_0[i] \le 0$, then $\wB_t[i] \le 0$ for all $t$.
This is done by induction: when $\wB_{t-1}[i]\le 0$, two possible updates happen: $\wB_t[i] = \wB_{t-1}[i] \le 0$ or $\wB_t[i] = (1-\gamma) \wB_{t-1}[i] \le 0$.
In both cases, it holds that $\wB_t[i]\le 0$. We have completed the induction.
Moreover, recall that we assume $\wB_*[i]\ge 0$,  so if $\wB_t[i]\le 0$, it holds that
\[
(\wB_t[i] - \wB_*[i])^2 \ge (\wB_*[i])^2,\quad t\ge 0.
\]
Similarly we can prove that when $\wB_*[i]\le 0$,
if $\wB_t[i]\ge 0$, it holds that
\[
(\wB_t[i] - \wB_*[i])^2 \ge (\wB_*[i])^2,\quad t\ge 0.
\]

Now recall that $\wB_*[i]$ is initialized with a uniformly random sign, so with half probability  $\wB_*[i]$ and $\wB_[i]$ will have different signs. 
Therefore we have 
\[
\Ebb_{\wB_*}\Ebb_{\algo}\big[ (\wB_t[i] - \wB_*[i])^2 \big]\ge
\Ebb_{\wB_*}\Ebb_{\algo} \big[ (\wB_t[i] - \wB_*[i])^2\cdot \ind{\wB_0[i] \cdot \wB_*[i] \le 0} \big] \ge 
 \half \cdot (\wB_*[i])^2, \quad t\ge 0.
\]
Therefore we have shown that for every $t\ge 0$, 
\[
\Ebb_{\wB_*}\Ebb_{\algo} \risk(\wB_t)  = \Ebb_{\wB_*}\sum_i \lambda_i (\wB_t[i] - \wB_*[i])^2 
\ge \half \sum_i \lambda_i (\wB_*[i])^2
= \half \cdot \| \wB_*\|_{\HB}^2 \ge \half \cdot \risk(0),
\]
where the last inequality is due to Lemma~\ref{thm:loss-landscape}.

\end{proof}

\section{Additional Experiments}
Figures \ref{fig:bernoulli3}, \ref{fig:gaussian1} and \ref{fig:gaussian2} show the additional experimental results, where we compare the excess risk achieved by \eqref{eq:tron} and \eqref{eq:sgd} on Bernoulli and Gaussian data. Figure \ref{fig:bernoulli3} provides the experimental results on Bernoulli data in the noiseless setting. We can clearly see that SGD finally reaches a point with constant risk, while GLM-tron achieves nearly zero excess risk. This backs up our Theorem \ref{thm:tron-vs-sgd:bad-case}. Figures \ref{fig:gaussian1} and \ref{fig:gaussian2} visualize the learning performance of GLM-tron and SGD on Gaussian data. We can also see that GLM-tron achieves smaller excess risk than SGD, which also supports our claim that GLM-tron is preferable to SGD for high-dimensional ReLU regression.

\begin{figure}[t!]
    \centering
    \subfigure{\includegraphics[width=0.35\linewidth]{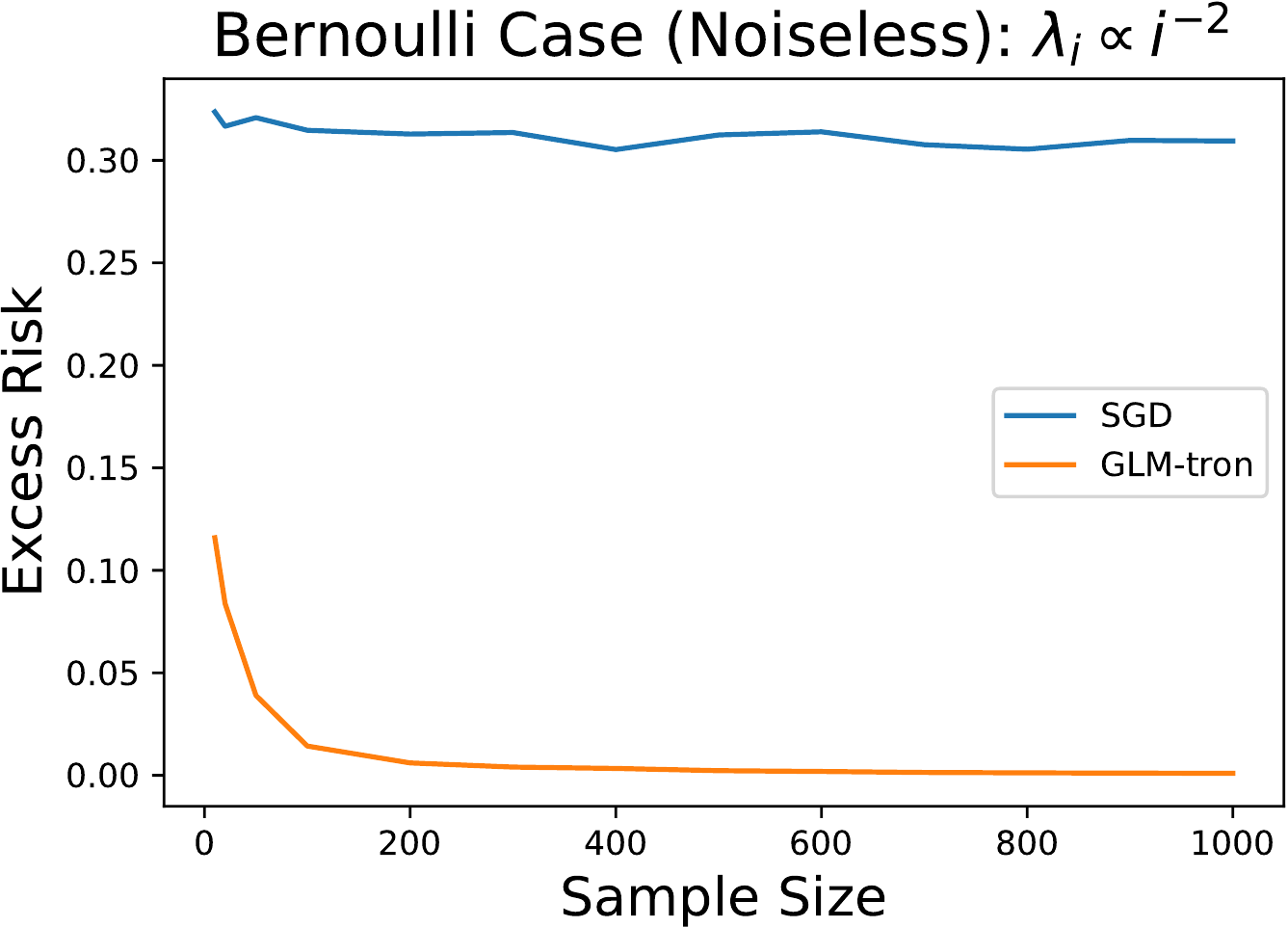}}
      \subfigure{\includegraphics[width=0.35\linewidth]{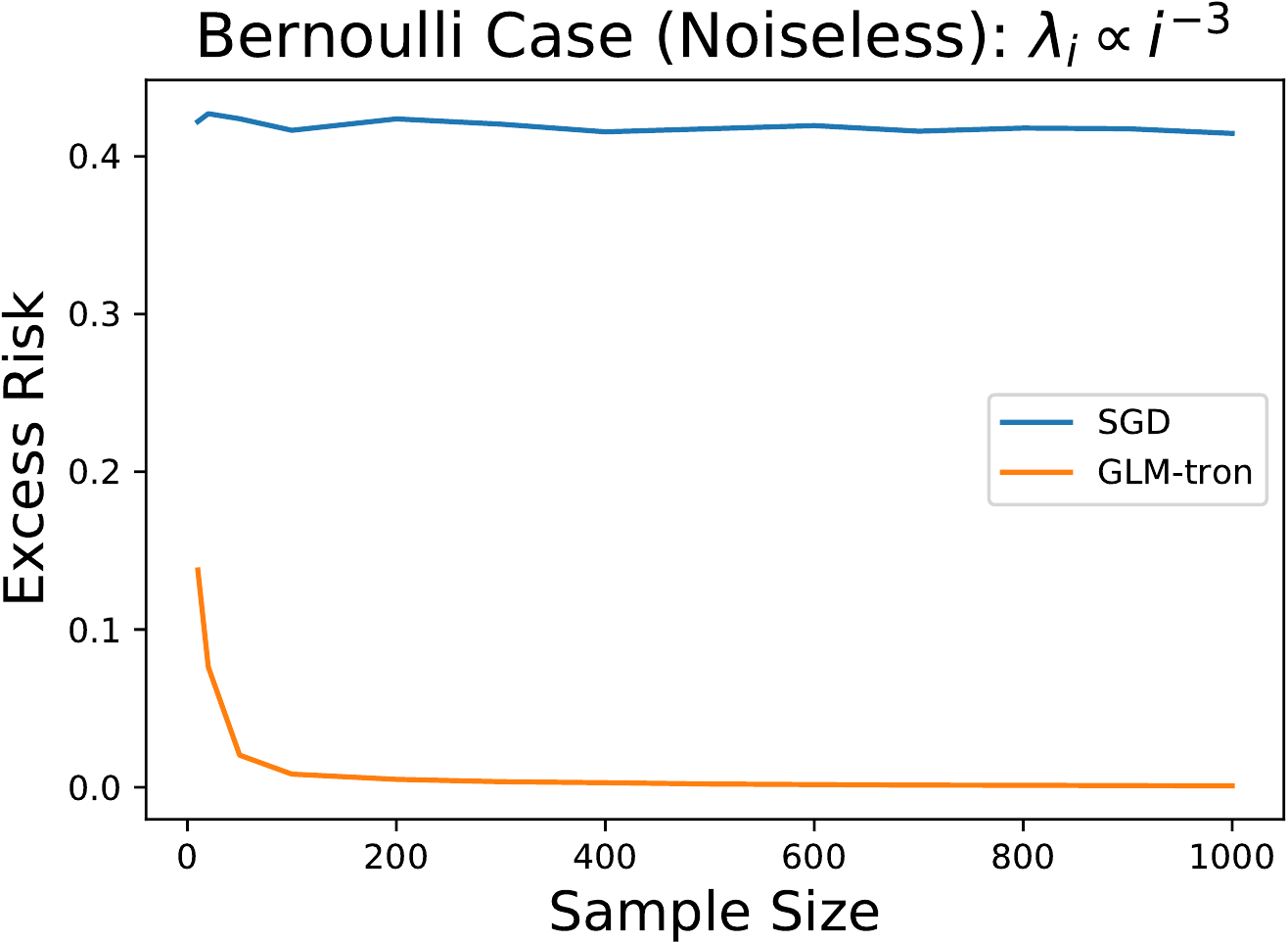}}
    \caption{\small Excess risk comparison between SGD and GLM-tron on Bernoulli Distribution.  The problem dimension $d = 1024$, and the regression model is well-specified without noise. We consider two different symmetric Bernoulli distributions and set true model parameters $\wB[i]^{*} = i^{-1}$. For each algorithm and each sample size, we do a grid search and report the best excess risk achieved by $\gamma_0 \in \{0.5, 0.25, 0.1, 0.075, 0.05, 0.025, 0.01\}$. The plots are averaged over $20$ independent runs.
}
    \label{fig:bernoulli3}
\end{figure}

\begin{figure}[t!]
    \centering
    \subfigure{\includegraphics[width=0.35\linewidth]{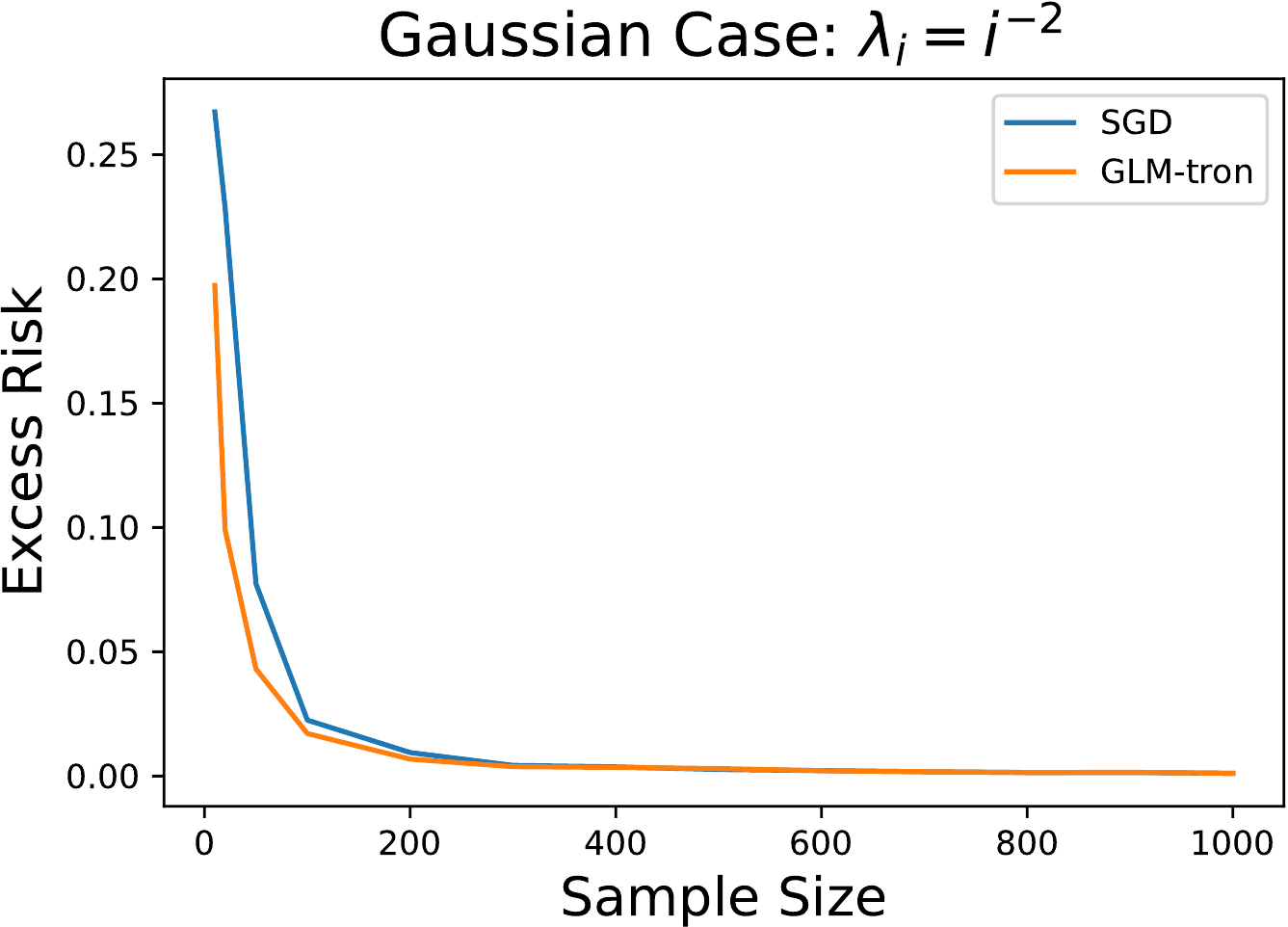}}
      \subfigure{\includegraphics[width=0.35\linewidth]{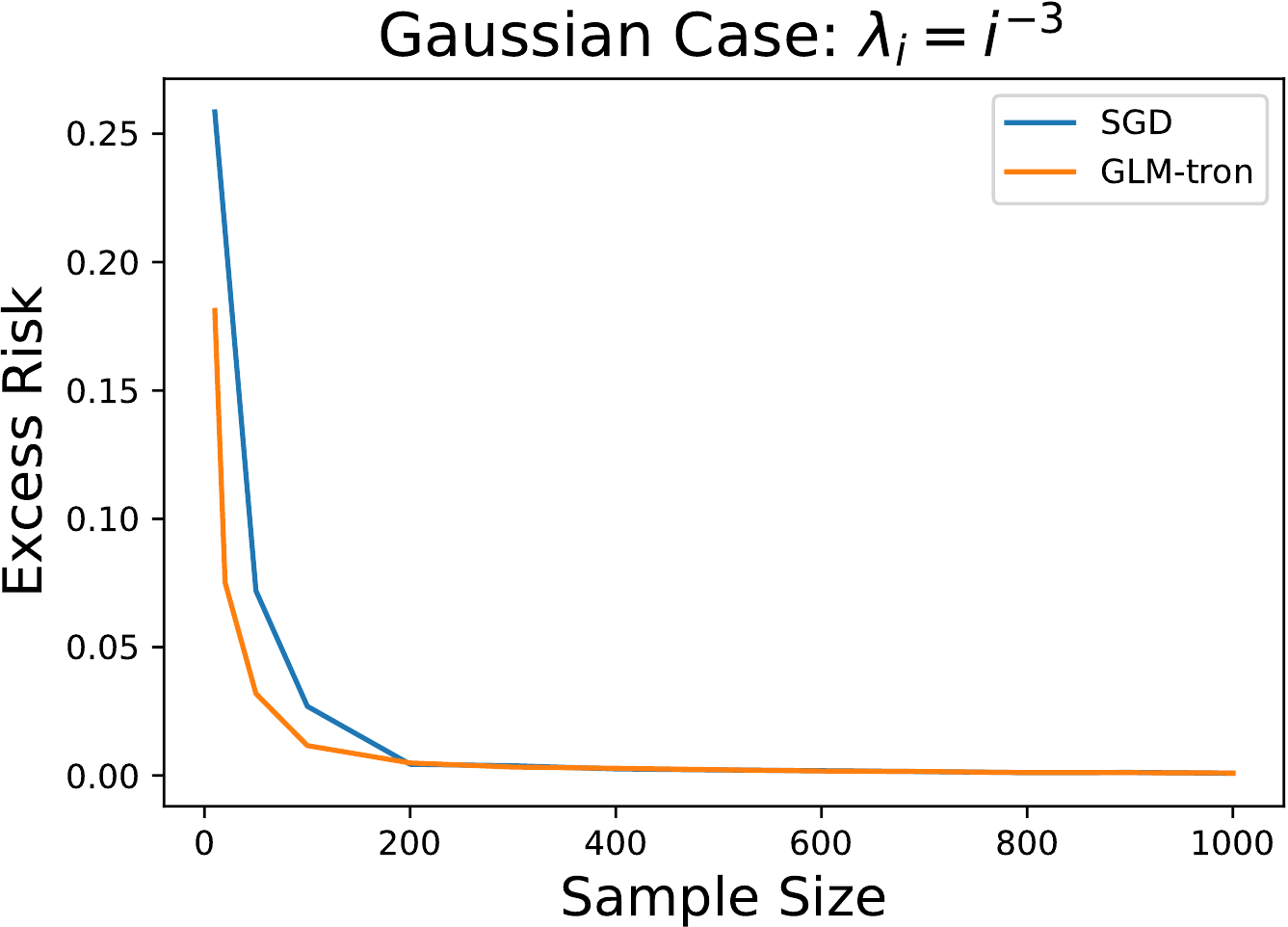}}
    \caption{\small Excess risk comparison between SGD and GLM-tron on Gaussian Distribution.  The regression model is well-specified with noise variance $\sigma = 0.1$. Other problem parameters and algorithm designs are the same as those in Figure \ref{fig:bernoulli3}. 
}
    \label{fig:gaussian1}
\end{figure}

\begin{figure}[t!]
    \centering
    \subfigure{\includegraphics[width=0.35\linewidth]{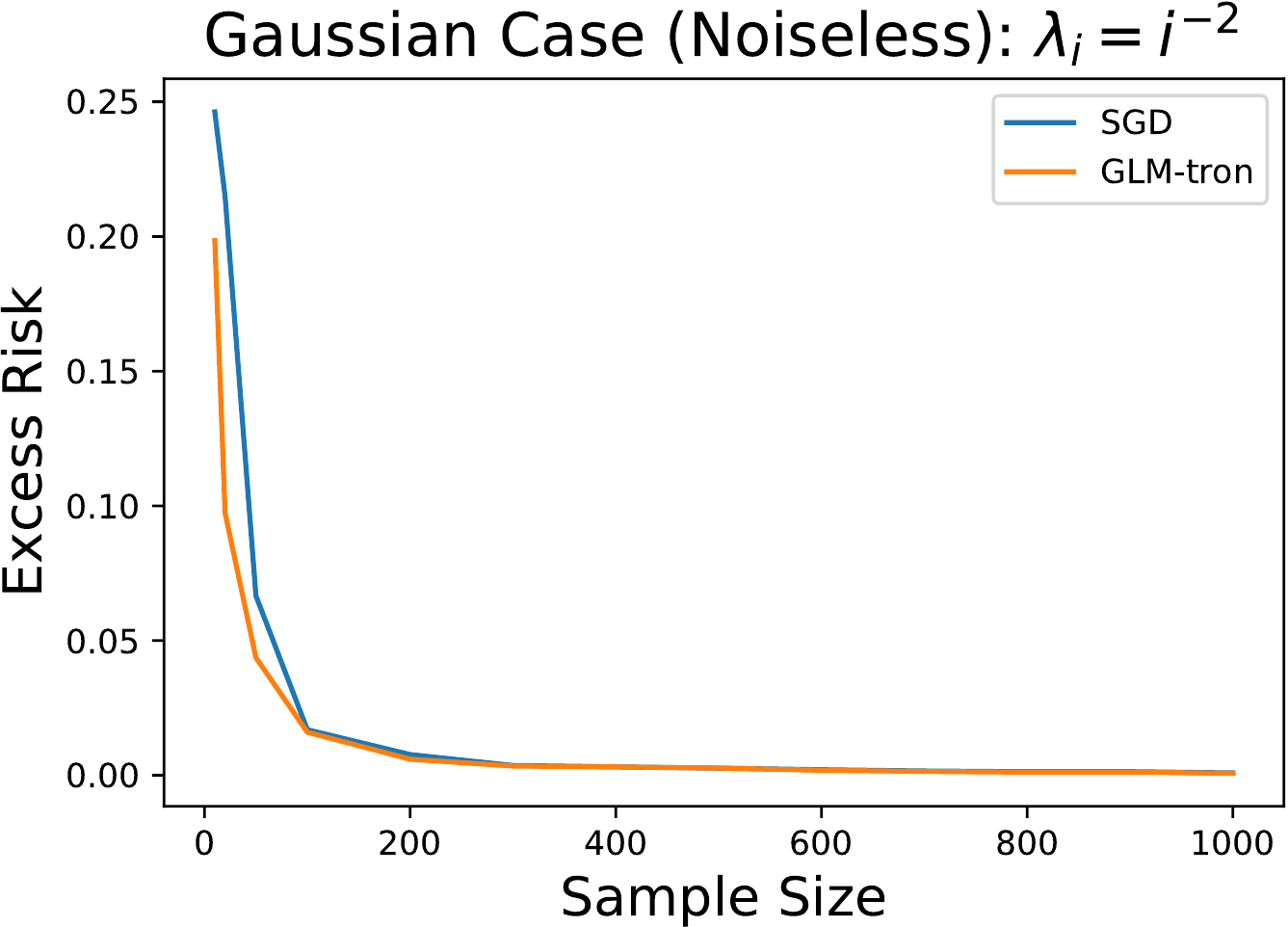}}
      \subfigure{\includegraphics[width=0.35\linewidth]{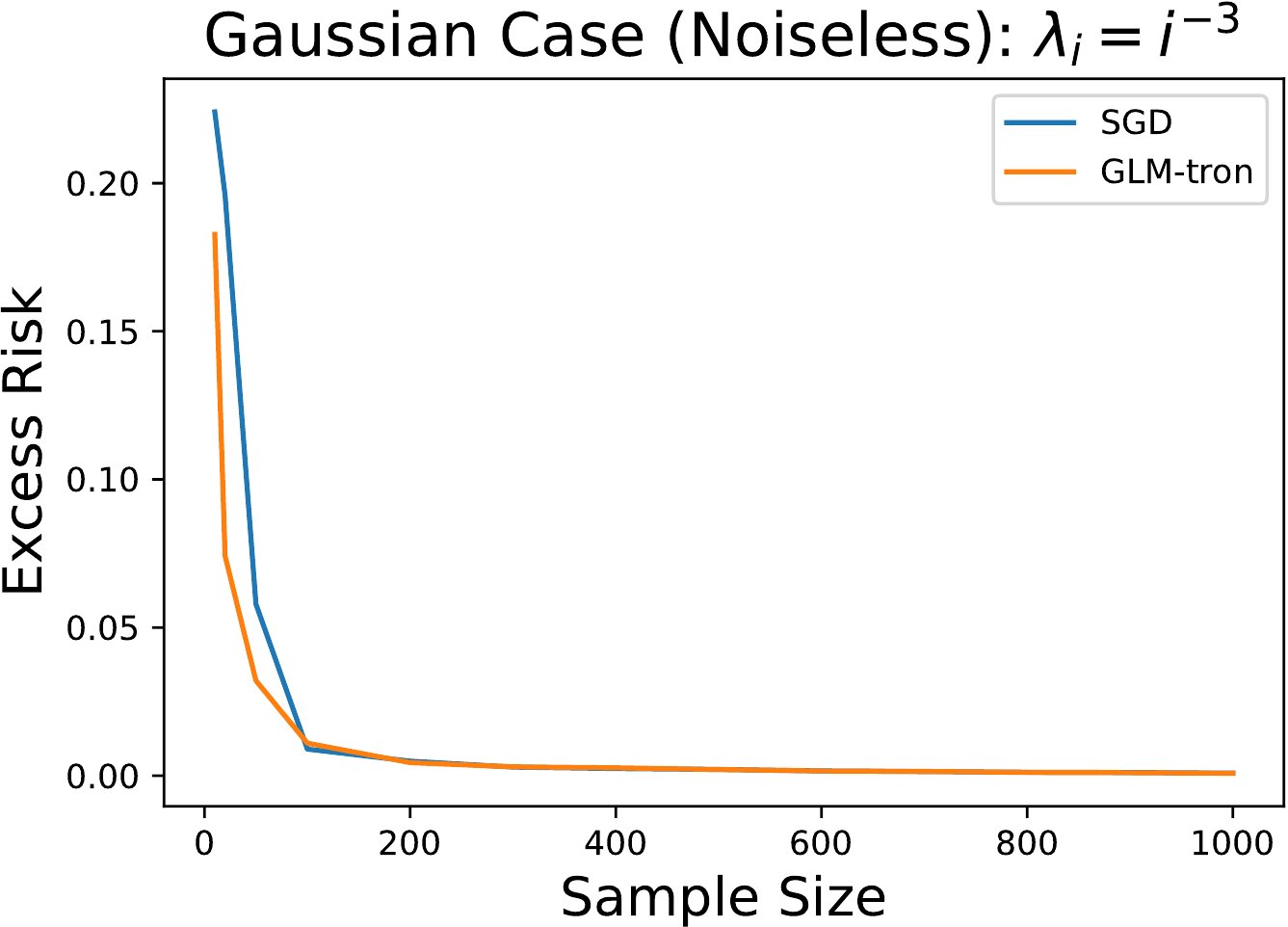}}
    \caption{\small Excess risk comparison between SGD and GLM-tron on Gaussian Distribution. The regression model is well-specified without noise. Other problem and algorithm parameters are the same as those in Figure \ref{fig:gaussian1}.
}
    \label{fig:gaussian2}
\end{figure}

\end{document}